\documentclass{article}

\usepackage{microtype}
\usepackage{graphicx}
\usepackage{subcaption}
\usepackage{booktabs} 

\usepackage{hyperref}

\usepackage[preprint]{icml2026}

\usepackage{amsmath}
\usepackage{amssymb}
\usepackage{mathtools}
\usepackage{amsthm}

\usepackage{tikz}
\usepackage{subcaption}
\usepackage{color}

\usepackage[capitalize,noabbrev]{cleveref}

\usepackage{graphicx}
\usepackage{multirow}
\usepackage{hyperref}
\usepackage{amssymb}
\usepackage{xcolor}
\usepackage{mathtools}
\usepackage{booktabs}
\usepackage{amsmath,amsfonts}
\usepackage{array}
\usepackage{booktabs}
\usepackage{longtable}
\usepackage{amsthm}

\newtheorem{theorem}{Theorem}[section]

\newtheorem{corollary}[theorem]{Corollary}
\theoremstyle{definition}
\newtheorem{definition}[theorem]{Definition}
\theoremstyle{remark}

\usepackage[textsize=tiny]{todonotes}

\icmltitlerunning{Towards a Unified Theory of MoE}

\begin{document}

\twocolumn[
  \icmltitle{Towards a Unified Theory of MoE: \\ Variational Inference, Entropy, and Orthogonality}



  \icmlsetsymbol{equal}{*}

  \begin{icmlauthorlist}
    \icmlauthor{Ye Su}{sch1,yyy}
    \icmlauthor{Yong Liu}{sch2}
  \end{icmlauthorlist}

  \icmlaffiliation{yyy}{University of Chinese Academy of Sciences, Beijing, China}
  \icmlaffiliation{sch1}{Shenzhen Institutes of Advanced Technology, Chinese Academy of Sciences, Shenzhen, China}
  \icmlaffiliation{sch2}{Gaoling School of Artificial Intelligence, Renmin University of China, Beijing, China}

  \icmlcorrespondingauthor{Yong Liu}{liuyonggsai@ruc.edu.cn}


  \vskip 0.3in
]




\printAffiliationsAndNotice{}

\begin{abstract}
    Mixture-of-Experts (MoE) models enable large language models to scale efficiently by activating a sparse subset of experts for each input. However, core mechanisms like Top-$k$ routing and auxiliary load balancing are largely treated as empirical heuristics. In this work, we propose the first unified theoretical framework that interprets these practices from dual perspectives: in Bayesian inference, they act as sparse posterior approximation and prior regularization; in information theory, they serve to minimize routing ambiguity and maximize channel capacity. Furthermore, we formulate MoE routing as an NP-hard sparse subset selection problem and identify a \textit{Coherence Barrier}: when expert representations exhibit high mutual coherence, greedy routing strategies fail to recover the optimal subset. To address this, we show that imposing Euclidean geometric orthogonality within the expert feature space reduces this combinatorial problem to a tractable polynomial-time greedy approximation. Our LLM diagnostics and prior empirical findings are consistent with the geometric insights of this framework. Our analysis suggests that explicit geometric orthogonality serves as a principled inductive bias for MoE optimization, offering a theoretical perspective for understanding and improving sparse routing mechanisms.
\end{abstract}

\section{Introduction}

The scaling laws of dense Large Language Models (LLMs)
highlight a persistent trade-off between performance,
computational cost, and model size
\citep{kaplan2020scaling,hoffmann2022training}.
As dense models scale, evaluating their Transformer blocks
incurs increasing per-token computational costs
\citep{kaplan2020scaling,hoffmann2022training}.
Mixture-of-Experts (MoE) mitigates this by scaling model capacity
through conditional computation, activating only a sparse subset
of experts through Top-$k$ selection for each token, where $k$
denotes the number of selected experts
\citep{fedus2022review,gao2025mola,liu2025netmoe}.
To maximize parameter utilization, MoE systems typically
introduce load balancing objectives that encourage more uniform
aggregate expert utilization during pre-training
\citep{fedus2022switch,pan2024dense}.
MoE has proven highly effective in practice, matching or
outperforming dense baselines with reduced computational
requirements
\citep{shazeer2017outrageously,du2022glam}.

Despite their empirical success, the theoretical understanding of these standard heuristics remains incomplete \citep{hazimeh2021dselect,chen2022towards}. Current designs are largely driven by empirical exploration. Recent studies have begun to incorporate Bayesian perspectives into the MoE \citep{zhang2023similarity,dialameh2025bayesian,li2025bayesian}. However, these works primarily leverage variational inference to optimize novel routing architectures for specific tasks \citep{heins2024gradient,bohne2025mix}, rather than providing a rigorous theoretical derivation for the MoE models. Therefore, this leaves critical questions unanswered:

\textbf{\textit{Do these heuristics introduce inherent inductive biases that limit the model potential? }}

\textbf{\textit{Could establishing a theoretical foundation transform this current ``alchemy" of trial and error into a principled science, thereby providing new insights for the architectural design and optimization of MoE?}}

To address the above questions, we introduce a unified analytical framework that evaluates standard MoE mechanisms through three distinct theoretical lenses.
Through Bayesian inference and information theory, we demonstrate that common heuristics structurally correspond to formal mathematical objectives. In the Bayesian perspective (Lens~1), renormalized Top-$k$ truncation optimally approximates the dense routing distribution under reverse KL, while auxiliary load balancing relates to marginal prior matching through a collision
surrogate and dispatch mismatch. In information theory (Lens~2), these mechanisms bound conditional routing entropy and promote marginal balance, thereby maximizing an explicit lower bound on input--expert mutual information.

Beyond this interpretive framing, we explore representation geometry by modeling expert combination as a Sparse Subset Selection (SSS) reconstruction surrogate with free linear coefficients (Lens~3) \citep{hartmanis1982computers,natarajan1995sparse}. The general SSS problem is NP-hard. By connecting this complexity with classical results in compressed sensing \citep{tropp2004greed,donoho2006compressed,tropp2007signal},
we characterize a \textbf{Coherence Barrier} through a sufficient support-recovery condition for normalized expert dictionaries and exactly sparse targets.
This condition quantifies how expert overlap, coefficient magnitudes, and router-score accuracy jointly affect the guarantee for independent expert selection.

To address this challenge, we demonstrate that \textit{Euclidean geometric orthogonality}, together with unit-norm expert outputs, yields a tractable
structural case of the general \textbf{NP-hard} SSS problem: selection by absolute target--expert correlations becomes globally optimal in \textbf{polynomial time}. Guided by this theoretical insight, we analyze orthogonality regularization and relate its geometric mechanisms to
published LLM evidence \citep{feng2025omoe,guo2025advancing}.
Our diagnostics across five sampling seeds on a frozen
1.33B-parameter Granite MoE show stronger reconstruction--prediction
agreement at lower expert coherence, with mean Spearman correlations
of $0.309$ versus $0.046$; this association persists after accounting
for score alignment and domain.
With identical candidate supports, reference targets, and gating
coefficients, readout-aware KL scoring increases mean ranking
correlation from $0.140$ to $0.280$.
To our knowledge, this work provides the first theoretical interpretation and formal analysis of the core technical modules of MoE models, and proposes practically viable optimization pathways grounded in theoretical insights, thereby offering essential theoretical support and technical assurance for a deeper understanding and novel designs of MoE architectures.

\section{Preliminaries}
\label{sec:preliminaries}

In this section, we introduce the MoE architecture,
renormalized Top-$k$ routing, and auxiliary load balancing
\citep{mu2025comprehensive,cai2025survey}.

\subsection{MoE Architecture}

Consider an MoE layer with $E\geq2$ feed-forward experts,
where expert $E_i:\mathbb R^d\to\mathbb R^d$ is parameterized
by $\theta_i$ and $d\geq1$ is the hidden dimension.
For an input token representation $\mathbf x\in\mathbb R^d$,
the router selects an expert set
$\mathcal S(\mathbf x)\subseteq\{1,\dots,E\}$
of size $k\in\{1,\dots,E\}$.
The layer output is
\[
    \widehat{\mathbf y}
    =\sum_{i\in\mathcal S(\mathbf x)}
      g_i(\mathbf x)E_i(\mathbf x)
    \in\mathbb R^d,
\]
where
$\mathbf g(\mathbf x)
=(g_1(\mathbf x),\dots,g_E(\mathbf x))
\in\mathbb R_{\geq0}^{E}$
denotes the sparse gating weights produced by the router.
Typically, $k\ll E$, so each token evaluates only a small
subset of experts. This enables scaling of total expert
parameters while controlling per-token expert computation
through $k$ \citep{hwang2024pre,xue2024openmoe}.

\subsection{Top-$k$ Routing Mechanism}

A linear router with parameters
$\mathbf W_r\in\mathbb R^{E\times d}$
produces logits $\mathbf h=\mathbf W_r\mathbf x$,
whose $i$-th entry is $h_i$.
The dense routing probabilities and selected expert set are
\begin{align*}
    p_i(\mathbf x)
    &:=\frac{\exp(h_i)}
             {\sum_{j=1}^{E}\exp(h_j)},\\
    \mathcal S(\mathbf x)
    &:=\operatorname{Top\text{-}k}(\mathbf h),
\end{align*}
where $\operatorname{Top\text{-}k}$ returns the indices
of the $k$ largest logits using a fixed tie-breaking rule.

Renormalized Top-$k$ routing normalizes the probabilities
over the selected experts
\citep{shazeer2017outrageously}.
The resulting sparse categorical distribution
$q_k(\cdot\mid\mathbf x)$ gives the gating weights:
\begin{equation}
    q_k(i\mid\mathbf x)
    :=g_i(\mathbf x)
    =
    \begin{cases}
        \displaystyle
        \frac{p_i(\mathbf x)}
             {\sum_{j\in\mathcal S(\mathbf x)}
              p_j(\mathbf x)},
            &i\in\mathcal S(\mathbf x),\\
        0,  &i\notin\mathcal S(\mathbf x).
    \end{cases}
    \label{eq:topk_normalization}
\end{equation}
Thus, $\sum_{i=1}^{E}g_i(\mathbf x)=1$.

\subsection{Auxiliary Load Balancing Loss}

A common failure mode in MoE training is routing collapse:
the router concentrates token assignments on a small subset
of experts, leaving others undertrained
(\emph{expert starvation}).
An auxiliary load-balancing loss
$\mathcal L_{\mathrm{aux}}$ is added to the training
objective to mitigate this imbalance
\citep{fedus2022switch,lepikhin2020gshard}.

For a batch of $N\geq1$ token representations
$\{\mathbf x^{(t)}\}_{t=1}^{N}$, let
$\mathcal S^{(t)}:=\mathcal S(\mathbf x^{(t)})$.
Define the selection frequency $f_i$, measured before
capacity-based token dropping, and the batch-averaged
dense routing probability $\bar p_i$ by
\begin{align*}
    f_i
    &:=\frac1N\sum_{t=1}^{N}
       \mathbf 1\{i\in\mathcal S^{(t)}\},\\
    \bar p_i
    &:=\frac1N\sum_{t=1}^{N}
       p_i(\mathbf x^{(t)}),
\end{align*}
where $\mathbf 1\{\cdot\}$ is the indicator function.
Let $r_i:=f_i/k$,
$\mathbf r:=(r_1,\dots,r_E)$, and
$\bar{\mathbf p}:=(\bar p_1,\dots,\bar p_E)$.
Since each token selects exactly $k$ experts,
\[
    \sum_{i=1}^{E}f_i=k,
    \quad
    \sum_{i=1}^{E}r_i
    =\sum_{i=1}^{E}\bar p_i=1.
\]

For a task loss $\mathcal L_{\mathrm{task}}$ and an
auxiliary-loss coefficient $\alpha\geq0$, we use
\begin{align*}
    \mathcal L_{\mathrm{aux}}
    &:=E\sum_{i=1}^{E}f_i\bar p_i,\\
    \mathcal L_{\mathrm{total}}
    &:=\mathcal L_{\mathrm{task}}
      +\alpha\mathcal L_{\mathrm{aux}}.
\end{align*}
The auxiliary loss couples selection frequencies with
dense routing probabilities to encourage balanced expert
utilization. The balanced-dispatch target is $f_i=k/E$,
equivalently $\mathbf r=\mathbf u$, where
$\mathbf u:=(1/E,\dots,1/E)$ is the uniform probability
vector.

\section{A Bayesian Perspective}
\label{sec:bayesian_perspective}

In this section, we analyze how Top-$k$ routing and load balancing
can be interpreted within a unified Bayesian framework.
Let $\mathcal X$ and $\mathcal Y$ denote the input and target
spaces, respectively, and let $\mathbf x\in\mathcal X$ and
$\mathbf y\in\mathcal Y$ denote an input and its target.
We introduce a latent expert index $Z\in\{1,\dots,E\}$,
representing the per-layer expert choice for a given token.
For this interpretation, we consider the conditional latent-variable model
\[
    P_\theta(\mathbf y\mid\mathbf x)
    =\sum_{z=1}^{E}
      P_\theta(\mathbf y\mid\mathbf x,z)P(z),
\]
where $\theta$ denotes the likelihood parameters and $P$ is
a categorical prior satisfying $P(z)>0$ for every expert $z$.

Evaluating the exact marginal likelihood generally requires
evaluating all experts. We therefore use variational inference
\citep{graves2011practical,blei2017variational} and introduce
the dense categorical inference distribution
$Q_\phi(z\mid\mathbf x)=p_z(\mathbf x)$, where $\phi$ denotes
the router parameters and $p_z(\mathbf x)$ is defined in
Section~\ref{sec:preliminaries}.
This distribution approximates the posterior
$P_\theta(z\mid\mathbf x,\mathbf y)$.
The corresponding evidence lower bound (ELBO) is
\begin{equation}
    \begin{split}
        \mathcal L_{\text{ELBO}}(\mathbf x,\mathbf y)
        ={}&\mathbb E_{z\sim Q_\phi(\cdot\mid\mathbf x)}
        [\log P_\theta(\mathbf y\mid\mathbf x,z)]\\
        &-D_{\mathrm{KL}}\!\left(
          Q_\phi(\cdot\mid\mathbf x)\,\|\,P\right),
    \end{split}
    \label{eq:elbo_general}
\end{equation}
where $D_{\mathrm{KL}}$ denotes Kullback--Leibler divergence
and all logarithms are natural.
Maximizing this ELBO provides a variational objective for
the conditional mixture model.

\subsection{Top-$k$ as Optimal Sparse Approximation}

For categorical latent variables, standard variational methods
can employ unrestricted categorical distributions for $Q_\phi$
\citep{jordan1999introduction,blei2017variational}.
Sparse MoE imposes an expert budget for computational efficiency
\citep{zhou2022mixture,huang2024toward}.
For the routing budget $k\in\{1,\dots,E\}$, we define the
constrained variational family as follows.

\begin{definition}[$k$-Sparse Variational Family]
    Let
    \[
        \begin{aligned}
        \Delta^{E-1}
        &:=\bigl\{\mathbf q=(q_1,\dots,q_E)\in\mathbb R^E:
        \\
        &\quad q_i\geq0,\quad\sum_{i=1}^{E}q_i=1\bigr\}
        \end{aligned}
    \]
    be the probability simplex over $E$ experts. Define
    \[
        \mathcal Q_k
        :=\{\mathbf q\in\Delta^{E-1}:
             \|\mathbf q\|_0\leq k\},
    \]
    where $\|\mathbf q\|_0$ counts its nonzero entries.
\end{definition}

For a fixed input $\mathbf x$, write
$\mathbf p:=(p_1(\mathbf x),\dots,p_E(\mathbf x))$.
Renormalized Top-$k$ routing solves a reverse-KL projection
onto $\mathcal Q_k$.

\begin{theorem}[Optimality of Top-$k$ Routing]
    \label{thm:topk_optimality}
    Let $\mathbf p\in\Delta^{E-1}$ satisfy $p_i>0$ for every $i$,
    and let $\mathbf q_k$ be its renormalized Top-$k$
    probability vector, constructed as in
    Eq.~\eqref{eq:topk_normalization}. Then
    \[
        \mathbf q_k
        \in\operatorname*{arg\,min}_{\mathbf q\in\mathcal Q_k}
        D_{\mathrm{KL}}(\mathbf q\,\|\,\mathbf p).
    \]
    The solution is unique if $k=E$, or if $k<E$ and the
    $k$-th largest probability is strictly greater than
    the $(k+1)$-th largest probability.
\end{theorem}

Theorem~\ref{thm:topk_optimality} identifies renormalized
Top-$k$ routing as an exact reverse-KL projection onto
$\mathcal Q_k$. It retains a support with the largest
total dense routing probability mass and gives the
closest sparse distribution under this divergence
for the prescribed expert budget.

\begin{proof}
    Please refer to Appendix~\ref{sec:appendix_b} for a detailed proof.
\end{proof}

For each input, we write
$Q_{\phi,k}(i\mid\mathbf x):=q_k(i\mid\mathbf x)$
for the sparse routing distribution defined in
Eq.~\eqref{eq:topk_normalization}.

\subsection{Load Balancing as Marginal Prior Matching}

The second term in Eq.~\eqref{eq:elbo_general} regularizes
the inference distribution toward the prior.
We take the uniform categorical prior $P=\mathcal U_E$,
where $\mathcal U_E(i)=1/E$, with probability vector
$\mathbf u=(1/E,\dots,1/E)$.
For unrestricted categorical distributions, minimizing the
per-input KL term alone gives
$Q_\phi(\cdot\mid\mathbf x)=\mathcal U_E$ for every input,
making the expert index independent of the input.
Matching the aggregated distribution to the prior encourages
balanced marginals while allowing input-dependent routing.

Let $\mathcal B=\{(\mathbf x^{(t)},\mathbf y^{(t)})\}_{t=1}^{N}$
be a batch of $N\geq1$ token--target pairs. Define
\[
    \bar Q_\phi(z)
    :=\frac1N\sum_{t=1}^{N}Q_\phi(z\mid\mathbf x^{(t)}),
\]
whose probability vector is
$\bar{\mathbf p}=(\bar p_1,\dots,\bar p_E)$, with
$\bar p_i=\bar Q_\phi(i)$ as in
Section~\ref{sec:preliminaries}.
Let $X$ follow the empirical distribution of the batch inputs,
with $Z\mid X=\mathbf x\sim Q_\phi(\cdot\mid\mathbf x)$.
The resulting mutual information is
\[
    I_{\mathcal B,\phi}(X;Z)
    :=\frac1N\sum_{t=1}^{N}
      D_{\mathrm{KL}}\!\left(
        Q_\phi(\cdot\mid\mathbf x^{(t)})
        \,\|\,\bar Q_\phi\right).
\]
Per-input prior regularization decomposes as
\begin{equation}
    \begin{split}
        &\frac1N\sum_{t=1}^{N}
        D_{\mathrm{KL}}\!\left(
          Q_\phi(\cdot\mid\mathbf x^{(t)})\,\|\,P\right)\\
        &\quad=I_{\mathcal B,\phi}(X;Z)
          +D_{\mathrm{KL}}(\bar Q_\phi\,\|\,P).
    \end{split}
    \label{eq:prior_matching_decomposition}
\end{equation}
Thus, aggregated prior matching directly regularizes
marginal expert utilization.

We use the pre-capacity selection frequencies $f_i$,
normalized frequencies $r_i=f_i/k$, and auxiliary loss
$\mathcal L_{\mathrm{aux}}$ defined in
Section~\ref{sec:preliminaries}, with
$\mathbf r=(r_1,\dots,r_E)\in\Delta^{E-1}$.

\begin{theorem}[Auxiliary Loss and a R\'enyi Entropy Surrogate]
    \label{thm:aux_loss}
    Define the collision surrogate
    \[
        \mathcal L_{\mathrm{coll}}
        :=E\sum_{i=1}^{E}\bar p_i^2
    \]
    and the collision entropy, or R\'enyi entropy of order $2$, by
    \[
        H_2(\bar{\mathbf p})
        :=-\log\sum_{i=1}^{E}\bar p_i^2.
    \]
    Then
    \begin{align*}
        \mathcal L_{\mathrm{coll}}
        &=E\exp[-H_2(\bar{\mathbf p})]\\
        &=1+E\|\bar{\mathbf p}-\mathbf u\|_2^2,\\
        D_{\mathrm{KL}}(\bar{\mathbf p}\,\|\,\mathbf u)
        &\leq\log\mathcal L_{\mathrm{coll}},
    \end{align*}
    where $\|\cdot\|_2$ denotes the Euclidean norm. Moreover,
    \begin{align*}
        \frac{\mathcal L_{\mathrm{aux}}}{k}
        &=\mathcal L_{\mathrm{coll}}
          +E(\mathbf r-\bar{\mathbf p})^\top
            (\bar{\mathbf p}-\mathbf u),\\
        \left|\frac{\mathcal L_{\mathrm{aux}}}{k}
              -\mathcal L_{\mathrm{coll}}\right|
        &\leq E\|\mathbf r-\bar{\mathbf p}\|_2
                 \|\bar{\mathbf p}-\mathbf u\|_2.
    \end{align*}
\end{theorem}

Theorem~\ref{thm:aux_loss} shows that the collision
surrogate is uniquely minimized at
$\bar{\mathbf p}=\mathbf u$, where collision entropy
is maximized and the marginal KL upper bound vanishes.
The dispatch-mismatch bound quantifies how closely
$\mathcal L_{\mathrm{aux}}/k$ approximates this surrogate,
with exact agreement when $\mathbf r=\bar{\mathbf p}$.
When these distributions are close, the auxiliary-loss
value admits a marginal-prior-matching interpretation
linked to balanced expert utilization.

\begin{proof}
    Please refer to Appendix~\ref{sec:appendix_c} for a detailed proof.
\end{proof}

\subsection{From Theory to Standard MoE Implementation}
\label{sec:closing_the_loop}

In the standard MoE objective
$\mathcal L_{\mathrm{task}}+\alpha\mathcal L_{\mathrm{aux}}$,
the task loss fits predictions formed from active expert
features, while the auxiliary loss encourages balanced
aggregate expert utilization.
Theorem~\ref{thm:topk_optimality} identifies the renormalized
gates as a sparse reverse-KL projection of the dense routing
distribution.
Theorem~\ref{thm:aux_loss} relates the normalized auxiliary-loss
value to uniform marginal prior matching through a collision
surrogate and an explicit dispatch-mismatch bound.
Appendix~\ref{app:aggregated_prior} develops these operational
correspondences and derives the aggregated-prior objective's
ELBO--mutual-information identity for dense and sparse routing.

In summary, renormalized Top-$k$ routing provides an exact sparse reverse-KL projection of the dense routing distribution. The collision surrogate and dispatch-mismatch bound quantify how auxiliary-loss values relate to matching the aggregate routing marginal to a uniform prior. Together, these results connect token-level computational sparsity and batch-level expert utilization within a unified Bayesian interpretation of standard MoE training.

\section{An Information-Theoretic Perspective}
\label{sec:info_theory}

In this section, we frame the router as a discrete
communication channel and analyze MoE regularization
through mutual information (MI).
Let $\mathcal X$ denote the input space and
$\mathcal Z=\{1,\dots,E\}$ the expert index set,
with $E\geq2$.
For a batch $\mathcal B$ of $N\geq1$ token representations
$\{\mathbf x^{(t)}\}_{t=1}^{N}$, let $\mathbf X$ follow
their empirical distribution.
For a routing budget $k\in\{1,\dots,E-1\}$,
let $p(z\mid\mathbf x)$ denote the dense routing distribution
and $q_k(z\mid\mathbf x)$ its renormalized Top-$k$
distribution defined in Eq.~\eqref{eq:topk_normalization}.
The expert index $Z\in\mathcal Z$ has conditional law
$Z\mid\mathbf X=\mathbf x\sim q_k(\cdot\mid\mathbf x)$.

The MI decomposes as
\citep{cover1999elements,tishby2000information}
\begin{equation}
    I(\mathbf X;Z)
    =\underbrace{H(Z)}_{\text{Utilization}}
    -\underbrace{H(Z\mid\mathbf X)}_{\text{Ambiguity}},
    \label{eq:mutual_info}
\end{equation}
where $H(Z)$ is the marginal Shannon entropy.
For a probability vector $\mathbf a=(a_1,\dots,a_E)$,
define
\[
    H(\mathbf a):=-\sum_{i=1}^{E}a_i\log a_i,
\]
with $0\log0:=0$ and natural logarithms.
The conditional entropy is
\[
    H(Z\mid\mathbf X)
    =\frac1N\sum_{t=1}^{N}
      H\!\left(q_k(\cdot\mid\mathbf x^{(t)})\right).
\]

\subsection{Top-$k$ as an Entropy Bottleneck}

At a fixed marginal entropy $H(Z)$, reducing $H(Z\mid\mathbf X)$ increases MI. High conditional entropy corresponds to a dispersed distribution of expert weights for individual inputs. Top-$k$ routing controls this dispersion through its support budget.

\begin{theorem}[Top-$k$ Bounds Routing Ambiguity]
    \label{thm:cond_entropy}
    For each input $\mathbf x$, let
    $p(\cdot\mid\mathbf x)$ be a dense categorical
    distribution with strictly positive entries,
    and let $q_k(\cdot\mid\mathbf x)$ be its
    renormalized Top-$k$ distribution defined in
    Eq.~\eqref{eq:topk_normalization}.
    Then
    \[
        0\leq H\!\left(q_k(\cdot\mid\mathbf x)\right)
        \leq
        \min\!\left\{
            H\!\left(p(\cdot\mid\mathbf x)\right),
            \log k
        \right\},
    \]
    and consequently $0\leq H(Z\mid\mathbf X)\leq\log k$. For $k=1$, $H(Z\mid\mathbf X)=0$ and $Z$ is a deterministic function of $\mathbf X$.
\end{theorem}

This theorem identifies Top-$k$ routing as an entropy-reducing projection. It restricts the maximum conditional-entropy budget from $\log E$ to $\log k$, a reduction of $\log(E/k)$ in that budget. The resulting entropy bottleneck concentrates the routing distribution on the selected experts.

\begin{proof}
    Please refer to Appendix \ref{sec:appendix_d} for a detailed proof.
\end{proof}

\subsection{Load Balancing as MI Lower-Bound Maximization}
\label{subsec:load_balance_mi_max}

Theorem~\ref{thm:cond_entropy} gives
$I(\mathbf X;Z)\geq H(Z)-\log k$,
motivating marginal entropy maximization for a fixed
routing budget.
Define the aggregated sparse routing distribution by
\[
    \bar Q_k(z)
    :=\frac1N\sum_{t=1}^{N}
      q_k(z\mid\mathbf x^{(t)}),
\]
and write
$\bar{\mathbf q}_k:=(\bar Q_k(i))_{i=1}^{E}$
and $\mathbf u:=(1/E,\dots,1/E)$.
The collision surrogate for this marginal is
\[
    \mathcal L_{\mathrm{coll},k}
    :=E\sum_{i=1}^{E}\bar Q_k(i)^2.
\]

\begin{theorem}[Load Balancing Maximizes an MI Lower Bound]
    \label{thm:capacity}
    The empirical routing distribution satisfies
    \[
        I(\mathbf X;Z)
        \geq
        \log(E/k)-\log\mathcal L_{\mathrm{coll},k}.
    \]
    Minimizing $\mathcal L_{\mathrm{coll},k}$
    maximizes this lower bound.
\end{theorem}

A uniform sparse marginal
$\bar{\mathbf q}_k=\mathbf u$ uniquely minimizes
the collision surrogate at
$\mathcal L_{\mathrm{coll},k}=1$ and yields
$I(\mathbf X;Z)\geq\log(E/k)$.
This connects balanced expert utilization to an explicit
information guarantee.
The upper bound
$I(\mathbf X;Z)\leq H(Z)\leq\log E$
also shows that concentration of the marginal on a small
subset of experts restricts the available information budget.

\begin{proof}
    Please refer to Appendix \ref{sec:appendix_e} for a detailed proof.
\end{proof}

To connect this sparse-marginal bound to the standard
auxiliary loss, use the dense marginal $\bar{\mathbf p}$
and collision surrogate $\mathcal L_{\mathrm{coll}}$
from Section~\ref{sec:bayesian_perspective}, and define
\[
    \delta_k
    :=\|\bar{\mathbf q}_k-\bar{\mathbf p}\|_2.
\]
The triangle inequality gives
\[
    \mathcal L_{\mathrm{coll},k}
    \leq
    1+\left(
        \sqrt{\mathcal L_{\mathrm{coll}}-1}
        +\sqrt E\,\delta_k
    \right)^2.
\]
Together with the auxiliary-loss mismatch bound in
Theorem~\ref{thm:aux_loss}, this inequality connects
standard load balancing to the sparse-routing MI lower
bound through the dispatch and truncation discrepancies.

In summary, the two standard mechanisms of MoE align structurally with the decomposition of MI. Top-$k$ routing bounds $H(Z\mid\mathbf X)$, while balancing the sparse marginal through $\mathcal L_{\mathrm{coll},k}$ maximizes an explicit lower bound on $I(\mathbf X;Z)$. The dispatch and truncation discrepancy bounds connect this marginal objective to the standard auxiliary loss. Together, these mechanisms connect routing concentration with aggregate expert utilization.

\section{The Combinatorial Cliff: From Theory to Reality}
\label{sec:combinatorial_cliff}

The Bayesian and information-theoretic perspectives characterize routing sparsity and marginal expert utilization. To examine how representation geometry affects the joint contribution of selected experts, we formulate expert aggregation as a Sparse Subset Selection (SSS) problem. This formulation connects independent expert rankings to a joint reconstruction objective and yields a sufficient condition for recovering its optimal support.

\subsection{The Hidden Optimization Problem}

Consider a fixed input $\mathbf x$. Let $E\geq2$ be the number of experts, $d\geq1$ their output dimension, and $k\in\{1,\dots,E-1\}$ the routing budget. Define the dictionary of expert outputs as
\[
    \mathbf E
    =[\mathbf e_1(\mathbf x),\dots,\mathbf e_E(\mathbf x)]
    \in\mathbb R^{d\times E},
\]
where $\mathbf e_i(\mathbf x)\in\mathbb R^d$ is the output of expert $i$. For this fixed input, write $\mathbf e_i=\mathbf e_i(\mathbf x)$ and treat $\mathbf E$ as a fixed dictionary. Given a target representation $\mathbf y\in\mathbb R^d$, the idealized objective is to find a sparse coefficient vector $\boldsymbol\alpha\in\mathbb R^E$ that reconstructs $\mathbf y$:
\begin{equation*}
    \min_{\boldsymbol\alpha\in\mathbb R^E}
    \|\mathbf y-\mathbf E\boldsymbol\alpha\|_2^2
    \quad\text{s.t.}\quad
    \|\boldsymbol\alpha\|_0\leq k,
\end{equation*}
where $\|\cdot\|_2$ is the Euclidean norm and $\|\cdot\|_0$ counts nonzero entries. For a nonempty expert subset $\mathcal S\subseteq\{1,\dots,E\}$, let $\mathbf E_{\mathcal S}$ be the submatrix containing its columns, and define
\[
    \mathcal L(\mathcal S)
    :=\min_{\mathbf a\in\mathbb R^{|\mathcal S|}}
    \|\mathbf y-\mathbf E_{\mathcal S}\mathbf a\|_2^2,
\]
where $\mathbf a$ contains the coefficients on $\mathcal S$. An optimal support of size $k$ satisfies
\[
    \mathcal S^*
    \in\operatorname*{arg\,min}_{|\mathcal S|=k}
    \mathcal L(\mathcal S).
\]

\begin{theorem}[Intractability of Sparse Subset Selection]
    \label{thm:np_hard}
    Given a rational dictionary $\mathbf E$, a rational target $\mathbf y$, and a sparsity budget $k$ as part of the input, deciding whether there exists $\boldsymbol\alpha\in\mathbb R^E$ such that
    \[
        \mathbf E\boldsymbol\alpha=\mathbf y,
        \quad
        \|\boldsymbol\alpha\|_0\leq k
    \]
    is NP-hard. Consequently, computing a globally optimal support for the general SSS problem is NP-hard.
\end{theorem}

\begin{proof}
    Please refer to Appendix \ref{sec:appendix_f} for a detailed proof.
\end{proof}

In practice, the input-side gating network produces learned scalar scores $\mathbf s=(s_1,\dots,s_E)\in\mathbb R^E$ and selects the indices of their $k$ largest entries. To analyze independent expert rankings through the reconstruction objective, define the correlation utilities
\[
    c_i:=|\langle\mathbf e_i,\mathbf y\rangle|,
    \quad
    \varepsilon:=\max_{1\leq i\leq E}|s_i-c_i|,
\]
where $\langle\cdot,\cdot\rangle$ denotes the Euclidean inner product and $\varepsilon$ measures the score approximation error. Let $\mathcal S_G=\operatorname{Top\text{-}k}(\mathbf s)$ denote the selected index set, using a fixed tie-breaking rule. The following analysis quantifies how dictionary geometry and score accuracy control this selection.

\subsection{The Failure Mode: The Coherence Barrier}

We use mutual coherence to quantify overlap between
expert outputs and analyze when independent rankings
recover the optimal expert subset.

\begin{definition}[Mutual Coherence]
    Let
    $\mathbf E=[\mathbf e_1,\dots,\mathbf e_E]
    \in\mathbb R^{d\times E}$
    be the expert-output dictionary for a fixed input,
    where $E\geq2$ is the number of experts and
    $d\geq1$ is their output dimension.
    Assume $\|\mathbf e_i\|_2=1$ for every $i$.
    Its mutual coherence is
    \[
        \mu(\mathbf E)
        :=
        \max_{1\leq i\neq j\leq E}
        |\mathbf e_i^\top\mathbf e_j|.
    \]
\end{definition}

Larger coherence indicates greater alignment between
a pair of expert outputs.
For an integer routing budget $1\leq k<E$,
consider a target $\mathbf y\in\mathbb R^d$
with an exactly sparse representation
\[
    \mathbf y=\mathbf E\boldsymbol\alpha^0,
    \quad
    \boldsymbol\alpha^0\in\mathbb R^E,
\]
whose support is
\[
    \mathcal S_0
    :=\{i:\alpha_i^0\neq0\},
    \quad
    |\mathcal S_0|=k.
\]
Define its smallest and largest nonzero
coefficient magnitudes by
\[
    \alpha_{\min}
    :=\min_{i\in\mathcal S_0}|\alpha_i^0|>0,
    \quad
    \alpha_{\max}
    :=\max_{i\in\mathcal S_0}|\alpha_i^0|.
\]

For expert scores
$\mathbf s=(s_1,\dots,s_E)\in\mathbb R^E$,
the absolute-correlation approximation error is
\[
    \varepsilon
    :=
    \max_{1\leq i\leq E}
    \left|
        s_i-|\mathbf e_i^\top\mathbf y|
    \right|.
\]
Let $\mathcal S_G$ contain the indices of the
$k$ largest scores, using a fixed tie-breaking rule.
Let $\mathcal L(\mathcal S)$ denote the SSS
reconstruction loss defined in the preceding subsection.

\begin{theorem}[Coherence-Based Support Recovery]
    \label{thm:coherence_barrier}
    Under the above setup, if
    \[
        \alpha_{\min}
        >(2k-1)\mu(\mathbf E)\alpha_{\max}
        +2\varepsilon,
    \]
    then $\mathcal S_G=\mathcal S_0$ is the unique
    optimal support of size $k$, with
    $\mathcal L(\mathcal S_G)=0$.
    Moreover, $\boldsymbol\alpha^0$ is the unique
    minimizer of the SSS problem.
\end{theorem}

\begin{proof}
    Please refer to Appendix \ref{sec:appendix_g} for a detailed proof.
\end{proof}

The condition guarantees that every expert in the
true support ranks above every expert outside it.
Lower coherence reduces cross-expert interference,
allowing greater differences in coefficient magnitudes
or larger score approximation errors.
For equal-magnitude nonzero coefficients and exact
correlation scores
($\alpha_{\min}=\alpha_{\max}$ and $\varepsilon=0$),
the condition reduces to
$\mu(\mathbf E)<1/(2k-1)$.

Figure~\ref{fig:coherence_barrier} illustrates support
recovery for sparsity $k=6$ using synthetic dictionaries
with varying coherence.

\begin{figure}[t]
    \centering
    \includegraphics[width=\columnwidth]{
        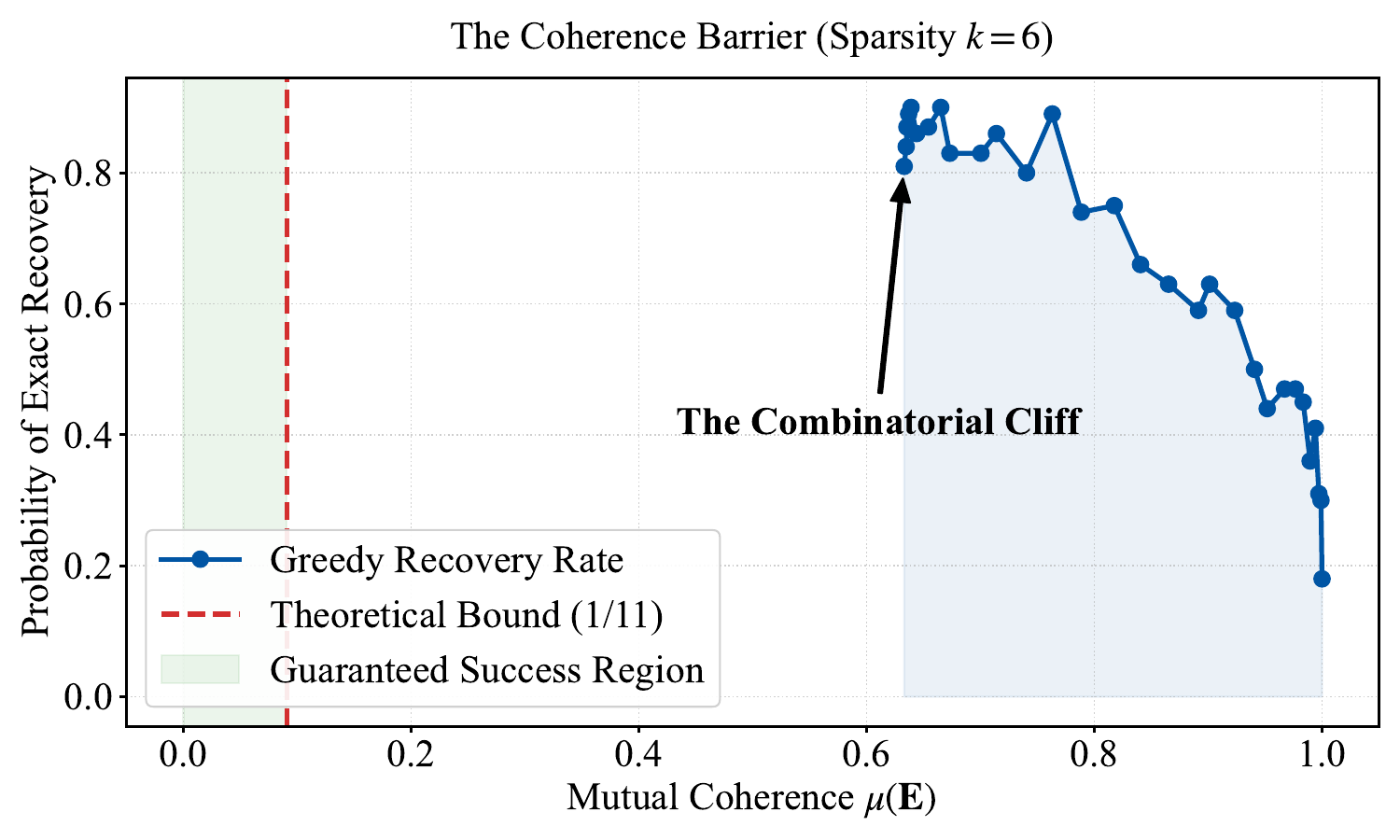
    }
    \caption{
        \textbf{The Coherence Barrier.}
        The blue curve shows empirical support-recovery
        success rates for $k=6$.
        The red dashed line marks the sufficient-condition
        reference $\mu=1/(2k-1)=1/11$ for equal-magnitude
        nonzero coefficients and exact correlation scores
        in Theorem~\ref{thm:coherence_barrier}.
    }
    \label{fig:coherence_barrier}
\end{figure}

\section{Bridging the Gap: From NP-Hard to P}
\label{sec:bridging_the_gap}

In this section, we show that an orthonormal expert dictionary
makes the sparse reconstruction objective separable and enables
globally optimal support selection in polynomial time.
We then connect this structural condition to orthogonality
regularization and published LLM evidence.

\subsection{The Orthogonality Condition}

We use the SSS objective $\mathcal L(\mathcal S)$ defined in
Section~\ref{sec:combinatorial_cliff}.
Let $E\geq2$ be the number of experts, $d\geq1$ their output
dimension, and $\Omega:=\{1,\dots,E\}$ the expert index set.
For each positive integer $m$, let $\mathbf I_m$ denote the
$m\times m$ identity matrix.

\begin{theorem}[Optimality of Greedy Selection on an Orthonormal Dictionary]
    \label{thm:orthogonality_optimality}
    Let
    $\mathbf E=[\mathbf e_1,\dots,\mathbf e_E]
    \in\mathbb R^{d\times E}$
    satisfy $\mathbf E^\top\mathbf E=\mathbf I_E$,
    let $\mathbf y\in\mathbb R^d$ be a target,
    and let $1\leq k\leq E$ be an integer.
    Let $\mathcal S_{\mathrm{corr}}$ contain the indices
    of the $k$ largest values of $|\mathbf e_i^\top\mathbf y|$,
    using a fixed tie-breaking rule. Then
    \[
        \mathcal S_{\mathrm{corr}}
        \in
        \operatorname*{arg\,min}_{\substack{
            \mathcal S\subseteq\Omega\\
            |\mathcal S|=k}}
        \mathcal L(\mathcal S).
    \]
\end{theorem}

\begin{proof}
    Please refer to Appendix~\ref{sec:appendix_h} for a detailed proof.
\end{proof}

Computing the correlations requires $O(dE)$ operations,
followed by $O(E)$ support selection using a linear-time
selection algorithm. This yields an exact polynomial-time
solution for orthonormal dictionaries.

\subsection{Orthogonality Regularization}

Motivated by output-space orthogonality methods
\citep{feng2025omoe,guo2025advancing},
we analyze a normalized pairwise penalty.
For nonzero expert outputs, define
\[
    \mathbf v_i
    :=\frac{\mathbf e_i}{\|\mathbf e_i\|_2},
    \quad
    \mathbf V:=[\mathbf v_1,\dots,\mathbf v_E].
\]
For $\mathcal T\subseteq\Omega$, let
\[
    \mathcal L_o(\mathcal T)
    :=\sum_{i\in\mathcal T}
      \sum_{\substack{j\in\mathcal T\\j\neq i}}
      (\mathbf v_i^\top\mathbf v_j)^2,
    \quad
    \mathcal L_{o,\mathrm{all}}
    :=\mathcal L_o(\Omega).
\]
For an active set $\mathcal S$ with $|\mathcal S|\geq2$,
\[
    \mu(\mathbf V_{\mathcal S})
    \leq
    \sqrt{\mathcal L_o(\mathcal S)/2},
\]
where $\mathbf V_{\mathcal S}$ contains the columns indexed
by $\mathcal S$ and $\mu$ is the mutual coherence defined in
Section~\ref{sec:combinatorial_cliff}.
The active-set penalty controls co-activated outputs;
its full-dictionary counterpart yields the following
recovery condition.

\begin{corollary}[Exact Recovery under a Global Orthogonality Bound]
    \label{thm:exact_recovery}
    Under the unit-column sparse representation and scoring
    assumptions of Theorem~\ref{thm:coherence_barrier}, if
    \[
        \alpha_{\min}
        >(2k-1)\alpha_{\max}
          \sqrt{\mathcal L_{o,\mathrm{all}}/2}
          +2\varepsilon,
    \]
    then $\mathcal S_G=\mathcal S_0$ is the unique optimal
    support of size $k$, with $\mathcal L(\mathcal S_G)=0$.
\end{corollary}

\begin{proof}
    Please refer to Appendix~\ref{sec:appendix_i} for a detailed proof.
\end{proof}

Appendix~\ref{sec:appendix_j} develops determinantal diversity
as a geometric extension.

\subsection{LLM Evidence and Theoretical Connections}
\label{subsec:published_llm_evidence}

Published LLM studies report benefits from expert-output
orthogonalization and joint expert--router regularization
\citep{feng2025omoe,guo2025advancing}.
Across five sampling seeds, our diagnostics on a frozen
1.33B-parameter Granite MoE show stronger agreement
between reconstruction error and next-token prediction
loss at lower expert coherence:
mean within-position Spearman correlation is $0.309$,
compared with $0.046$ at higher coherence.
This association persists after accounting for score
alignment and domain.
With identical candidate supports, reference targets,
and gating coefficients, readout-aware KL scoring
increases mean correlation from $0.140$ to $0.280$.
These findings align with the geometric motivation of
Theorem~\ref{thm:orthogonality_optimality} and
Corollary~\ref{thm:exact_recovery}, connecting expert
geometry with task-related candidate evaluation.
Appendix~\ref{sec:published_llm_evidence_appendix}
presents the published evidence, experimental settings,
and complete five-seed results.

\section{Conclusion and Discussion}

\subsection{Conclusion}

This work unifies Bayesian inference and information theory to formalize standard MoE heuristics. Through combinatorial analysis, we identify the coherence barrier as the fundamental limit of greedy routing and prove that geometric orthogonality reduces the routing problem from NP-hard to P. These findings frame orthogonality regularization as a principled inductive bias for MoE optimization. Our results suggest a paradigm shift in conditional computation: the primary bottleneck lies in the latent geometry of expert representations rather than routing algorithm complexity. We advocate for architectures that learn intrinsically routable representations, where geometric orthogonality ensures efficient and stable sparse activation.

\subsection{Discussion}

Existing MoE studies promote expert diversity through output orthogonalization and regularization \citep{guo2025advancing}. Our analysis explains how this diversity interacts with the computation performed by a language-model router.

First, load balancing describes how tokens are distributed across experts.
Dispatch mismatch and truncation discrepancy quantify how the auxiliary-loss value relates to the sparse routing distribution used after Top-$k$ selection.
This provides concrete diagnostics for connecting training-time balancing statistics to sparse expert utilization: expert counts, dense probabilities, and
active gating weights describe different stages of the routing computation.

Second, two experts can occupy two activation slots for the same token while producing nearly identical features. Under the reconstruction model, orthonormal expert outputs make reconstruction contributions additive, so independent absolute-correlation rankings become globally optimal \citep{tropp2004greed}.
The recovery condition further quantifies how expert separation, coefficient magnitudes, and alignment between learned router scores and absolute target--expert correlations jointly support correct selection.

These findings identify three practical design targets for LLM MoE layers: distribute tokens evenly, use the activation budget for distinct expert features, and rank relevant experts accurately. They motivate jointly examining output geometry and router-score alignment, and measuring how these quantities relate to token-prediction loss.

\section*{Limitations and Future Work}

Our empirical evaluation combines diagnostics of a frozen
1.33B-parameter MoE model on CPU across five sampling seeds
with published training results.
The diagnostics focus on expert selection in the final MoE
layer.
GPU resources were unavailable for independent large-scale
training experiments examining how orthogonality regularization
changes expert and router learning.

Future work should extend this analysis to controlled MoE
training across model sizes and layers under matched
computational budgets, jointly measuring expert coherence,
routing-score alignment, marginal balance, and task performance.
Such experiments would clarify when improved expert separation
and routing scores translate into better language-model
optimization and prediction.

\bibliography{example_paper}
\bibliographystyle{icml2026}

\appendix

\twocolumn

\section{Proof of Theorem~\ref{thm:topk_optimality}}
\label{sec:appendix_b}

\begin{proof}
Let $E\geq1$ and $1\leq k\leq E$ be integers. Write
\[
    \begin{aligned}
    \Delta^{E-1}
    &:=\Bigl\{
    \mathbf q=(q_1,\ldots,q_E)\in\mathbb R^E:\\
    &\quad q_i\geq0,\ \sum_{i=1}^{E}q_i=1
    \Bigr\},
    \end{aligned}
\]
\[
    \mathcal Q_k
    :=\{\mathbf q\in\Delta^{E-1}:\|\mathbf q\|_0\leq k\},
\]
where $\|\mathbf q\|_0$ counts the nonzero entries of $\mathbf q$.
Let $\mathbf p=(p_1,\ldots,p_E)\in\Delta^{E-1}$ with $p_i>0$
for every $i$. All logarithms are natural, with
$0\log(0/p_i):=0$.

For $a>0$, define $h(a):=a-1-\log a$. Then
\[
    h'(a)=1-\frac1a=\frac{a-1}{a},
    \quad
    h''(a)=\frac1{a^2}>0,
\]
\[
    h'(1)=0,
    \quad
    h(1)=0.
\]
Thus
\begin{equation}
    \begin{gathered}
        h(a)\geq h(1)=0,
        \quad
        \log a\leq a-1,\\
        \log a=a-1\iff a=1.
    \end{gathered}
    \label{eq:appendix_b_log_bound}
\end{equation}

For any $\mathbf q\in\mathcal Q_k$, define
\[
    S:=\operatorname{supp}(\mathbf q)
      :=\{i:q_i>0\},
\]
\[
    M_S:=\sum_{i\in S}p_i,
    \quad
    \rho_i^S:=\frac{p_i}{M_S}\quad(i\in S).
\]
Then
\[
    1\leq|S|\leq k,
    \quad M_S>0,
    \quad \rho_i^S>0\quad(i\in S),
\]
and
\[
    \sum_{i\in S}q_i=1,
    \quad
    \sum_{i\in S}\rho_i^S
    =\frac{\sum_{i\in S}p_i}{M_S}
    =\frac{M_S}{M_S}=1.
\]
By Eq.~\eqref{eq:appendix_b_log_bound},
\begin{align*}
    D_{\mathrm{KL}}(\mathbf q\,\|\,\mathbf p)
    &=\sum_{i\in S}q_i\log\frac{q_i}{p_i}\\
    &=\sum_{i\in S}q_i
      \log\frac{q_i}{M_S\rho_i^S}\\
    &=\sum_{i\in S}q_i
      \left(
      \log\frac{q_i}{\rho_i^S}-\log M_S
      \right)\\
    &=\sum_{i\in S}q_i\log\frac{q_i}{\rho_i^S}
      -\left(\sum_{i\in S}q_i\right)\log M_S\\
    &=-\sum_{i\in S}q_i\log\frac{\rho_i^S}{q_i}
      -\log M_S\\
    &\geq-\sum_{i\in S}q_i
      \left(\frac{\rho_i^S}{q_i}-1\right)
      -\log M_S\\
    &=-\sum_{i\in S}\rho_i^S
      +\sum_{i\in S}q_i-\log M_S\\
    &=-1+1-\log M_S\\
    &=-\log M_S.
\end{align*}
Since $q_i>0$ for every $i\in S$, equality holds precisely when
\[
    \frac{\rho_i^S}{q_i}=1
    \quad\text{for every }i\in S,
\]
or equivalently,
\[
    q_i=\frac{p_i}{M_S}
    \quad(i\in S).
\]

For each nonempty $S\subseteq\{1,\ldots,E\}$ with $|S|\leq k$,
define
\[
    q_i^S:=
    \begin{cases}
        p_i/M_S,&i\in S,\\
        0,&i\notin S,
    \end{cases}
    \quad
    \mathbf q^S:=(q_1^S,\ldots,q_E^S).
\]
Then $\mathbf q^S\in\mathcal Q_k$ and
\begin{align*}
    D_{\mathrm{KL}}(\mathbf q^S\,\|\,\mathbf p)
    &=\sum_{i\in S}\frac{p_i}{M_S}
      \log\frac{p_i/M_S}{p_i}\\
    &=\sum_{i\in S}\frac{p_i}{M_S}
      \log\frac1{M_S}\\
    &=\frac{\sum_{i\in S}p_i}{M_S}
      \log\frac1{M_S}\\
    &=\log\frac1{M_S}\\
    &=-\log M_S.
\end{align*}
Consequently,
\[
    \min_{\substack{
        \mathbf q\in\mathcal Q_k\\
        \operatorname{supp}(\mathbf q)=S
    }}
    D_{\mathrm{KL}}(\mathbf q\,\|\,\mathbf p)
    =-\log M_S,
\]
and $\mathbf q^S$ is the unique minimizer on that support.

If $|S|<k$, choose $T\supset S$ with $|T|=k$. Positivity gives
\[
    M_T=M_S+\sum_{j\in T\setminus S}p_j
    >M_S>0.
\]
Hence
\[
    \begin{aligned}
        D_{\mathrm{KL}}(\mathbf q^T\,\|\,\mathbf p)
        &=-\log M_T\\
        &<-\log M_S\\
        &\leq D_{\mathrm{KL}}(\mathbf q\,\|\,\mathbf p).
    \end{aligned}
\]
It follows that
\begin{align*}
    \min_{\mathbf q\in\mathcal Q_k}
    D_{\mathrm{KL}}(\mathbf q\,\|\,\mathbf p)
    &=
    \min_{\substack{
        S\subseteq\{1,\ldots,E\}\\
        1\leq|S|\leq k
    }}
    \left(-\log M_S\right)\\
    &=
    \min_{\substack{
        S\subseteq\{1,\ldots,E\}\\
        |S|=k
    }}
    \left(-\log M_S\right).
\end{align*}

If $k=E$, the unique support of size $k$ is $\{1,\ldots,E\}$,
and
\[
    M_{\{1,\ldots,E\}}
    =\sum_{i=1}^{E}p_i=1,
    \quad
    q_i^*=\frac{p_i}{1}=p_i.
\]
Thus $\mathbf q^*=\mathbf p$ is the unique minimizer.

For $k<E$, choose a permutation $\pi$ of $\{1,\ldots,E\}$
such that
\[
    p_{\pi(1)}\geq\cdots\geq p_{\pi(E)}>0,
    \quad
    \mathcal I_k=\{\pi(1),\ldots,\pi(k)\}.
\]
The ordering within ties is arbitrary. For any $S$ with $|S|=k$,
define
\[
    A:=\mathcal I_k\setminus S,
    \quad
    B:=S\setminus\mathcal I_k.
\]
Their cardinalities satisfy
\begin{align*}
    |A|
    &=|\mathcal I_k|-|\mathcal I_k\cap S|\\
    &=k-|\mathcal I_k\cap S|\\
    &=|S|-|\mathcal I_k\cap S|\\
    &=|B|.
\end{align*}
Write $m:=|A|=|B|$. Since
\[
    p_i\geq p_{\pi(k)}
    \quad(i\in A),
    \quad
    p_j\leq p_{\pi(k+1)}
    \quad(j\in B),
\]
we have
\begin{align*}
    M_{\mathcal I_k}-M_S
    &=
    \left(
        \sum_{i\in\mathcal I_k\cap S}p_i
        +\sum_{i\in A}p_i
    \right)\\
    &\quad-
    \left(
        \sum_{i\in\mathcal I_k\cap S}p_i
        +\sum_{j\in B}p_j
    \right)\\
    &=\sum_{i\in A}p_i-\sum_{j\in B}p_j\\
    &\geq m p_{\pi(k)}-m p_{\pi(k+1)}\\
    &=m\left(p_{\pi(k)}-p_{\pi(k+1)}\right)\\
    &\geq0.
\end{align*}
Therefore,
\[
    \min_{\mathbf q\in\mathcal Q_k}
    D_{\mathrm{KL}}(\mathbf q\,\|\,\mathbf p)
    =-\log M_{\mathcal I_k},
\]
attained by
\[
    q_i^*=q_i^{\mathcal I_k}
    =
    \begin{cases}
        \displaystyle
        \frac{p_i}{\sum_{j\in\mathcal I_k}p_j},
            &i\in\mathcal I_k,\\
        0,&i\notin\mathcal I_k.
    \end{cases}
\]

If $p_{\pi(k)}>p_{\pi(k+1)}$ and $S\neq\mathcal I_k$,
then $m\geq1$ and
\[
    M_{\mathcal I_k}-M_S
    \geq
    m\left(p_{\pi(k)}-p_{\pi(k+1)}\right)
    >0.
\]
Thus
\[
    -\log M_S>-\log M_{\mathcal I_k}.
\]
The unique optimal support is $\mathcal I_k$, and its unique
optimal distribution is $\mathbf q^{\mathcal I_k}$.

If $p_{\pi(k)}=p_{\pi(k+1)}$, define
\[
    \mathcal I_k'
    :=
    \left(\mathcal I_k\setminus\{\pi(k)\}\right)
    \cup\{\pi(k+1)\}.
\]
Then
\begin{align*}
    M_{\mathcal I_k'}
    &=M_{\mathcal I_k}
      -p_{\pi(k)}+p_{\pi(k+1)}\\
    &=M_{\mathcal I_k},
\end{align*}
and
\[
    q_{\pi(k)}^{\mathcal I_k}
    =\frac{p_{\pi(k)}}{M_{\mathcal I_k}}>0,
    \quad
    q_{\pi(k)}^{\mathcal I_k'}=0.
\]
Both distributions attain $-\log M_{\mathcal I_k}$.
\end{proof}

\section{Proof of Theorem~\ref{thm:aux_loss}}
\label{sec:appendix_c}

\begin{proof}
Let $N,E\geq1$ and $1\leq k\leq E$ be integers. For each token
$t\in\{1,\ldots,N\}$, let
\[
    \begin{aligned}
    \mathbf p^{(t)}
    &=(p_1^{(t)},\ldots,p_E^{(t)})
    \in\Delta^{E-1},\\
    &\quad p_i^{(t)}>0
    \quad(i=1,\ldots,E),
    \end{aligned}
\]
and let $\mathcal I_k^{(t)}$ be a selected Top-$k$ support,
so $|\mathcal I_k^{(t)}|=k$. Define
\[
    \begin{aligned}
        \bar p_i
        &:=\frac1N\sum_{t=1}^{N}p_i^{(t)},\\
        f_i
        &:=\frac1N\sum_{t=1}^{N}
        \mathbf1\{i\in\mathcal I_k^{(t)}\},\\
        r_i&:=\frac{f_i}{k},
        \quad
        u_i:=\frac1E,
    \end{aligned}
\]
where $\mathbf1\{\cdot\}$ is the indicator function. Write $\bar{\mathbf p}:=(\bar p_1,\ldots,\bar p_E),\quad \mathbf r:=(r_1,\ldots,r_E),\quad \mathbf u:=(u_1,\ldots,u_E)$. Then
\begin{align*}
    \sum_{i=1}^{E}\bar p_i
    &=\sum_{i=1}^{E}
      \frac1N\sum_{t=1}^{N}p_i^{(t)}\\
    &=\frac1N\sum_{t=1}^{N}
      \sum_{i=1}^{E}p_i^{(t)}\\
    &=\frac1N\sum_{t=1}^{N}1\\
    &=1,\\[2pt]
    \sum_{i=1}^{E}f_i
    &=\frac1N\sum_{t=1}^{N}
      \sum_{i=1}^{E}
      \mathbf1\{i\in\mathcal I_k^{(t)}\}\\
    &=\frac1N\sum_{t=1}^{N}
      |\mathcal I_k^{(t)}|\\
    &=\frac1N\sum_{t=1}^{N}k\\
    &=k,\\[2pt]
    \sum_{i=1}^{E}r_i
    &=\frac1k\sum_{i=1}^{E}f_i\\
    &=\frac{k}{k}\\
    &=1,\\[2pt]
    \sum_{i=1}^{E}u_i
    &=\sum_{i=1}^{E}\frac1E\\
    &=\frac{E}{E}\\
    &=1.
\end{align*}
In particular, $\bar p_i>0$, $r_i\geq0$, and
\[
    \sum_{i=1}^{E}(r_i-\bar p_i)=1-1=0.
\]

Using the definitions
\[
    H_2(\bar{\mathbf p})
    :=-\log\left(\sum_{i=1}^{E}\bar p_i^2\right),
    \quad
    \mathcal L_{\mathrm{coll}}
    :=E\sum_{i=1}^{E}\bar p_i^2,
\]
we obtain
\begin{align*}
    E\exp[-H_2(\bar{\mathbf p})]
    &=E\exp\left[
        \log\left(\sum_{i=1}^{E}\bar p_i^2\right)
      \right]\\
    &=E\sum_{i=1}^{E}\bar p_i^2\\
    &=\mathcal L_{\mathrm{coll}}.
\end{align*}

For a vector $\mathbf v=(v_1,\ldots,v_E)$, write
$\|\mathbf v\|_2:=(\sum_{i=1}^{E}v_i^2)^{1/2}$.
Then
\begin{align*}
    1+E\|\bar{\mathbf p}-\mathbf u\|_2^2
    &=1+E\sum_{i=1}^{E}
      \left(\bar p_i-\frac1E\right)^2\\
    &=1+E\sum_{i=1}^{E}
      \left(
          \bar p_i^2-\frac{2\bar p_i}{E}
          +\frac1{E^2}
      \right)\\
    &=1+E\sum_{i=1}^{E}\bar p_i^2
      -2\sum_{i=1}^{E}\bar p_i
      +E\sum_{i=1}^{E}\frac1{E^2}\\
    &=1+\mathcal L_{\mathrm{coll}}
      -2+E\frac{E}{E^2}\\
    &=1+\mathcal L_{\mathrm{coll}}-2+1\\
    &=\mathcal L_{\mathrm{coll}}.
\end{align*}
Thus
\[
    \mathcal L_{\mathrm{coll}}
    =E\exp[-H_2(\bar{\mathbf p})]
    =1+E\|\bar{\mathbf p}-\mathbf u\|_2^2.
\]

Let $c:=\sum_{i=1}^{E}\bar p_i^2>0$.
Applying Eq.~\eqref{eq:appendix_b_log_bound} to
$a=\bar p_i/c>0$ gives
\begin{align*}
    D_{\mathrm{KL}}(\bar{\mathbf p}\,\|\,\mathbf u)
    &=\sum_{i=1}^{E}
      \bar p_i\log\frac{\bar p_i}{1/E}\\
    &=\sum_{i=1}^{E}
      \bar p_i\log(E\bar p_i)\\
    &=\sum_{i=1}^{E}\bar p_i
      \left[
          \log(Ec)+\log\frac{\bar p_i}{c}
      \right]\\
    &=\left(\sum_{i=1}^{E}\bar p_i\right)
      \log(Ec)
      +\sum_{i=1}^{E}
      \bar p_i\log\frac{\bar p_i}{c}\\
    &=\log(Ec)
      +\sum_{i=1}^{E}
      \bar p_i\log\frac{\bar p_i}{c}\\
    &\leq\log(Ec)
      +\sum_{i=1}^{E}\bar p_i
      \left(\frac{\bar p_i}{c}-1\right)\\
    &=\log(Ec)
      +\frac1c\sum_{i=1}^{E}\bar p_i^2
      -\sum_{i=1}^{E}\bar p_i\\
    &=\log(Ec)+\frac{c}{c}-1\\
    &=\log(Ec)\\
    &=\log\mathcal L_{\mathrm{coll}}.
\end{align*}

By definition,
\[
    \mathcal L_{\mathrm{aux}}
    :=E\sum_{i=1}^{E}f_i\bar p_i.
\]
Using $r_i=f_i/k$,
\begin{align*}
    \frac{\mathcal L_{\mathrm{aux}}}{k}
    -\mathcal L_{\mathrm{coll}}
    &=\frac{E}{k}\sum_{i=1}^{E}f_i\bar p_i
      -E\sum_{i=1}^{E}\bar p_i^2\\
    &=E\sum_{i=1}^{E}\frac{f_i}{k}\bar p_i
      -E\sum_{i=1}^{E}\bar p_i^2\\
    &=E\sum_{i=1}^{E}r_i\bar p_i
      -E\sum_{i=1}^{E}\bar p_i^2\\
    &=E\sum_{i=1}^{E}(r_i-\bar p_i)\bar p_i.
\end{align*}
Since
\begin{align*}
    \sum_{i=1}^{E}(r_i-\bar p_i)u_i
    &=\frac1E\sum_{i=1}^{E}(r_i-\bar p_i)\\
    &=\frac1E
      \left(
          \sum_{i=1}^{E}r_i
          -\sum_{i=1}^{E}\bar p_i
      \right)\\
    &=\frac1E(1-1)\\
    &=0,
\end{align*}
we have
\begin{align*}
    E\sum_{i=1}^{E}(r_i-\bar p_i)\bar p_i
    &=E\sum_{i=1}^{E}(r_i-\bar p_i)
      [(\bar p_i-u_i)+u_i]\\
    &=E\sum_{i=1}^{E}
      (r_i-\bar p_i)(\bar p_i-u_i)\\
    &\quad+
      E\sum_{i=1}^{E}(r_i-\bar p_i)u_i\\
    &=E\sum_{i=1}^{E}
      (r_i-\bar p_i)(\bar p_i-u_i)\\
    &=E(\mathbf r-\bar{\mathbf p})^\top
      (\bar{\mathbf p}-\mathbf u).
\end{align*}
Therefore,
\[
    \frac{\mathcal L_{\mathrm{aux}}}{k}
    =\mathcal L_{\mathrm{coll}}
     +E(\mathbf r-\bar{\mathbf p})^\top
       (\bar{\mathbf p}-\mathbf u).
\]
Finally, the Cauchy--Schwarz inequality yields
\begin{align*}
    \left|
        \frac{\mathcal L_{\mathrm{aux}}}{k}
        -\mathcal L_{\mathrm{coll}}
    \right|
    &=E\left|
      \sum_{i=1}^{E}
      (r_i-\bar p_i)(\bar p_i-u_i)
      \right|\\
    &\leq E
      \left(
          \sum_{i=1}^{E}(r_i-\bar p_i)^2
      \right)^{1/2}\\
    &\quad\cdot
      \left(
          \sum_{i=1}^{E}(\bar p_i-u_i)^2
      \right)^{1/2}\\
    &=E\|\mathbf r-\bar{\mathbf p}\|_2
      \|\bar{\mathbf p}-\mathbf u\|_2.
\end{align*}
This completes the proof.
\end{proof}

\section{Bayesian Interpretation of Standard MoE Training}
\label{app:aggregated_prior}

In this section, we relate the data-fitting, routing, and
load-balancing components of standard MoE training to the
conditional-mixture formulation. We then derive an
ELBO--mutual-information identity for dense and sparse routing
and identify the marginal-entropy term associated with a
uniform prior.

\subsection{From Theory to Standard MoE Implementation}
\label{app:standard_moe_implementation}

We relate the Bayesian formulation in
Section~\ref{sec:bayesian_perspective} to the standard
operational objective
\[
    \mathcal L_{\mathrm{total}}
    =\mathcal L_{\mathrm{task}}
     +\alpha\mathcal L_{\mathrm{aux}},
\]
where $\mathcal L_{\mathrm{task}}$ is the task loss,
$\mathcal L_{\mathrm{aux}}$ is the auxiliary load-balancing loss,
and $\alpha\geq0$ is its coefficient, as defined in
Section~\ref{sec:preliminaries}.
For a batch
$\mathcal B=\{(\mathbf x^{(t)},\mathbf y^{(t)})\}_{t=1}^{N}$
of $N\geq1$ input--target pairs, the aggregated dense
routing distribution is
\[
    \bar Q_\phi(z)
    :=\frac1N\sum_{t=1}^{N}
      Q_\phi(z\mid\mathbf x^{(t)}).
\]
Define the aggregated-prior objective
\begin{equation}
    \begin{split}
        \mathcal J_{\mathrm{agg}}:={}&
        \frac1N\sum_{t=1}^{N}
        \mathbb E_{z\sim Q_\phi(\cdot\mid\mathbf x^{(t)})}
        [\log P_\theta(
          \mathbf y^{(t)}\mid\mathbf x^{(t)},z)]\\
        &-D_{\mathrm{KL}}(\bar Q_\phi\,\|\,P).
    \end{split}
    \label{eq:aggregated_prior_surrogate}
\end{equation}
By Eq.~\eqref{eq:prior_matching_decomposition},
\begin{equation}
    \begin{split}
        \mathcal J_{\mathrm{agg}}
        ={}&\frac1N\sum_{t=1}^{N}
          \mathcal L_{\mathrm{ELBO}}(
            \mathbf x^{(t)},\mathbf y^{(t)})\\
        &+I_{\mathcal B,\phi}(X;Z).
    \end{split}
    \label{eq:aggregated_elbo_mi_dense}
\end{equation}
Thus, $\mathcal J_{\mathrm{agg}}$ combines the average ELBO
with empirical input--expert mutual information.
The same identity holds for $Q_{\phi,k}$ and its marginal
\[
    \bar Q_{\phi,k}(z)
    :=\frac1N\sum_{t=1}^{N}
      Q_{\phi,k}(z\mid\mathbf x^{(t)}).
\]
The following correspondences connect this formulation
to standard MoE training.

\begin{itemize}
    \item \textbf{Likelihood / Task Loss:}
    The negative expected log-likelihood supplies the
    data-fitting term in the conditional mixture model.
    In standard MoE, $\mathcal L_{\mathrm{task}}$ evaluates
    the prediction formed from the active expert features.
    These terms identify the data-fitting role in their
    respective formulations.

    \item \textbf{Variational Distribution / Top-$k$ Operator:}
    The sparse distribution $Q_{\phi,k}$ is obtained by
    projecting the dense routing distribution $Q_\phi$
    onto $\mathcal Q_k$, as established in
    Theorem~\ref{thm:topk_optimality}.
    For implementations using renormalized Top-$k$
    probabilities as gates, this projection supplies
    the active gating distribution defined in
    Eq.~\eqref{eq:topk_normalization}.

    \item \textbf{Prior / Uniform Distribution:}
    The prior $P=\mathcal U_E$ represents the dataset-level
    goal of balanced marginal expert utilization,
    while allowing the conditional routing distribution
    to depend on the input.

    \item \textbf{KL Divergence / Auxiliary Loss:}
    Under $P=\mathcal U_E$, the logarithm of the collision
    surrogate in Theorem~\ref{thm:aux_loss} upper-bounds
    $D_{\mathrm{KL}}(\bar Q_\phi\,\|\,P)$.
    The surrogate and the marginal KL divergence share
    the same uniform minimizer.
    The theorem quantifies its approximation to
    $\mathcal L_{\mathrm{aux}}/k$ through the dispatch
    mismatch $\mathbf r-\bar{\mathbf p}$, providing
    a marginal-prior interpretation of the auxiliary-loss
    value with an explicit approximation bound.
\end{itemize}

Together, sparse reverse-KL projection and aggregated prior
matching provide a unified Bayesian interpretation of
routing sparsity and aggregate expert utilization.
The following subsections establish the aggregated-prior
identity for both dense and sparse routing.

\subsection{Derivation of the Aggregated-Prior Objective}
\label{app:aggregated_prior_derivation}

We use a common conditional routing distribution to establish
the relationship between per-input prior regularization,
marginal prior matching, and input--expert mutual information.

Let
$\mathcal B
=\{(\mathbf x^{(t)},\mathbf y^{(t)})\}_{t=1}^{N}$
be a batch of $N\geq1$ input--target pairs, and let $E\geq2$
be the number of experts.
Let $R(i\mid\mathbf x)$ be a categorical conditional
distribution on $\{1,\dots,E\}$.
The choices $R=Q_\phi$ and $R=Q_{\phi,k}$ give dense and
sparse routing, respectively; $Q_{\phi,k}$ uses the
definition in Eq.~\eqref{eq:topk_normalization}.
Let $P$ be a categorical prior and set $P_i:=P(i)>0$.
Define
\begin{align*}
    R_t(i)
    &:=R(i\mid\mathbf x^{(t)}),\\
    \bar R(i)
    &:=\frac1N\sum_{t=1}^{N}R_t(i),\\
    \mathcal A_t
    &:=\{i:R_t(i)>0\},\\
    \mathcal A
    &:=\{i:\bar R(i)>0\}.
\end{align*}
These sets collect the experts with positive conditional
and marginal routing probabilities, respectively.
Every expert in $\mathcal A_t$ also has positive marginal
probability, since
\[
    \bar R(i)
    =\frac1N\sum_{s=1}^{N}R_s(i)
    \geq\frac{R_t(i)}{N}>0
\]
for $i\in\mathcal A_t$.
Thus, $\mathcal A_t\subseteq\mathcal A$, and the ratios
$R_t(i)/\bar R(i)$ are well-defined on $\mathcal A_t$.
All logarithms are natural, and sums of log-ratios are
taken over positive-probability supports.

To express the batch average as mutual information,
we let $X$ follow the empirical distribution of the inputs.
Repeated inputs receive weights proportional to their
occurrence counts.
Let $\mathcal V:=\{\mathbf x^{(t)}:1\leq t\leq N\}$
be the set of distinct batch inputs, and define
\[
    w(\mathbf x)
    :=\frac1N\sum_{t=1}^{N}
      \mathbf 1\{\mathbf x^{(t)}=\mathbf x\},
    \quad \mathbf x\in\mathcal V.
\]
The empirical joint distribution is
\[
    \Pr_R(X=\mathbf x,Z=i)
    :=w(\mathbf x)R(i\mid\mathbf x).
\]
Its marginals satisfy
\begin{align*}
    \Pr_R(X=\mathbf x)
    &=\sum_{i=1}^{E}
      w(\mathbf x)R(i\mid\mathbf x)\\
    &=w(\mathbf x),\\
    \Pr_R(Z=i)
    &=\sum_{\mathbf x\in\mathcal V}
      w(\mathbf x)R(i\mid\mathbf x)\\
    &=\frac1N\sum_{\mathbf x\in\mathcal V}
      \sum_{\{t:\mathbf x^{(t)}=\mathbf x\}}
      R(i\mid\mathbf x)\\
    &=\frac1N\sum_{t=1}^{N}R_t(i)\\
    &=\bar R(i).
\end{align*}
Under this joint distribution, mutual information equals
the batch-averaged divergence between conditional routing
and the routing marginal:
\begin{align*}
    &I_{\mathcal B,R}(X;Z)\\
    &:=
      \sum_{\mathbf x\in\mathcal V}w(\mathbf x)
      \sum_{\{i:R(i\mid\mathbf x)>0\}}
      R(i\mid\mathbf x)
      \log
      \frac{w(\mathbf x)R(i\mid\mathbf x)}
           {w(\mathbf x)\bar R(i)}\\
    &=
      \sum_{\mathbf x\in\mathcal V}w(\mathbf x)
      \sum_{\{i:R(i\mid\mathbf x)>0\}}
      R(i\mid\mathbf x)
      \log\frac{R(i\mid\mathbf x)}{\bar R(i)}\\
    &=\frac1N\sum_{t=1}^{N}
      \sum_{i\in\mathcal A_t}
      R_t(i)\log\frac{R_t(i)}{\bar R(i)}\\
    &=\frac1N\sum_{t=1}^{N}
      D_{\mathrm{KL}}(R_t\,\|\,\bar R).
\end{align*}

We next separate the likelihood contribution from prior
regularization.
For $i\in\mathcal A_t$, assume
\[
    \ell_{ti}
    :=\log P_\theta(
        \mathbf y^{(t)}\mid\mathbf x^{(t)},i)
    \in\mathbb R.
\]
Define the average log-likelihood term and
per-example ELBO by
\begin{align*}
    \mathcal A_R
    &:=\frac1N\sum_{t=1}^{N}
      \sum_{i\in\mathcal A_t}R_t(i)\ell_{ti},\\
    \mathcal L_{\mathrm{ELBO}}^R(t)
    &:=\sum_{i\in\mathcal A_t}R_t(i)\ell_{ti}
      -D_{\mathrm{KL}}(R_t\,\|\,P).
\end{align*}
The aggregated-prior objective is
\begin{equation}
    \mathcal J_{\mathrm{agg}}(R)
    :=\mathcal A_R
      -D_{\mathrm{KL}}(\bar R\,\|\,P).
    \label{eq:aggregated_prior_general}
\end{equation}

Splitting each log-ratio through $\bar R$ decomposes
average per-input prior regularization into input--expert
mutual information and marginal prior mismatch:
\begin{align*}
    &\frac1N\sum_{t=1}^{N}
      D_{\mathrm{KL}}(R_t\,\|\,P)\\
    &\quad=
      \frac1N\sum_{t=1}^{N}
      \sum_{i\in\mathcal A_t}
      R_t(i)\log\frac{R_t(i)}{P_i}\\
    &\quad=
      \frac1N\sum_{t=1}^{N}
      \sum_{i\in\mathcal A_t}
      R_t(i)
      \left(
          \log\frac{R_t(i)}{\bar R(i)}
          +\log\frac{\bar R(i)}{P_i}
      \right)\\
    &\quad=
      I_{\mathcal B,R}(X;Z)
      +\frac1N\sum_{t=1}^{N}
       \sum_{i\in\mathcal A_t}
       R_t(i)\log\frac{\bar R(i)}{P_i}\\
    &\quad=
      I_{\mathcal B,R}(X;Z)
      +\sum_{i\in\mathcal A}
       \left(
           \frac1N
           \sum_{\{t:i\in\mathcal A_t\}}R_t(i)
       \right)
       \log\frac{\bar R(i)}{P_i}\\
    &\quad=
      I_{\mathcal B,R}(X;Z)
      +\sum_{i\in\mathcal A}
       \left(
           \frac1N\sum_{t=1}^{N}R_t(i)
       \right)
       \log\frac{\bar R(i)}{P_i}\\
    &\quad=
      I_{\mathcal B,R}(X;Z)
      +\sum_{i\in\mathcal A}
       \bar R(i)\log\frac{\bar R(i)}{P_i}\\
    &\quad=
      I_{\mathcal B,R}(X;Z)
      +D_{\mathrm{KL}}(\bar R\,\|\,P).
\end{align*}
For $R=Q_\phi$, this gives
Eq.~\eqref{eq:prior_matching_decomposition}.
The two terms quantify input-dependent expert selection
and aggregate deviation from the prior, respectively.

Substituting this decomposition into
Eq.~\eqref{eq:aggregated_prior_general} expresses the
aggregated-prior objective as the average ELBO plus
empirical mutual information:
\begin{align*}
    &\mathcal J_{\mathrm{agg}}(R)\\
    &=\mathcal A_R
      -D_{\mathrm{KL}}(\bar R\,\|\,P)\\
    &=\mathcal A_R
      -\left(
          \frac1N\sum_{t=1}^{N}
          D_{\mathrm{KL}}(R_t\,\|\,P)
          -I_{\mathcal B,R}(X;Z)
      \right)\\
    &=\mathcal A_R
      -\frac1N\sum_{t=1}^{N}
       D_{\mathrm{KL}}(R_t\,\|\,P)
      +I_{\mathcal B,R}(X;Z)\\
    &=\frac1N\sum_{t=1}^{N}
      \left(
          \sum_{i\in\mathcal A_t}R_t(i)\ell_{ti}
          -D_{\mathrm{KL}}(R_t\,\|\,P)
      \right)\\
      &\quad+I_{\mathcal B,R}(X;Z)\\
    &=\frac1N\sum_{t=1}^{N}
      \mathcal L_{\mathrm{ELBO}}^R(t)
      +I_{\mathcal B,R}(X;Z).
\end{align*}
Thus,
\begin{equation}
    \boxed{
    \mathcal J_{\mathrm{agg}}(R)
    =\frac1N\sum_{t=1}^{N}
      \mathcal L_{\mathrm{ELBO}}^R(t)
     +I_{\mathcal B,R}(X;Z).
    }
    \label{eq:aggregated_elbo_mi_identity}
\end{equation}

For $R=Q_\phi$, we have
$\mathcal J_{\mathrm{agg}}(R)=\mathcal J_{\mathrm{agg}}$,
$\bar R=\bar Q_\phi$,
$I_{\mathcal B,R}=I_{\mathcal B,\phi}$, and
\[
    \mathcal L_{\mathrm{ELBO}}^R(t)
    =\mathcal L_{\mathrm{ELBO}}(
        \mathbf x^{(t)},\mathbf y^{(t)}),
\]
which proves Eq.~\eqref{eq:aggregated_elbo_mi_dense}.
Substituting $R=Q_{\phi,k}$ gives the sparse-routing
identity with the same support-restricted sums.

For a uniform prior, marginal prior matching has a
direct entropy interpretation.
Let $P=\mathcal U_E$, with $\mathcal U_E(i)=1/E$,
and define
\[
    H(\bar R)
    :=-\sum_{i\in\mathcal A}\bar R(i)\log\bar R(i).
\]
Then
\begin{align*}
    D_{\mathrm{KL}}(\bar R\,\|\,\mathcal U_E)
    &=\sum_{i\in\mathcal A}
      \bar R(i)\log\frac{\bar R(i)}{1/E}\\
    &=\sum_{i\in\mathcal A}
      \bar R(i)\log\bar R(i)
      +\log E\sum_{i\in\mathcal A}\bar R(i)\\
    &=-H(\bar R)+\log E,
\end{align*}
and therefore
\begin{align*}
    \mathcal J_{\mathrm{agg}}(R)
    &=\mathcal A_R
      -D_{\mathrm{KL}}(\bar R\,\|\,\mathcal U_E)\\
    &=\mathcal A_R+H(\bar R)-\log E.
\end{align*}
The uniform-prior objective therefore combines the
expected log-likelihood with marginal routing entropy,
up to the constant $-\log E$.
This entropy term encourages balanced aggregate expert
utilization while allowing input-dependent routing.

\section{Proof of Theorem~\ref{thm:cond_entropy}}
\label{sec:appendix_d}

\begin{proof}
Let $E\geq2$ and $1\leq k<E$ be integers. All logarithms are natural.
Fix an input $\mathbf x$ and write
\[
    p_i:=p(i\mid\mathbf x),
    \quad
    q_i:=q_k(i\mid\mathbf x),
\]
\[
    \mathbf p:=(p_1,\ldots,p_E),
    \quad
    \mathbf q:=(q_1,\ldots,q_E).
\]
Define
\[
    S:=\mathcal I_k(\mathbf x),
    \quad
    S^c:=\{1,\ldots,E\}\setminus S,
    \quad
    M:=\sum_{i\in S}p_i.
\]
The assumptions give
\[
    |S|=k,\quad
    p_i>0,\quad
    \sum_{i=1}^{E}p_i=1,\quad
    0<M<1,
\]
\[
    \sum_{j\in S^c}p_j=1-M,
    \quad
    p_i\geq p_j
    \quad(i\in S,\ j\in S^c),
\]
and
\[
    q_i=
    \begin{cases}
        p_i/M,&i\in S,\\
        0,&i\in S^c.
    \end{cases}
\]
Hence
\[
    \sum_{i\in S}q_i
    =\frac1M\sum_{i\in S}p_i
    =\frac{M}{M}=1,
    \quad
    0<q_i\leq1\quad(i\in S).
\]
Using $H(\mathbf a):=-\sum_i a_i\log a_i$ with $0\log0:=0$,
\[
    H(\mathbf q)
    =-\sum_{i\in S}q_i\log q_i
    \geq0.
\]

By $\log a\leq a-1$ for $a>0$,
\begin{align*}
    \log k-H(\mathbf q)
    &=\log k+\sum_{i\in S}q_i\log q_i\\
    &=\left(\sum_{i\in S}q_i\right)\log k
      +\sum_{i\in S}q_i\log q_i\\
    &=\sum_{i\in S}q_i(\log k+\log q_i)\\
    &=\sum_{i\in S}q_i\log(kq_i)\\
    &=-\sum_{i\in S}q_i\log\frac1{kq_i}\\
    &\geq-\sum_{i\in S}q_i
      \left(\frac1{kq_i}-1\right)\\
    &=-\sum_{i\in S}\frac1k+\sum_{i\in S}q_i\\
    &=-\frac{|S|}{k}+1\\
    &=-\frac{k}{k}+1\\
    &=0.
\end{align*}

Furthermore,
\begin{align*}
    H(\mathbf q)
    &=-\sum_{i\in S}\frac{p_i}{M}\log\frac{p_i}{M}\\
    &=-\frac1M\sum_{i\in S}p_i(\log p_i-\log M)\\
    &=-\frac1M\sum_{i\in S}p_i\log p_i
      +\frac{\log M}{M}\sum_{i\in S}p_i\\
    &=-\frac1M\sum_{i\in S}p_i\log p_i+\log M.
\end{align*}
Therefore,
\begin{align*}
    H(\mathbf p)-H(\mathbf q)
    &=-\sum_{i=1}^{E}p_i\log p_i\\
    &\quad+\frac1M\sum_{i\in S}p_i\log p_i-\log M\\
    &=\left(\frac1M-1\right)
      \sum_{i\in S}p_i\log p_i\\
    &\quad-\sum_{j\in S^c}p_j\log p_j-\log M\\
    &=\frac{1-M}{M}\sum_{i\in S}p_i\log p_i\\
    &\quad-\sum_{j\in S^c}p_j\log p_j-\log M\\
    &=\frac1M
      \left(\sum_{j\in S^c}p_j\right)
      \left(\sum_{i\in S}p_i\log p_i\right)\\
    &\quad-\frac1M
      \left(\sum_{i\in S}p_i\right)
      \left(\sum_{j\in S^c}p_j\log p_j\right)\\
      &\quad-\log M\\
    &=\frac1M\sum_{i\in S}\sum_{j\in S^c}
      p_ip_j\log p_i\\
    &\quad-\frac1M\sum_{i\in S}\sum_{j\in S^c}
      p_ip_j\log p_j-\log M\\
    &=\frac1M\sum_{i\in S}\sum_{j\in S^c}
      p_ip_j(\log p_i-\log p_j)\\
    &\quad-\log M\\
    &=\frac1M\sum_{i\in S}\sum_{j\in S^c}
      p_ip_j\log\frac{p_i}{p_j}-\log M.
\end{align*}
For every $i\in S$ and $j\in S^c$,
\[
    \frac{p_i}{p_j}\geq1,
    \quad
    \log\frac{p_i}{p_j}\geq0,
    \quad
    p_ip_j\log\frac{p_i}{p_j}\geq0.
\]
Consequently,
\[
    H(\mathbf p)-H(\mathbf q)
    \geq-\log M>0.
\]
Combining the bounds gives
\[
    0\leq H\!\left(q_k(\cdot\mid\mathbf x)\right)
    \leq
    \min\!\left\{
        H\!\left(p(\cdot\mid\mathbf x)\right),
        \log k
    \right\}.
\]

Let $\{\mathbf x^{(t)}\}_{t=1}^{N}$ be the inputs in a batch
of size $N\geq1$, and let $\mathbf X$ follow their empirical
distribution, with
$Z\mid\mathbf X=\mathbf x\sim q_k(\cdot\mid\mathbf x)$.
Then
\begin{align*}
    0\leq H(Z\mid\mathbf X)
    &=\frac1N\sum_{t=1}^{N}
      H\!\left(q_k(\cdot\mid\mathbf x^{(t)})\right)\\
    &\leq\frac1N\sum_{t=1}^{N}\log k\\
    &=\frac{N}{N}\log k\\
    &=\log k.
\end{align*}

For $k=1$, let $j(\mathbf x)$ be the unique index in the
selected support $\mathcal I_1(\mathbf x)$. Then
\[
    q_1(i\mid\mathbf x)
    =
    \begin{cases}
        \displaystyle
        \frac{p(j(\mathbf x)\mid\mathbf x)}
             {p(j(\mathbf x)\mid\mathbf x)}
        =1,
            &i=j(\mathbf x),\\
        0,&i\neq j(\mathbf x).
    \end{cases}
\]
Thus
\[
    H\!\left(q_1(\cdot\mid\mathbf x)\right)
    =-1\log1=0,
\]
\[
    H(Z\mid\mathbf X)
    =\frac1N\sum_{t=1}^{N}0=0,
\]
and
\[
    \Pr\!\left(Z=j(\mathbf X)\mid\mathbf X\right)=1,
\]
\[
    \Pr\!\left(Z=j(\mathbf X)\right)
    =\mathbb E_{\mathbf X}[1]=1.
\]
This completes the proof.
\end{proof}

\section{Proof of Theorem~\ref{thm:capacity}}
\label{sec:appendix_e}

\begin{proof}
Let $N\geq1$, $E\geq2$, and $1\leq k<E$ be integers.
For the batch inputs $\{\mathbf x^{(t)}\}_{t=1}^{N}$, define
\[
    \mathcal V
    :=\{\mathbf x^{(t)}:1\leq t\leq N\},
\]
\[
    P_{\mathbf X}(\mathbf x)
    :=\frac1N\sum_{t=1}^{N}
      \mathbf1\{\mathbf x^{(t)}=\mathbf x\},
\]
where $\mathbf1\{\cdot\}$ is the indicator function.
Then $P_{\mathbf X}(\mathbf x)>0$ for every
$\mathbf x\in\mathcal V$, and
\begin{align*}
    \sum_{\mathbf x\in\mathcal V}P_{\mathbf X}(\mathbf x)
    &=\frac1N
      \sum_{\mathbf x\in\mathcal V}\sum_{t=1}^{N}
      \mathbf1\{\mathbf x^{(t)}=\mathbf x\}\\
    &=\frac1N
      \sum_{t=1}^{N}\sum_{\mathbf x\in\mathcal V}
      \mathbf1\{\mathbf x^{(t)}=\mathbf x\}\\
    &=\frac1N\sum_{t=1}^{N}1\\
    &=1.
\end{align*}
For every real-valued function $F$ on $\mathcal V$,
\begin{align*}
    \mathbb E_{\mathbf X}[F(\mathbf X)]
    &=\sum_{\mathbf x\in\mathcal V}
      P_{\mathbf X}(\mathbf x)F(\mathbf x)\\
    &=\frac1N
      \sum_{\mathbf x\in\mathcal V}\sum_{t=1}^{N}
      \mathbf1\{\mathbf x^{(t)}=\mathbf x\}F(\mathbf x)\\
    &=\frac1N
      \sum_{t=1}^{N}\sum_{\mathbf x\in\mathcal V}
      \mathbf1\{\mathbf x^{(t)}=\mathbf x\}F(\mathbf x)\\
    &=\frac1N\sum_{t=1}^{N}F(\mathbf x^{(t)}).
\end{align*}

Write
\[
\begin{aligned}
    & q_i(\mathbf x):=q_k(i\mid\mathbf x),
    \quad q_{ti}:=q_i(\mathbf x^{(t)}), \\
    & S_t:=\{i:q_{ti}>0\},
\end{aligned}
\]
\[
\begin{aligned}
    & b_i:=\bar Q_k(i)
    =\frac1N\sum_{t=1}^{N}q_{ti}, \\
    & \mathbf b:=(b_1,\ldots,b_E)=\bar{\mathbf q}_k,
\end{aligned}
\]
\[
\begin{aligned}
    & A:=\{i:b_i>0\},
    \quad u_i:=\frac1E, \\
    & \mathbf u:=(u_1,\ldots,u_E).
\end{aligned}
\]
Since $q_{ti}\geq0$ and $\sum_{i=1}^{E}q_{ti}=1$,
\begin{align*}
    \sum_{i=1}^{E}b_i
    &=\sum_{i=1}^{E}
      \frac1N\sum_{t=1}^{N}q_{ti}\\
    &=\frac1N\sum_{t=1}^{N}\sum_{i=1}^{E}q_{ti}\\
    &=\frac1N\sum_{t=1}^{N}1\\
    &=1.
\end{align*}
Moreover,
\[
    b_i=0
    \iff\sum_{t=1}^{N}q_{ti}=0
    \iff q_{ti}=0\quad\text{for every }t,
\]
and
\[
    q_{ti}>0
    \implies
    b_i=\frac1N\sum_{s=1}^{N}q_{si}
    \geq\frac{q_{ti}}N>0.
\]
Thus
\[
    S_t\subseteq A,
    \quad
    \sum_{i\in A}b_i=1,
    \quad
    \sum_{i\in A}q_{ti}=1.
\]

Let $Z\mid\mathbf X=\mathbf x\sim q_k(\cdot\mid\mathbf x)$.
The joint distribution is
\[
    P_{\mathbf X,Z}(\mathbf x,i)
    =P_{\mathbf X}(\mathbf x)q_i(\mathbf x).
\]
Its expert marginal satisfies
\begin{align*}
    \Pr(Z=i)
    &=\sum_{\mathbf x\in\mathcal V}
      P_{\mathbf X}(\mathbf x)q_i(\mathbf x)\\
    &=\frac1N\sum_{t=1}^{N}q_i(\mathbf x^{(t)})\\
    &=b_i.
\end{align*}
Using $0\log0:=0$, define
\[
    H(\mathbf b):=-\sum_{i\in A}b_i\log b_i,
\]
\[
    H(\mathbf q^{(t)})
    :=-\sum_{i\in S_t}q_{ti}\log q_{ti},
    \quad
    \mathbf q^{(t)}:=(q_{t1},\ldots,q_{tE}).
\]
Then
\[
    H(Z)=H(\mathbf b),
    \quad
    H(Z\mid\mathbf X)
    =\frac1N\sum_{t=1}^{N}H(\mathbf q^{(t)}).
\]

By the definition of mutual information,
\begin{align*}
    &I(\mathbf X;Z)\\
    &=\sum_{\mathbf x\in\mathcal V}
      \sum_{i:q_i(\mathbf x)>0}
      P_{\mathbf X}(\mathbf x)q_i(\mathbf x)
      \log
      \frac{P_{\mathbf X}(\mathbf x)q_i(\mathbf x)}
           {P_{\mathbf X}(\mathbf x)b_i}\\
    &=\sum_{\mathbf x\in\mathcal V}
      P_{\mathbf X}(\mathbf x)
      \sum_{i:q_i(\mathbf x)>0}
      q_i(\mathbf x)\log\frac{q_i(\mathbf x)}{b_i}\\
    &=\frac1N\sum_{t=1}^{N}\sum_{i\in S_t}
      q_{ti}\log\frac{q_{ti}}{b_i}\\
    &=\frac1N\sum_{t=1}^{N}\sum_{i\in S_t}
      q_{ti}\log q_{ti}\\
    &\quad-\frac1N\sum_{t=1}^{N}\sum_{i\in S_t}
      q_{ti}\log b_i\\
    &=-\frac1N\sum_{t=1}^{N}H(\mathbf q^{(t)})
      -\frac1N\sum_{t=1}^{N}\sum_{i\in A}
      q_{ti}\log b_i\\
    &=-H(Z\mid\mathbf X)
      -\sum_{i\in A}
      \left(\frac1N\sum_{t=1}^{N}q_{ti}\right)
      \log b_i\\
    &=-H(Z\mid\mathbf X)
      -\sum_{i\in A}b_i\log b_i\\
    &=H(\mathbf b)-H(Z\mid\mathbf X)\\
    &=H(Z)-H(Z\mid\mathbf X).
\end{align*}
Theorem~\ref{thm:cond_entropy} therefore gives
\[
    I(\mathbf X;Z)
    =H(Z)-H(Z\mid\mathbf X)
    \geq H(Z)-\log k.
\]

For the uniform vector $\mathbf u$,
\begin{align*}
    D_{\mathrm{KL}}(\mathbf b\,\|\,\mathbf u)
    &=\sum_{i\in A}b_i\log\frac{b_i}{1/E}\\
    &=\sum_{i\in A}b_i\log(Eb_i)\\
    &=\sum_{i\in A}b_i(\log E+\log b_i)\\
    &=\left(\sum_{i\in A}b_i\right)\log E
      +\sum_{i\in A}b_i\log b_i\\
    &=\log E-H(\mathbf b)\\
    &=\log E-H(Z).
\end{align*}
Consequently,
\[
    I(\mathbf X;Z)
    \geq
    \log E
    -D_{\mathrm{KL}}(\mathbf b\,\|\,\mathbf u)
    -\log k.
\]

Define
\[
    c:=\sum_{i=1}^{E}b_i^2
      =\sum_{i\in A}b_i^2>0,
    \quad
    \mathcal L_{\mathrm{coll},k}:=Ec.
\]
By $\log a\leq a-1$ for $a>0$,
\begin{align*}
    D_{\mathrm{KL}}(\mathbf b\,\|\,\mathbf u)
    -\log\mathcal L_{\mathrm{coll},k}
    &=\sum_{i\in A}b_i\log(Eb_i)-\log(Ec)\\
    &=\sum_{i\in A}b_i\log(Eb_i)\\
      &\quad-\left(\sum_{i\in A}b_i\right)\log(Ec)\\
    &=\sum_{i\in A}b_i
      [\log(Eb_i)-\log(Ec)]\\
    &=\sum_{i\in A}b_i\log\frac{b_i}{c}\\
    &\leq\sum_{i\in A}b_i
      \left(\frac{b_i}{c}-1\right)\\
    &=\frac1c\sum_{i\in A}b_i^2
      -\sum_{i\in A}b_i\\
    &=\frac{c}{c}-1\\
    &=0.
\end{align*}
Combining the inequalities yields
\begin{align*}
    I(\mathbf X;Z)
    &\geq H(Z)-\log k\\
    &=\log E
      -D_{\mathrm{KL}}(\mathbf b\,\|\,\mathbf u)
      -\log k\\
    &\geq
      \log E-\log\mathcal L_{\mathrm{coll},k}-\log k\\
    &=\log(E/k)-\log\mathcal L_{\mathrm{coll},k}.
\end{align*}

For $\mathbf v=(v_1,\ldots,v_E)\in\mathbb R^E$, let
$\|\mathbf v\|_2:=(\sum_{i=1}^{E}v_i^2)^{1/2}$.
Then
\begin{align*}
    E\|\mathbf b-\mathbf u\|_2^2
    &=E\sum_{i=1}^{E}
      \left(b_i-\frac1E\right)^2\\
    &=E\sum_{i=1}^{E}
      \left(
          b_i^2-\frac{2b_i}{E}+\frac1{E^2}
      \right)\\
    &=E\sum_{i=1}^{E}b_i^2
      -2\sum_{i=1}^{E}b_i
      +E\sum_{i=1}^{E}\frac1{E^2}\\
    &=\mathcal L_{\mathrm{coll},k}
      -2+E\frac{E}{E^2}\\
    &=\mathcal L_{\mathrm{coll},k}-2+1\\
    &=\mathcal L_{\mathrm{coll},k}-1.
\end{align*}
Thus
\[
    \mathcal L_{\mathrm{coll},k}
    =1+E\|\mathbf b-\mathbf u\|_2^2
    \geq1,
\]
and
\begin{align*}
    \mathcal L_{\mathrm{coll},k}=1
    &\iff E\|\mathbf b-\mathbf u\|_2^2=0\\
    &\iff
      \sum_{i=1}^{E}
      \left(b_i-\frac1E\right)^2=0\\
    &\iff b_i=\frac1E
      \quad\text{for every }i\\
    &\iff\mathbf b=\mathbf u.
\end{align*}

Let $m:=|A|$. Applying $\log a\leq a-1$
to $a=1/(Eb_i)$ gives
\begin{align*}
    -D_{\mathrm{KL}}(\mathbf b\,\|\,\mathbf u)
    &=\sum_{i\in A}b_i\log\frac1{Eb_i}\\
    &\leq\sum_{i\in A}b_i
      \left(\frac1{Eb_i}-1\right)\\
    &=\sum_{i\in A}\frac1E-\sum_{i\in A}b_i\\
    &=\frac{m}{E}-1\\
    &\leq0.
\end{align*}
If $D_{\mathrm{KL}}(\mathbf b\,\|\,\mathbf u)=0$, then
\[
    0\leq\frac mE-1\leq0
    \quad\Longrightarrow\quad
    m=E.
\]
For $m=E$ and $D_{\mathrm{KL}}(\mathbf b\,\|\,\mathbf u)=0$,
\begin{align*}
    &\sum_{i\in A}b_i
      \left[
          \frac1{Eb_i}-1-\log\frac1{Eb_i}
      \right]\\
    &\quad=
      \frac mE-1
      +D_{\mathrm{KL}}(\mathbf b\,\|\,\mathbf u)\\
    &\quad=1-1+0\\
    &\quad=0.
\end{align*}
Each bracket is nonnegative and every $b_i>0$ on $A$.
Since equality in $\log a\leq a-1$ holds precisely at $a=1$,
\[
    \frac1{Eb_i}=1
    \quad(i\in A),
\]
and hence
\[
    b_i=\frac1E
    \quad(i=1,\ldots,E).
\]
Conversely, if $\mathbf b=\mathbf u$, then
\[
    D_{\mathrm{KL}}(\mathbf b\,\|\,\mathbf u)
    =\sum_{i=1}^{E}\frac1E\log1=0.
\]
Therefore,
\begin{align*}
    H(Z)=\log E
    &\iff
      D_{\mathrm{KL}}(\mathbf b\,\|\,\mathbf u)=0\\
    &\iff\mathbf b=\mathbf u\\
    &\iff\mathcal L_{\mathrm{coll},k}=1.
\end{align*}

Define $\ell(s):=\log(E/k)-\log s$ for $s\geq1$.
For $s_1,s_2\geq1$,
\begin{align*}
    s_1\leq s_2
    &\iff\log s_1\leq\log s_2\\
    &\iff-\log s_1\geq-\log s_2\\
    &\iff\ell(s_1)\geq\ell(s_2).
\end{align*}
For any nonempty feasible family $\mathcal F$
of aggregated sparse marginals,
\[
    \begin{aligned}
        &\operatorname*{arg\,min}_{\mathbf b\in\mathcal F}
        E\sum_{i=1}^{E}b_i^2\\
        &\quad=
        \operatorname*{arg\,max}_{\mathbf b\in\mathcal F}
        \left[
            \log(E/k)
            -\log\left(E\sum_{i=1}^{E}b_i^2\right)
        \right].
    \end{aligned}
\]
Finally, $\mathbf b=\mathbf u$ gives
\begin{align*}
    I(\mathbf X;Z)
    &=H(Z)-H(Z\mid\mathbf X)\\
    &=\log E-H(Z\mid\mathbf X)\\
    &\geq\log E-\log k.
\end{align*}
This completes the proof.
\end{proof}

\section{Proof of Theorem~\ref{thm:np_hard}}
\label{sec:appendix_f}

\begin{proof}
We reduce from Exact Cover by 3-Sets (X3C). An X3C instance consists of
\[
    U=\{u_1,\ldots,u_{3m}\},
    \quad
    \mathcal C=(C_1,\ldots,C_n),
\]
\[
    C_j\subseteq U,
    \quad
    |C_j|=3,
\]
where $m\geq1$ and $n\geq0$ are integers. The question is whether
there exists $I\subseteq\{1,\ldots,n\}$ such that
\[
    |I|=m,
    \quad
    \sum_{j\in I}\mathbf1\{u_i\in C_j\}=1
    \quad(i=1,\ldots,3m),
\]
where $\mathbf1\{\cdot\}$ denotes the indicator function.
X3C is NP-complete.

Define the incidence matrix $A\in\{0,1\}^{3m\times n}$ by
\[
    A_{ij}:=\mathbf1\{u_i\in C_j\}.
\]
For every $j\in\{1,\ldots,n\}$,
\[
    \sum_{i=1}^{3m}A_{ij}
    =\sum_{i=1}^{3m}\mathbf1\{u_i\in C_j\}
    =|C_j|=3.
\]
Construct
\[
    d:=3m,
    \quad
    E:=n+m+1,
    \quad
    k:=m,
\]
\[
    \mathbf E:=[A,\mathbf0_{3m\times(m+1)}],
    \quad
    \mathbf y:=\mathbf1_{3m},
\]
where $\mathbf0_{3m\times(m+1)}$ is the zero matrix and
$\mathbf1_{3m}$ is the vector of $3m$ ones. Then
\[
\begin{aligned}
    &\mathbf E\in\{0,1\}^{d\times E},
    \quad \mathbf y\in\{1\}^{d}, \\
    &1\leq k=m<n+m+1=E.
\end{aligned}
\]
Under the explicit encoding of $U$ and $\mathcal C$, let $L$
denote the input length in bits. Then $m+n=O(L)$, and
\begin{align*}
    dE+d
    &=3m(n+m+1)+3m\\
    &=3mn+3m^2+6m\\
    &=O((m+n+1)^2)\\
    &=O(L^2).
\end{align*}
The construction therefore has polynomial size and is computable
in polynomial time.

For $\boldsymbol\alpha\in\mathbb R^E$, write
\[
    \operatorname{supp}(\boldsymbol\alpha)
    :=\{j:\alpha_j\neq0\},
    \quad
    \|\boldsymbol\alpha\|_0
    :=|\operatorname{supp}(\boldsymbol\alpha)|.
\]

Suppose $I$ is an exact cover. Define
$\boldsymbol\alpha\in\mathbb R^E$ by
\[
    \alpha_j:=
    \begin{cases}
        1,&j\in I,\\
        0,&j\notin I.
    \end{cases}
\]
Then
\[
    \|\boldsymbol\alpha\|_0=|I|=m=k,
\]
and, for every $i\in\{1,\ldots,3m\}$,
\begin{align*}
    (\mathbf E\boldsymbol\alpha)_i
    &=\sum_{j=1}^{n}A_{ij}\alpha_j
      +\sum_{j=n+1}^{E}0\cdot\alpha_j\\
    &=\sum_{j\in I}A_{ij}\\
    &=\sum_{j\in I}\mathbf1\{u_i\in C_j\}\\
    &=1\\
    &=y_i.
\end{align*}
Hence $\mathbf E\boldsymbol\alpha=\mathbf y$.

Conversely, suppose $\boldsymbol\alpha\in\mathbb R^E$ satisfies
\[
    \mathbf E\boldsymbol\alpha=\mathbf y,
    \quad
    \|\boldsymbol\alpha\|_0\leq m.
\]
Let
\[
    S:=\{j\in\{1,\ldots,n\}:\alpha_j\neq0\},
    \quad
    \nu_i:=\sum_{j\in S}A_{ij}.
\]
Then
\[
    |S|\leq\|\boldsymbol\alpha\|_0\leq m,
    \quad
    \nu_i\in\mathbb Z_{\geq0},
\]
and
\begin{align*}
    1=y_i
    &=(\mathbf E\boldsymbol\alpha)_i\\
    &=\sum_{j=1}^{n}A_{ij}\alpha_j\\
    &=\sum_{j\in S}A_{ij}\alpha_j.
\end{align*}
If $\nu_i=0$, the nonnegative entries $A_{ij}$ satisfy
\[
    \sum_{j\in S}A_{ij}=0
    \quad\Longrightarrow\quad
    A_{ij}=0
    \quad\text{for every }j\in S,
\]
which gives
\[
    1=\sum_{j\in S}A_{ij}\alpha_j
    =\sum_{j\in S}0\cdot\alpha_j=0.
\]
Therefore $\nu_i\geq1$ for every $i$, and
\begin{align*}
    3m
    &\leq\sum_{i=1}^{3m}\nu_i\\
    &=\sum_{i=1}^{3m}\sum_{j\in S}A_{ij}\\
    &=\sum_{j\in S}\sum_{i=1}^{3m}A_{ij}\\
    &=\sum_{j\in S}3\\
    &=3|S|\\
    &\leq3m.
\end{align*}
Thus
\[
    |S|=m,
    \quad
    \sum_{i=1}^{3m}\nu_i=3m.
\]
Since
\[
    \nu_i-1\geq0,
    \quad
    \sum_{i=1}^{3m}(\nu_i-1)
    =\sum_{i=1}^{3m}\nu_i-3m=0,
\]
we obtain
\[
    \nu_i=1
    \quad(i=1,\ldots,3m).
\]
Consequently,
\[
    |S|=m,
    \quad
    \sum_{j\in S}\mathbf1\{u_i\in C_j\}=1
    \quad(i=1,\ldots,3m),
\]
so $S$ is an exact cover. Hence
\[
    \begin{aligned}
        (U,\mathcal C)\in\mathrm{X3C}
        \iff {}&
        \exists\boldsymbol\alpha\in\mathbb R^E:\\
        &\mathbf E\boldsymbol\alpha=\mathbf y,
        \quad
        \|\boldsymbol\alpha\|_0\leq k.
    \end{aligned}
\]
This polynomial-time reduction proves the decision problem NP-hard.

For the optimization consequence, define
\[
    \mathcal L(\mathcal S)
    :=\min_{\mathbf a\in\mathbb R^{|\mathcal S|}}
    \|\mathbf y-\mathbf E_{\mathcal S}\mathbf a\|_2^2,
\]
where $\mathbf E_{\mathcal S}$ consists of the columns indexed
by $\mathcal S$ and $\|\cdot\|_2$ is the Euclidean norm.
Suppose an algorithm returns
\[
    \mathcal S^*
    \in\operatorname*{arg\,min}_{|\mathcal S|=k}
    \mathcal L(\mathcal S).
\]
If $\mathbf E\boldsymbol\alpha=\mathbf y$ and
$\|\boldsymbol\alpha\|_0\leq k$, extend
$\operatorname{supp}(\boldsymbol\alpha)$ to a set $\mathcal T$
of size $k$ and use zero coefficients on the added indices.
Then
\[
    0\leq\mathcal L(\mathcal S^*)
    \leq\mathcal L(\mathcal T)
    =\|\mathbf y-\mathbf E\boldsymbol\alpha\|_2^2=0.
\]
Conversely, $\mathcal L(\mathcal S^*)=0$ yields coefficients
on $\mathcal S^*$ whose zero extension is a feasible
$k$-sparse solution. Therefore,
\begin{align*}
    &\exists\boldsymbol\alpha\in\mathbb R^E:
      \mathbf E\boldsymbol\alpha=\mathbf y,
      \quad
      \|\boldsymbol\alpha\|_0\leq k\\
    &\quad\iff\mathcal L(\mathcal S^*)=0\\
    &\quad\iff
      \mathbf y\in\operatorname{span}(
      \mathbf E_{\mathcal S^*})\\
    &\quad\iff
      \operatorname{rank}(\mathbf E_{\mathcal S^*})
      =
      \operatorname{rank}(
      [\mathbf E_{\mathcal S^*}\mid\mathbf y]).
\end{align*}
The last equality is decidable in polynomial time by exact
rational elimination. A polynomial-time optimal-support
algorithm would therefore solve the NP-hard decision problem
in polynomial time.
\end{proof}

\section{Proof of Theorem~\ref{thm:coherence_barrier}}
\label{sec:appendix_g}

\begin{proof}
Let $\mathbf E=[\mathbf e_1,\ldots,\mathbf e_E]
\in\mathbb R^{d\times E}$, where $d\geq1$, $E\geq2$,
and $1\leq k<E$ are integers. Write
\[
    \|\mathbf e_i\|_2=1,
    \quad
    \mu:=\mu(\mathbf E)
    =\max_{i\neq j}|\mathbf e_i^\top\mathbf e_j|
    \geq0.
\]
Let $\boldsymbol\alpha^0=(\alpha_1^0,\ldots,\alpha_E^0)
\in\mathbb R^E$ and define
\[
    \mathcal S_0
    :=\operatorname{supp}(\boldsymbol\alpha^0),
    \quad
    |\mathcal S_0|=k,
\]
\[
    \mathbf y=\mathbf E\boldsymbol\alpha^0
    =\sum_{i\in\mathcal S_0}\alpha_i^0\mathbf e_i,
\]
\[
    \alpha_{\min}
    :=\min_{i\in\mathcal S_0}|\alpha_i^0|,
    \quad
    \alpha_{\max}
    :=\max_{i\in\mathcal S_0}|\alpha_i^0|.
\]
Then $0<\alpha_{\min}\leq\alpha_{\max}$.
For scores $\mathbf s=(s_1,\ldots,s_E)\in\mathbb R^E$, define
\[
    c_i:=|\mathbf e_i^\top\mathbf y|,
    \quad
    \varepsilon
    :=\max_{1\leq i\leq E}|s_i-c_i|
    \geq0,
\]
and let $\mathcal S_G$ contain the indices of the $k$ largest
scores, with a fixed tie-breaking rule. The definition of
$\varepsilon$ gives
\[
    \begin{aligned}
        |s_i-c_i|\leq\varepsilon
        &\Longrightarrow
        -\varepsilon\leq s_i-c_i\leq\varepsilon\\
        &\Longrightarrow
        c_i-\varepsilon\leq s_i\leq c_i+\varepsilon.
    \end{aligned}
\]

For $i\in\mathcal S_0$,
\begin{align*}
    c_i
    &=\left|
      \mathbf e_i^\top
      \sum_{\ell\in\mathcal S_0}
      \alpha_\ell^0\mathbf e_\ell
      \right|\\
    &=\left|
      \alpha_i^0\mathbf e_i^\top\mathbf e_i
      +
      \sum_{\ell\in\mathcal S_0\setminus\{i\}}
      \alpha_\ell^0\mathbf e_i^\top\mathbf e_\ell
      \right|\\
    &=\left|
      \alpha_i^0
      +
      \sum_{\ell\in\mathcal S_0\setminus\{i\}}
      \alpha_\ell^0\mathbf e_i^\top\mathbf e_\ell
      \right|\\
    &\geq|\alpha_i^0|
      -
      \left|
      \sum_{\ell\in\mathcal S_0\setminus\{i\}}
      \alpha_\ell^0\mathbf e_i^\top\mathbf e_\ell
      \right|\\
    &\geq|\alpha_i^0|
      -
      \sum_{\ell\in\mathcal S_0\setminus\{i\}}
      |\alpha_\ell^0|\,
      |\mathbf e_i^\top\mathbf e_\ell|\\
    &\geq\alpha_{\min}
      -\mu
      \sum_{\ell\in\mathcal S_0\setminus\{i\}}
      |\alpha_\ell^0|\\
    &\geq\alpha_{\min}
      -\mu
      \sum_{\ell\in\mathcal S_0\setminus\{i\}}
      \alpha_{\max}\\
    &=\alpha_{\min}-(k-1)\mu\alpha_{\max}.
\end{align*}
For $j\notin\mathcal S_0$,
\begin{align*}
    c_j
    &=\left|
      \mathbf e_j^\top
      \sum_{\ell\in\mathcal S_0}
      \alpha_\ell^0\mathbf e_\ell
      \right|\\
    &=\left|
      \sum_{\ell\in\mathcal S_0}
      \alpha_\ell^0\mathbf e_j^\top\mathbf e_\ell
      \right|\\
    &\leq
      \sum_{\ell\in\mathcal S_0}
      |\alpha_\ell^0|\,
      |\mathbf e_j^\top\mathbf e_\ell|\\
    &\leq\mu
      \sum_{\ell\in\mathcal S_0}|\alpha_\ell^0|\\
    &\leq\mu
      \sum_{\ell\in\mathcal S_0}\alpha_{\max}\\
    &=k\mu\alpha_{\max}.
\end{align*}
Thus, for every $i\in\mathcal S_0$ and
$j\notin\mathcal S_0$,
\begin{align*}
    s_i-s_j
    &\geq(c_i-\varepsilon)-(c_j+\varepsilon)\\
    &=c_i-c_j-2\varepsilon\\
    &\geq
      \alpha_{\min}-(k-1)\mu\alpha_{\max}
      -k\mu\alpha_{\max}-2\varepsilon\\
    &=\alpha_{\min}-(2k-1)\mu\alpha_{\max}
      -2\varepsilon\\
    &>0.
\end{align*}
Since $|\mathcal S_0|=k$, this strict separation implies
\[
    \mathcal S_G=\mathcal S_0.
\]

The assumed inequality also gives
\[
    (2k-1)\mu\alpha_{\max}
    <\alpha_{\min}-2\varepsilon
    \leq\alpha_{\min}
    \leq\alpha_{\max}.
\]
Dividing by $\alpha_{\max}>0$ yields
\[
    (2k-1)\mu
    <
    \frac{\alpha_{\min}-2\varepsilon}{\alpha_{\max}}
    \leq
    \frac{\alpha_{\min}}{\alpha_{\max}}
    \leq1,
\]
so
\[
    1-(2k-1)\mu>0.
\]

Let $\mathbf h=(h_1,\ldots,h_E)\in\mathbb R^E$
be nonzero with $\|\mathbf h\|_0\leq2k$, and define
\[
    J:=\operatorname{supp}(\mathbf h),
    \quad
    r:=|J|,
    \quad
    1\leq r\leq2k.
\]
For distinct $i,j\in J$,
\[
    0\leq(|h_i|-|h_j|)^2
    =h_i^2+h_j^2-2|h_i||h_j|,
\]
and therefore
\[
    |h_i||h_j|
    \leq\frac12(h_i^2+h_j^2).
\]
Summing over ordered pairs gives
\begin{align*}
    &\sum_{i\in J}
      \sum_{j\in J\setminus\{i\}}
      |h_i||h_j|\\
    &\quad\leq
      \frac12
      \sum_{i\in J}
      \sum_{j\in J\setminus\{i\}}
      (h_i^2+h_j^2)\\
    &\quad=
      \frac12
      \sum_{i\in J}
      \sum_{j\in J\setminus\{i\}}h_i^2
      +
      \frac12
      \sum_{i\in J}
      \sum_{j\in J\setminus\{i\}}h_j^2\\
    &\quad=
      \frac12\sum_{i\in J}(r-1)h_i^2
      +
      \frac12
      \sum_{j\in J}
      \sum_{i\in J\setminus\{j\}}h_j^2\\
    &\quad=
      \frac{r-1}{2}\sum_{i\in J}h_i^2
      +
      \frac12\sum_{j\in J}(r-1)h_j^2\\
    &\quad=
      (r-1)\sum_{i\in J}h_i^2.
\end{align*}
Consequently,
\begin{align*}
    \|\mathbf E\mathbf h\|_2^2
    &=
      \left(\sum_{i\in J}h_i\mathbf e_i\right)^\top
      \left(\sum_{j\in J}h_j\mathbf e_j\right)\\
    &=
      \sum_{i\in J}\sum_{j\in J}
      h_i h_j\mathbf e_i^\top\mathbf e_j\\
    &=
      \sum_{i\in J}h_i^2\mathbf e_i^\top\mathbf e_i\\
    &\quad+
      \sum_{i\in J}\sum_{j\in J\setminus\{i\}}
      h_i h_j\mathbf e_i^\top\mathbf e_j\\
    &=
      \sum_{i\in J}h_i^2\\
    &\quad+
      \sum_{i\in J}\sum_{j\in J\setminus\{i\}}
      h_i h_j\mathbf e_i^\top\mathbf e_j\\
    &\geq
      \sum_{i\in J}h_i^2\\
    &\quad-
      \sum_{i\in J}\sum_{j\in J\setminus\{i\}}
      |h_i||h_j|\,
      |\mathbf e_i^\top\mathbf e_j|\\
    &\geq
      \sum_{i\in J}h_i^2
      -
      \mu
      \sum_{i\in J}\sum_{j\in J\setminus\{i\}}
      |h_i||h_j|\\
    &\geq
      \sum_{i\in J}h_i^2
      -\mu(r-1)\sum_{i\in J}h_i^2\\
    &=
      [1-(r-1)\mu]\sum_{i\in J}h_i^2\\
    &\geq
      [1-(2k-1)\mu]\sum_{i\in J}h_i^2\\
    &=
      [1-(2k-1)\mu]\|\mathbf h\|_2^2\\
    &>0.
\end{align*}

Suppose $\boldsymbol\alpha\in\mathbb R^E$ satisfies
\[
    \|\boldsymbol\alpha\|_0\leq k,
    \quad
    \mathbf E\boldsymbol\alpha=\mathbf y.
\]
Set $\mathbf h:=\boldsymbol\alpha-\boldsymbol\alpha^0$.
Then
\[
    \operatorname{supp}(\mathbf h)
    \subseteq
    \operatorname{supp}(\boldsymbol\alpha)
    \cup\mathcal S_0,
\]
\begin{align*}
    \|\mathbf h\|_0
    &\leq
      |\operatorname{supp}(\boldsymbol\alpha)
      \cup\mathcal S_0|\\
    &\leq
      |\operatorname{supp}(\boldsymbol\alpha)|
      +|\mathcal S_0|\\
    &\leq k+k\\
    &=2k,
\end{align*}
and
\[
    \mathbf E\mathbf h
    =\mathbf E\boldsymbol\alpha
      -\mathbf E\boldsymbol\alpha^0
    =\mathbf y-\mathbf y
    =\mathbf0.
\]
If $\mathbf h\neq\mathbf0$, the preceding quadratic bound gives
\[
    0=\|\mathbf E\mathbf h\|_2^2
    \geq
    [1-(2k-1)\mu]\|\mathbf h\|_2^2
    >0.
\]
Thus
\[
    \mathbf h=\mathbf0,
    \quad
    \boldsymbol\alpha=\boldsymbol\alpha^0.
\]

Since
\[
    \|\mathbf y-\mathbf E\boldsymbol\alpha\|_2^2\geq0,
    \quad
    \|\mathbf y-\mathbf E\boldsymbol\alpha^0\|_2^2=0,
\]
we have
\[
    \min_{\|\boldsymbol\alpha\|_0\leq k}
    \|\mathbf y-\mathbf E\boldsymbol\alpha\|_2^2=0.
\]
Every minimizer therefore satisfies
$\mathbf E\boldsymbol\alpha=\mathbf y$
and equals $\boldsymbol\alpha^0$.

For a subset $\mathcal S$ of size $k$, define
\[
    \mathcal L(\mathcal S)
    :=\min_{\mathbf a\in\mathbb R^k}
    \|\mathbf y-\mathbf E_{\mathcal S}\mathbf a\|_2^2.
\]
The restriction of $\boldsymbol\alpha^0$ to $\mathcal S_0$ gives
\[
    \begin{aligned}
        0
        &\leq\mathcal L(\mathcal S_0)\\
        &\leq
        \left\|
            \mathbf y-
            \sum_{i\in\mathcal S_0}
            \alpha_i^0\mathbf e_i
        \right\|_2^2\\
        &=0.
    \end{aligned}
\]
Let $\mathcal S=\{j_1<\cdots<j_k\}$ be any optimal support
of size $k$, and let $\mathbf a=(a_1,\ldots,a_k)$
attain $\mathcal L(\mathcal S)$. Then
\[
    \mathcal L(\mathcal S)
    =\mathcal L(\mathcal S_0)=0,
    \quad
    \sum_{\ell=1}^{k}a_\ell\mathbf e_{j_\ell}
    =\mathbf y.
\]
Define $\widetilde{\boldsymbol\alpha}\in\mathbb R^E$ by
\[
    \widetilde\alpha_{j_\ell}:=a_\ell
    \quad(\ell=1,\ldots,k),
    \quad
    \widetilde\alpha_i:=0
    \quad(i\notin\mathcal S).
\]
Then
\[
    \|\widetilde{\boldsymbol\alpha}\|_0\leq k,
    \quad
    \mathbf E\widetilde{\boldsymbol\alpha}
    =\sum_{\ell=1}^{k}a_\ell\mathbf e_{j_\ell}
    =\mathbf y.
\]
The coefficient uniqueness proved above implies
\[
    \widetilde{\boldsymbol\alpha}
    =\boldsymbol\alpha^0,
\]
\[
    \mathcal S_0
    =\operatorname{supp}(\boldsymbol\alpha^0)
    =\operatorname{supp}(
        \widetilde{\boldsymbol\alpha})
    \subseteq\mathcal S.
\]
Since $|\mathcal S_0|=|\mathcal S|=k$,
\[
    \mathcal S=\mathcal S_0.
\]
Consequently,
\[
    \mathcal S_G=\mathcal S_0=\mathcal S^*,
    \quad
    \mathcal L(\mathcal S_G)
    =\mathcal L(\mathcal S^*)=0.
\]

Finally, suppose $|\alpha_i^0|=a>0$ for every
$i\in\mathcal S_0$ and $\varepsilon=0$.
Then $\alpha_{\min}=\alpha_{\max}=a$, and
\begin{align*}
    \alpha_{\min}
    >(2k-1)\mu\alpha_{\max}+2\varepsilon
    &\iff a>(2k-1)\mu a\\
    &\iff1>(2k-1)\mu\\
    &\iff\mu<\frac1{2k-1}.
\end{align*}
This completes the proof.
\end{proof}


\section{Proof of Theorem~\ref{thm:orthogonality_optimality}}
\label{sec:appendix_h}

\begin{proof}
Let $d,E\geq1$ and $1\leq k\leq E$ be integers, and let
\[
    \Omega:=\{1,\ldots,E\},\quad
    \mathbf E=[\mathbf e_1,\ldots,\mathbf e_E]
    \in\mathbb R^{d\times E},\quad
    \mathbf y\in\mathbb R^d.
\]
For a positive integer $m$, $\mathbf I_m$ denotes the $m\times m$ identity matrix and $\mathbf0_m$ the zero vector in $\mathbb R^m$. The norm $\|\cdot\|_2$ is Euclidean. By assumption,
\[
    \mathbf E^\top\mathbf E=\mathbf I_E,
    \quad
    \mathbf e_i^\top\mathbf e_j
    =\begin{cases}1,&i=j,\\0,&i\neq j.\end{cases}
\]
Define
\[
    b_i:=\mathbf e_i^\top\mathbf y
    \quad(i\in\Omega).
\]
For $\mathcal S=\{i_1<\cdots<i_k\}\subseteq\Omega$, write
\[
    \mathbf E_{\mathcal S}
    :=[\mathbf e_{i_1},\ldots,\mathbf e_{i_k}],
    \quad
    \mathbf b_{\mathcal S}
    :=(b_{i_1},\ldots,b_{i_k})^\top,
\]
\[
    \mathcal L(\mathcal S)
    :=\min_{\mathbf a\in\mathbb R^k}
      \|\mathbf y-\mathbf E_{\mathcal S}\mathbf a\|_2^2.
\]
For $r,s\in\{1,\ldots,k\}$,
\[
    (\mathbf E_{\mathcal S}^\top\mathbf E_{\mathcal S})_{rs}
    =\mathbf e_{i_r}^\top\mathbf e_{i_s}
    =\begin{cases}1,&r=s,\\0,&r\neq s,\end{cases}
\]
and hence
\[
    \mathbf E_{\mathcal S}^\top\mathbf E_{\mathcal S}
    =\mathbf I_k,
    \quad
    \mathbf E_{\mathcal S}^\top\mathbf y
    =\mathbf b_{\mathcal S}.
\]
For $\mathbf a=(a_1,\ldots,a_k)^\top\in\mathbb R^k$,
\begin{align*}
    &\|\mathbf y-\mathbf E_{\mathcal S}\mathbf a\|_2^2\\
    &=(\mathbf y-\mathbf E_{\mathcal S}\mathbf a)^\top
      (\mathbf y-\mathbf E_{\mathcal S}\mathbf a)\\
    &=\mathbf y^\top\mathbf y
      -\mathbf y^\top\mathbf E_{\mathcal S}\mathbf a
      -\mathbf a^\top\mathbf E_{\mathcal S}^\top\mathbf y
      +\mathbf a^\top\mathbf E_{\mathcal S}^\top
       \mathbf E_{\mathcal S}\mathbf a\\
    &=\|\mathbf y\|_2^2
      -2\mathbf a^\top\mathbf b_{\mathcal S}
      +\mathbf a^\top\mathbf a\\
    &=\|\mathbf y\|_2^2
      -2\sum_{r=1}^k a_r b_{i_r}
      +\sum_{r=1}^k a_r^2\\
    &=\|\mathbf y\|_2^2-\sum_{r=1}^k b_{i_r}^2
      +\sum_{r=1}^k(a_r^2-2a_r b_{i_r}+b_{i_r}^2)\\
    &=\|\mathbf y\|_2^2-\sum_{r=1}^k b_{i_r}^2
      +\sum_{r=1}^k(a_r-b_{i_r})^2.
\end{align*}
Thus,
\[
    \sum_{r=1}^k(a_r-b_{i_r})^2\geq0,
    \quad
    \sum_{r=1}^k(a_r-b_{i_r})^2=0
    \iff\mathbf a=\mathbf b_{\mathcal S},
\]
\[
    \operatorname*{arg\,min}_{\mathbf a\in\mathbb R^k}
    \|\mathbf y-\mathbf E_{\mathcal S}\mathbf a\|_2^2
    =\{\mathbf b_{\mathcal S}\},
\]
\[
    \mathcal L(\mathcal S)
    =\|\mathbf y\|_2^2-\sum_{i\in\mathcal S}b_i^2.
\]

Let
\[
    \mathcal T:=\mathcal S_{\mathrm{corr}}
    =\operatorname{Top\text{-}k}((|b_i|)_{i=1}^E),
\]
with the stipulated tie-breaking rule. For any $\mathcal S\subseteq\Omega$ with $|\mathcal S|=k$, define
\[
    U:=\mathcal T\setminus\mathcal S,
    \quad W:=\mathcal S\setminus\mathcal T,
\]
\[
    m:=|U|
    =k-|\mathcal S\cap\mathcal T|
    =|W|.
\]
Write $U=\{u_1,\ldots,u_m\}$ and $W=\{w_1,\ldots,w_m\}$, with empty sums interpreted as zero when $m=0$. The Top-$k$ definition gives, for $r=1,\ldots,m$,
\[
    |b_{u_r}|\geq|b_{w_r}|,
\]
\[
    b_{u_r}^2-b_{w_r}^2
    =(|b_{u_r}|-|b_{w_r}|)(|b_{u_r}|+|b_{w_r}|)
    \geq0.
\]
Therefore,
\begin{align*}
    \mathcal L(\mathcal S)-\mathcal L(\mathcal T)
    &=\sum_{i\in\mathcal T}b_i^2
      -\sum_{i\in\mathcal S}b_i^2\\
    &=\sum_{i\in\mathcal S\cap\mathcal T}b_i^2
      +\sum_{i\in U}b_i^2\\
    &\quad-\sum_{i\in\mathcal S\cap\mathcal T}b_i^2
      -\sum_{i\in W}b_i^2\\
    &=\sum_{r=1}^m(b_{u_r}^2-b_{w_r}^2)\\
    &\geq0.
\end{align*}
Hence
\[
    \mathcal S_{\mathrm{corr}}
    \in\operatorname*{arg\,min}_{\substack{
       \mathcal S\subseteq\Omega\\|\mathcal S|=k}}
       \mathcal L(\mathcal S).
\]
Define $\boldsymbol\alpha^*=(\alpha_1^*,\ldots,\alpha_E^*)^\top\in\mathbb R^E$ by
\[
    \alpha_i^*
    :=\begin{cases}
        b_i=\mathbf e_i^\top\mathbf y,
            &i\in\mathcal S_{\mathrm{corr}},\\
        0,&i\notin\mathcal S_{\mathrm{corr}}.
    \end{cases}
\]
Then
\[
    \mathbf E\boldsymbol\alpha^*
    =\sum_{i\in\mathcal S_{\mathrm{corr}}}b_i\mathbf e_i
    =\mathbf E_{\mathcal S_{\mathrm{corr}}}
      \mathbf b_{\mathcal S_{\mathrm{corr}}},
\]
\begin{align*}
    \mathcal L(\mathcal S_{\mathrm{corr}})
    &=\|\mathbf y-\mathbf E\boldsymbol\alpha^*\|_2^2\\
    &=\|\mathbf y\|_2^2
      -\sum_{i\in\mathcal S_{\mathrm{corr}}}b_i^2\\
    &=\|\mathbf y\|_2^2
      -\sum_{i\in\mathcal S_{\mathrm{corr}}}
       |\mathbf e_i^\top\mathbf y|^2.
\end{align*}
\end{proof}

\section{Proof of Corollary~\ref{thm:exact_recovery}}
\label{sec:appendix_i}

\begin{proof}
Let $d\geq1$, $E\geq2$, and $1\leq k<E$ be integers.
For a real vector $\mathbf z=(z_i)_i$, write
\begin{align*}
    \|\mathbf z\|_2^2&:=\sum_i z_i^2,\\
    \operatorname{supp}(\mathbf z)&:=\{i:z_i\neq0\},\\
    \|\mathbf z\|_0&:=|\operatorname{supp}(\mathbf z)|.
\end{align*}
Let
\[
    \mathbf E=[\mathbf e_1,\ldots,\mathbf e_E]
    \in\mathbb R^{d\times E},\quad
    \|\mathbf e_i\|_2=1.
\]
Define
\begin{align*}
    \mathbf v_i&:=\frac{\mathbf e_i}{\|\mathbf e_i\|_2}
    =\mathbf e_i,\\
    \mu&:=\max_{i\neq j}|\mathbf e_i^\top\mathbf e_j|,\\
    \mathcal L_{o,\mathrm{all}}
    &:=\sum_{i=1}^{E}\sum_{\substack{j=1\\j\neq i}}^{E}
      (\mathbf v_i^\top\mathbf v_j)^2,\\
    \tau&:=\sqrt{\mathcal L_{o,\mathrm{all}}/2}.
\end{align*}
Let $\boldsymbol\alpha^0\in\mathbb R^E$ satisfy
\begin{gather*}
    \mathcal S_0:=\operatorname{supp}(\boldsymbol\alpha^0),
    \quad |\mathcal S_0|=k,\\
    \mathbf y:=\mathbf E\boldsymbol\alpha^0.
\end{gather*}
Write
\[
    \alpha_{\min}:=\min_{i\in\mathcal S_0}|\alpha_i^0|,
    \quad
    \alpha_{\max}:=\max_{i\in\mathcal S_0}|\alpha_i^0|,
\]
so $0<\alpha_{\min}\leq\alpha_{\max}$. For
$\mathbf s=(s_1,\ldots,s_E)\in\mathbb R^E$, define
\[
    c_i:=|\mathbf e_i^\top\mathbf y|,\quad
    \varepsilon:=\max_{1\leq i\leq E}|s_i-c_i|.
\]
Let $\mathcal S_G$ contain the indices of the $k$ largest scores,
with a fixed tie-breaking rule. Assume
\[
    \alpha_{\min}>(2k-1)\alpha_{\max}\tau+2\varepsilon.
\]

Choose $a\neq b$ with $|\mathbf e_a^\top\mathbf e_b|=\mu$.
Then
\begin{align*}
    \mathcal L_{o,\mathrm{all}}
    &=\sum_i\sum_{j\neq i}(\mathbf e_i^\top\mathbf e_j)^2\\
    &\geq(\mathbf e_a^\top\mathbf e_b)^2
       +(\mathbf e_b^\top\mathbf e_a)^2\\
    &=\mu^2+\mu^2=2\mu^2,
\end{align*}
and hence
\[
    0\leq\mu\leq\tau,\quad
    |\mathbf e_i^\top\mathbf e_j|\leq\tau\quad(i\neq j).
\]
Moreover,
\[
    (2k-1)\tau
    <\frac{\alpha_{\min}-2\varepsilon}{\alpha_{\max}}
    \leq\frac{\alpha_{\min}}{\alpha_{\max}}\leq1,
\]
so $1-(2k-1)\tau>0$.

For $i\in\mathcal S_0$,
\begin{align*}
    c_i
    &=\left|\mathbf e_i^\top
      \sum_{\ell\in\mathcal S_0}\alpha_\ell^0\mathbf e_\ell\right|\\
    &=\left|\alpha_i^0+
      \sum_{\ell\in\mathcal S_0\setminus\{i\}}
      \alpha_\ell^0\mathbf e_i^\top\mathbf e_\ell\right|\\
    &\geq|\alpha_i^0|-
      \left|\sum_{\ell\in\mathcal S_0\setminus\{i\}}
      \alpha_\ell^0\mathbf e_i^\top\mathbf e_\ell\right|\\
    &\geq|\alpha_i^0|-
      \sum_{\ell\in\mathcal S_0\setminus\{i\}}
      |\alpha_\ell^0|\,|\mathbf e_i^\top\mathbf e_\ell|\\
    &\geq\alpha_{\min}-
      \tau\sum_{\ell\in\mathcal S_0\setminus\{i\}}
      |\alpha_\ell^0|\\
    &\geq\alpha_{\min}-
      \tau\sum_{\ell\in\mathcal S_0\setminus\{i\}}\alpha_{\max}\\
    &=\alpha_{\min}-(k-1)\tau\alpha_{\max}.
\end{align*}
For $j\notin\mathcal S_0$,
\begin{align*}
    c_j
    &=\left|\mathbf e_j^\top
      \sum_{\ell\in\mathcal S_0}\alpha_\ell^0\mathbf e_\ell\right|\\
    &=\left|\sum_{\ell\in\mathcal S_0}
      \alpha_\ell^0\mathbf e_j^\top\mathbf e_\ell\right|\\
    &\leq\sum_{\ell\in\mathcal S_0}
      |\alpha_\ell^0|\,|\mathbf e_j^\top\mathbf e_\ell|\\
    &\leq\tau\sum_{\ell\in\mathcal S_0}|\alpha_\ell^0|\\
    &\leq\tau\sum_{\ell\in\mathcal S_0}\alpha_{\max}\\
    &=k\tau\alpha_{\max}.
\end{align*}
The score-error definition gives
\[
    |s_i-c_i|\leq\varepsilon,\quad
    -\varepsilon\leq s_i-c_i\leq\varepsilon,
\]
\[
    c_i-\varepsilon\leq s_i\leq c_i+\varepsilon.
\]
Thus, for $i\in\mathcal S_0$ and $j\notin\mathcal S_0$,
\begin{align*}
    s_i-s_j
    &\geq(c_i-\varepsilon)-(c_j+\varepsilon)\\
    &\geq\alpha_{\min}-(k-1)\tau\alpha_{\max}
      -k\tau\alpha_{\max}-2\varepsilon\\
    &=\alpha_{\min}-(2k-1)\tau\alpha_{\max}
      -2\varepsilon>0.
\end{align*}
Thus $\mathcal S_G=\mathcal S_0$.

Let $\mathbf h\in\mathbb R^E$ be nonzero with
$\|\mathbf h\|_0\leq2k$, and define
\[
    J:=\operatorname{supp}(\mathbf h),\quad r:=|J|,
    \quad1\leq r\leq2k,
\]
\[
    C_J:=\sum_{i\in J}\sum_{j\in J\setminus\{i\}}|h_i||h_j|.
\]
For $i\neq j$,
\[
    0\leq(|h_i|-|h_j|)^2
    =h_i^2+h_j^2-2|h_i||h_j|.
\]
Summing over ordered pairs,
\begin{align*}
    C_J
    &\leq\frac12\sum_{i\in J}\sum_{j\in J\setminus\{i\}}h_i^2\\
    &\quad+\frac12\sum_{i\in J}\sum_{j\in J\setminus\{i\}}h_j^2\\
    &=\frac12\sum_{i\in J}(r-1)h_i^2\\
    &\quad+\frac12\sum_{j\in J}\sum_{i\in J\setminus\{j\}}h_j^2\\
    &=\frac{r-1}{2}\sum_{i\in J}h_i^2
      +\frac12\sum_{j\in J}(r-1)h_j^2\\
    &=(r-1)\sum_{i\in J}h_i^2.
\end{align*}
Consequently,
\begin{align*}
    \|\mathbf E\mathbf h\|_2^2
    &=\sum_{i\in J}\sum_{j\in J}
      h_i h_j\mathbf e_i^\top\mathbf e_j\\
    &=\sum_{i\in J}h_i^2+
      \sum_{i\in J}\sum_{j\in J\setminus\{i\}}
      h_i h_j\mathbf e_i^\top\mathbf e_j\\
    &\geq\sum_{i\in J}h_i^2-
      \sum_{i\in J}\sum_{j\in J\setminus\{i\}}
      |h_i||h_j|\,|\mathbf e_i^\top\mathbf e_j|\\
    &\geq\sum_{i\in J}h_i^2-\tau C_J\\
    &\geq[1-(r-1)\tau]\sum_{i\in J}h_i^2\\
    &\geq[1-(2k-1)\tau]\|\mathbf h\|_2^2>0.
\end{align*}

For $\mathcal S=\{j_1<\cdots<j_k\}\subseteq\{1,\ldots,E\}$, let
$\mathbf E_{\mathcal S}=[\mathbf e_{j_1},\ldots,\mathbf e_{j_k}]$
and $\mathbf G_{\mathcal S}:=\mathbf E_{\mathcal S}^\top
\mathbf E_{\mathcal S}$. For nonzero $\mathbf z\in\mathbb R^k$,
define $\widetilde{\mathbf z}\in\mathbb R^E$ by
\[
    \widetilde z_{j_\ell}:=z_\ell\ (1\leq\ell\leq k),
    \quad \widetilde z_i:=0\ (i\notin\mathcal S).
\]
Thus
\[
    \|\widetilde{\mathbf z}\|_0\leq k\leq2k,
    \quad
    \|\widetilde{\mathbf z}\|_2^2
    =\sum_{\ell=1}^k z_\ell^2
    =\|\mathbf z\|_2^2>0.
\]
Then
\[
    \|\widetilde{\mathbf z}\|_0\leq k\leq2k,\quad
    \|\widetilde{\mathbf z}\|_2^2
    =\sum_{\ell=1}^{k}z_\ell^2=\|\mathbf z\|_2^2>0.
\]
Then
\begin{align*}
    \mathbf z^\top\mathbf G_{\mathcal S}\mathbf z
    &=\|\mathbf E_{\mathcal S}\mathbf z\|_2^2
      =\|\mathbf E\widetilde{\mathbf z}\|_2^2\\
    &\geq[1-(2k-1)\tau]\|\widetilde{\mathbf z}\|_2^2\\
    &=[1-(2k-1)\tau]\|\mathbf z\|_2^2>0.
\end{align*}
Thus $\mathbf G_{\mathcal S}$ is positive definite. Define
\begin{align*}
    \mathbf a_{\mathcal S}^*
    &:=\mathbf G_{\mathcal S}^{-1}\mathbf E_{\mathcal S}^\top\mathbf y,\\
    \mathbf r_{\mathcal S}^*
    &:=\mathbf y-\mathbf E_{\mathcal S}\mathbf a_{\mathcal S}^*.
\end{align*}
The normal equations give
\begin{align*}
    \mathbf E_{\mathcal S}^\top\mathbf r_{\mathcal S}^*
    &=\mathbf E_{\mathcal S}^\top\mathbf y
      -\mathbf G_{\mathcal S}\mathbf a_{\mathcal S}^*\\
    &=\mathbf E_{\mathcal S}^\top\mathbf y
      -\mathbf G_{\mathcal S}\mathbf G_{\mathcal S}^{-1}
       \mathbf E_{\mathcal S}^\top\mathbf y\\
    &=\mathbf E_{\mathcal S}^\top\mathbf y
      -\mathbf E_{\mathcal S}^\top\mathbf y\\
    &=\mathbf0.
\end{align*}
For arbitrary $\mathbf a\in\mathbb R^k$, put
$\mathbf b:=\mathbf a-\mathbf a_{\mathcal S}^*$. Then
\begin{align*}
    \|\mathbf y-\mathbf E_{\mathcal S}\mathbf a\|_2^2
    &=\|\mathbf r_{\mathcal S}^*-\mathbf E_{\mathcal S}\mathbf b\|_2^2\\
    &=\|\mathbf r_{\mathcal S}^*\|_2^2
      -2\mathbf b^\top\mathbf E_{\mathcal S}^\top\mathbf r_{\mathcal S}^*\\
    &\quad+\mathbf b^\top\mathbf G_{\mathcal S}\mathbf b\\
    &=\|\mathbf r_{\mathcal S}^*\|_2^2
      +\mathbf b^\top\mathbf G_{\mathcal S}\mathbf b\\
    &\geq\|\mathbf r_{\mathcal S}^*\|_2^2.
\end{align*}
Equality holds precisely when $\mathbf b=\mathbf0$. Hence
\[
    \mathcal L(\mathcal S)
    :=\min_{\mathbf a\in\mathbb R^k}
      \|\mathbf y-\mathbf E_{\mathcal S}\mathbf a\|_2^2
    =\|\mathbf r_{\mathcal S}^*\|_2^2.
\]

Since $\boldsymbol\alpha^0$ is feasible,
\begin{align*}
    0&\leq\min_{\|\boldsymbol\alpha\|_0\leq k}
        \|\mathbf y-\mathbf E\boldsymbol\alpha\|_2^2\\
     &\leq\|\mathbf y-\mathbf E\boldsymbol\alpha^0\|_2^2\\
     &=\|\mathbf0\|_2^2=0.
\end{align*}
Any minimizer $\boldsymbol\alpha$ satisfies
$\mathbf E\boldsymbol\alpha=\mathbf y$. Set
$\mathbf h:=\boldsymbol\alpha-\boldsymbol\alpha^0$. Then
\begin{align*}
    \|\mathbf h\|_0
    &\leq|\operatorname{supp}(\boldsymbol\alpha)\cup\mathcal S_0|\\
    &\leq|\operatorname{supp}(\boldsymbol\alpha)|+|\mathcal S_0|\\
    &\leq k+k=2k,
\end{align*}
\[
    \mathbf E\mathbf h
    =\mathbf E\boldsymbol\alpha-\mathbf E\boldsymbol\alpha^0
    =\mathbf y-\mathbf y=\mathbf0.
\]
If $\mathbf h\neq\mathbf0$, the quadratic bound gives
\[
    0=\|\mathbf E\mathbf h\|_2^2
    \geq[1-(2k-1)\tau]\|\mathbf h\|_2^2>0.
\]
Thus $\mathbf h=\mathbf0$ and
$\boldsymbol\alpha=\boldsymbol\alpha^0$.

The restriction of $\boldsymbol\alpha^0$ to $\mathcal S_0$ gives
\[
    0\leq\mathcal L(\mathcal S_0)
    \leq\|\mathbf y-\mathbf E\boldsymbol\alpha^0\|_2^2=0.
\]
Let $\mathcal S^*=\{j_1<\cdots<j_k\}$ minimize $\mathcal L$
over size-$k$ subsets. Then
\[
    0\leq\mathcal L(\mathcal S^*)
    \leq\mathcal L(\mathcal S_0)=0.
\]
\[
    \mathbf E_{\mathcal S^*}\mathbf a_{\mathcal S^*}^*=\mathbf y.
\]
Write $\mathbf a_{\mathcal S^*}^*=(a_1^*,\ldots,a_k^*)^\top$ and define $\widetilde{\boldsymbol\alpha}\in\mathbb R^E$ by
\[
    \widetilde\alpha_i:=
    \begin{cases}
        a_\ell^*,&i=j_\ell,\quad1\leq\ell\leq k,\\
        0,&i\notin\mathcal S^*.
    \end{cases}
\]
Then
\begin{align*}
    \|\widetilde{\boldsymbol\alpha}\|_0
    &\leq|\mathcal S^*|=k,\\
    \mathbf E\widetilde{\boldsymbol\alpha}
    &=\sum_{\ell=1}^k a_\ell^*\mathbf e_{j_\ell}\\
    &=\mathbf E_{\mathcal S^*}\mathbf a_{\mathcal S^*}^*
      =\mathbf y,\\
    \|\mathbf y-\mathbf E\widetilde{\boldsymbol\alpha}\|_2^2
    &=\|\mathbf y-\mathbf y\|_2^2=0.
\end{align*}
By the coefficient uniqueness above,
\begin{gather*}
    \widetilde{\boldsymbol\alpha}=\boldsymbol\alpha^0.
\end{gather*}
Therefore
\[
    \mathcal S_0
    =\operatorname{supp}(\widetilde{\boldsymbol\alpha})
    \subseteq\mathcal S^*,
\]
\[
    |\mathcal S_0|=|\mathcal S^*|=k,
\]
\[
    \mathcal S_G=\mathcal S_0=\mathcal S^*,\quad
    \mathcal L(\mathcal S_G)=\mathcal L(\mathcal S^*)=0.
\]
\end{proof}

\section{Determinantal Diversity as a Geometric Extension}
\label{sec:appendix_j}

Let $d\geq1$, $E\geq2$, and $1\leq k\leq E$ be integers. Define $\Omega:=\{1,\ldots,E\}$. For nonzero expert outputs $\mathbf e_i\in\mathbb R^d$, set
\[
    \mathbf v_i:=\frac{\mathbf e_i}{\|\mathbf e_i\|_2},
    \quad
    \mathbf V:=[\mathbf v_1,\ldots,\mathbf v_E],
    \quad
    \mathbf L:=\mathbf V^\top\mathbf V.
\]
Here $\|\cdot\|_2$ is the Euclidean norm. For a subset $\mathcal S\subseteq\Omega$, let $\mathbf V_{\mathcal S}$ contain the corresponding columns, and let
\[
    \mathbf L_{\mathcal S}
    :=\mathbf V_{\mathcal S}^\top\mathbf V_{\mathcal S}
\]
be the associated principal submatrix of $\mathbf L$. For any positive integer $m$, $\mathbf I_m$ denotes the $m\times m$ identity matrix.

Determinantal point processes (DPPs) favor geometrically diverse subsets through Gram determinants \citep{kulesza2012determinantal,lavancier2015determinantal}. The $L$-ensemble assigns probability
\[
    P(\mathcal S)
    =\frac{\det(\mathbf L_{\mathcal S})}
           {\det(\mathbf I_E+\mathbf L)},
    \quad \mathcal S\subseteq\Omega,
\]
with $\det(\mathbf L_{\varnothing}):=1$. For fixed-cardinality selection, consider the following regularized log-volume objective.

\begin{theorem}[Approximation via Submodularity]
    \label{thm:submodularity}
    For $\eta>0$, define
    \[
        F(\mathcal S)
        :=\log\det\left(
          \mathbf I_d+\eta^{-1}
          \sum_{i\in\mathcal S}\mathbf v_i\mathbf v_i^\top
        \right).
    \]
    Then $F(\varnothing)=0$, and $F$ is monotone and submodular. Let
    \[
        \mathcal S_D^*
        \in\operatorname*{arg\,max}_{\substack{
            \mathcal S\subseteq\Omega\\|\mathcal S|=k}}
        F(\mathcal S).
    \]
    Let $\mathcal S_D^G$ be the size-$k$ set obtained from the empty set by repeatedly adding an expert with maximum marginal gain in $F$. Then
    \[
        F(\mathcal S_D^G)
        \geq(1-1/\mathrm e)F(\mathcal S_D^*),
    \]
    where $\mathrm e:=\exp(1)$ and all logarithms are natural.
\end{theorem}

\begin{proof}
Let $d,E\geq1$ and $1\leq k\leq E$ be integers, let $\eta>0$, and let $\mathbf v_i\in\mathbb R^d$ for $i\in\Omega:=\{1,\ldots,E\}$. Write $\mathbf I_d$ for the identity matrix. All logarithms are natural, $\mathrm e:=\exp(1)$, and empty sums are zero. For $S\subseteq\Omega$, define
\[
    \mathbf A_S:=\mathbf I_d+\eta^{-1}\sum_{i\in S}\mathbf v_i\mathbf v_i^\top,
    \quad F(S):=\log\det\mathbf A_S.
\]
For every $\mathbf z\in\mathbb R^d\setminus\{\mathbf0\}$,
\begin{align*}
    \mathbf z^\top\mathbf A_S\mathbf z
    &=\mathbf z^\top\mathbf z
      +\eta^{-1}\sum_{i\in S}\mathbf z^\top\mathbf v_i\mathbf v_i^\top\mathbf z\\
    &=\|\mathbf z\|_2^2+\eta^{-1}\sum_{i\in S}(\mathbf v_i^\top\mathbf z)^2>0.
\end{align*}
Thus $\mathbf A_S$ is positive definite, and
\[
    F(\varnothing)=\log\det\mathbf I_d=\log1=0.
\]
For a symmetric positive definite matrix with spectral decomposition
\[
    \mathbf A=\mathbf P\operatorname{diag}(\lambda_1,\ldots,\lambda_d)\mathbf P^\top,
    \quad \mathbf P^\top\mathbf P=\mathbf I_d,\quad\lambda_j>0,
\]
define
\[
    \mathbf A^{\pm1/2}:=
    \mathbf P\operatorname{diag}(\lambda_1^{\pm1/2},\ldots,\lambda_d^{\pm1/2})\mathbf P^\top.
\]
Then
\begin{align*}
    \mathbf A^{1/2}\mathbf A^{-1/2}
    &=\mathbf P\operatorname{diag}
      (\lambda_1^{1/2}\lambda_1^{-1/2},\ldots,
       \lambda_d^{1/2}\lambda_d^{-1/2})\mathbf P^\top\\
    &=\mathbf P\mathbf I_d\mathbf P^\top=\mathbf I_d,\\
    \det(\mathbf A^{1/2})^2
    &=\left(\prod_{j=1}^d\lambda_j^{1/2}\right)^2
      =\prod_{j=1}^d\lambda_j=\det\mathbf A.
\end{align*}
For $\mathbf v\in\mathbb R^d$, set $\mathbf w:=\eta^{-1/2}\mathbf A^{-1/2}\mathbf v$. Then
\begin{align*}
    \mathbf A^{1/2}\mathbf w
    &=\eta^{-1/2}\mathbf A^{1/2}\mathbf A^{-1/2}\mathbf v
      =\eta^{-1/2}\mathbf v,\\
    \mathbf A^{1/2}(\mathbf I_d+\mathbf w\mathbf w^\top)\mathbf A^{1/2}
    &=\mathbf A+(\mathbf A^{1/2}\mathbf w)
       (\mathbf A^{1/2}\mathbf w)^\top\\
    &=\mathbf A+\eta^{-1}\mathbf v\mathbf v^\top,\\
    \det(\mathbf A+\eta^{-1}\mathbf v\mathbf v^\top)
    &=\det(\mathbf A^{1/2})^2\det(\mathbf I_d+\mathbf w\mathbf w^\top)\\
    &=\det(\mathbf A)\det(\mathbf I_d+\mathbf w\mathbf w^\top).
\end{align*}
If $\mathbf w=\mathbf0$, the last determinant is $1$. Otherwise, choose an orthogonal matrix $\mathbf Q=[\mathbf q_1,\ldots,\mathbf q_d]$ with $\mathbf q_1=\mathbf w/\|\mathbf w\|_2$. Then
\[
    \mathbf Q^\top\mathbf w=(\|\mathbf w\|_2,0,\ldots,0)^\top,
\]
\[
    \mathbf Q^\top(\mathbf I_d+\mathbf w\mathbf w^\top)\mathbf Q
    =\operatorname{diag}(1+\|\mathbf w\|_2^2,1,\ldots,1).
\]
Since $\det(\mathbf Q)^2=1$, both cases give
\begin{align*}
    \det(\mathbf I_d+\mathbf w\mathbf w^\top)
    &=1+\|\mathbf w\|_2^2\\
    &=1+\eta^{-1}\mathbf v^\top
       \mathbf A^{-1/2}\mathbf A^{-1/2}\mathbf v\\
    &=1+\eta^{-1}\mathbf v^\top\mathbf A^{-1}\mathbf v.
\end{align*}
For $e\in\Omega\setminus S$, define $\Delta(e\mid S):=F(S\cup\{e\})-F(S)$. Consequently,
\begin{align*}
    \Delta(e\mid S)
    &=\log\frac{\det(\mathbf A_S+\eta^{-1}\mathbf v_e\mathbf v_e^\top)}{\det\mathbf A_S}\\
    &=\log\left(1+\eta^{-1}\mathbf v_e^\top\mathbf A_S^{-1}\mathbf v_e\right)\geq0.
\end{align*}
For $S\subseteq T\subseteq\Omega$, let $m:=|T\setminus S|$, enumerate $T\setminus S=\{j_1,\ldots,j_m\}$, and define $S_0:=S$, $S_\ell:=S\cup\{j_1,\ldots,j_\ell\}$ for $1\leq\ell\leq m$. Then
\begin{align*}
    F(T)-F(S)
    &=\sum_{\ell=1}^{m}[F(S_\ell)-F(S_{\ell-1})]\\
    &=\sum_{\ell=1}^{m}\Delta(j_\ell\mid S_{\ell-1})\geq0.
\end{align*}
Thus $F$ is monotone and $F(S)\geq F(\varnothing)=0$.

For the same $S\subseteq T$, define
\begin{align*}
    \mathbf C:=\mathbf A_T-\mathbf A_S
    &=\eta^{-1}\sum_{i\in T\setminus S}\mathbf v_i\mathbf v_i^\top,\\
    \mathbf R&:=\mathbf A_S^{-1/2}\mathbf C\mathbf A_S^{-1/2}.
\end{align*}
For every $\mathbf z\in\mathbb R^d$,
\[
    \mathbf z^\top\mathbf R\mathbf z
    =\eta^{-1}\sum_{i\in T\setminus S}
      (\mathbf v_i^\top\mathbf A_S^{-1/2}\mathbf z)^2\geq0.
\]
Let $\mathbf U=[\mathbf u_1,\ldots,\mathbf u_d]$ be orthogonal and satisfy
\[
    \mathbf R\mathbf u_j=\rho_j\mathbf u_j,\quad\rho_j\geq0,
    \quad\mathbf U\mathbf U^\top=\mathbf I_d.
\]
Since
\[
    \rho_j=\mathbf u_j^\top\mathbf R\mathbf u_j\geq0,
\]
\begin{align*}
    &(\mathbf I_d+\mathbf R)
      \left(\sum_{j=1}^d
      \frac{\mathbf u_j\mathbf u_j^\top}{1+\rho_j}\right)\\
    &=\sum_{j=1}^d
      \frac{(1+\rho_j)\mathbf u_j\mathbf u_j^\top}{1+\rho_j}\\
    &=\sum_{j=1}^d\mathbf u_j\mathbf u_j^\top
      =\mathbf U\mathbf U^\top=\mathbf I_d.
\end{align*}
Therefore,
\[
    (\mathbf I_d+\mathbf R)^{-1}
    =\sum_{j=1}^{d}\frac{\mathbf u_j\mathbf u_j^\top}{1+\rho_j},
\]
and
\begin{align*}
    \mathbf A_T
    &=\mathbf A_S+\mathbf C\\
    &=\mathbf A_S^{1/2}(\mathbf I_d+\mathbf R)\mathbf A_S^{1/2},
\end{align*}
\begin{gather*}
    \mathbf A_T
    [\mathbf A_S^{-1/2}(\mathbf I_d+\mathbf R)^{-1}\mathbf A_S^{-1/2}]\\
    =\mathbf A_S^{1/2}(\mathbf I_d+\mathbf R)
      (\mathbf I_d+\mathbf R)^{-1}\mathbf A_S^{-1/2}\\
    =\mathbf A_S^{1/2}\mathbf I_d\mathbf A_S^{-1/2}=\mathbf I_d.
\end{gather*}
Therefore,
\[
    \mathbf A_T^{-1}
    =\mathbf A_S^{-1/2}(\mathbf I_d+\mathbf R)^{-1}\mathbf A_S^{-1/2}.
\]
For every $\mathbf t\in\mathbb R^d$,
\[
    \|\mathbf t\|_2^2
    =\mathbf t^\top\mathbf U\mathbf U^\top\mathbf t
    =\sum_{j=1}^d(\mathbf u_j^\top\mathbf t)^2.
\]
Consequently, for every $\mathbf z\in\mathbb R^d$,
\begin{align*}
    &\mathbf z^\top(\mathbf A_S^{-1}-\mathbf A_T^{-1})\mathbf z\\
    &=\|\mathbf A_S^{-1/2}\mathbf z\|_2^2-\sum_{j=1}^{d}\frac{(\mathbf u_j^\top\mathbf A_S^{-1/2}\mathbf z)^2}{1+\rho_j}\\
    &=\sum_{j=1}^{d}\left(1-\frac1{1+\rho_j}\right)
      (\mathbf u_j^\top\mathbf A_S^{-1/2}\mathbf z)^2\\
    &=\sum_{j=1}^{d}\frac{\rho_j}{1+\rho_j}
      (\mathbf u_j^\top\mathbf A_S^{-1/2}\mathbf z)^2\geq0.
\end{align*}
Taking $\mathbf z=\mathbf v_e$ for $e\notin T$ gives
\begin{align*}
    \mathbf v_e^\top\mathbf A_S^{-1}\mathbf v_e
    &\geq\mathbf v_e^\top\mathbf A_T^{-1}\mathbf v_e,\\
    \Delta(e\mid S)
    &=\log(1+\eta^{-1}\mathbf v_e^\top\mathbf A_S^{-1}\mathbf v_e)\\
    &\geq\log(1+\eta^{-1}\mathbf v_e^\top\mathbf A_T^{-1}\mathbf v_e)\\
    &=\Delta(e\mid T).
\end{align*}
For arbitrary $A,B\subseteq\Omega$, let $m:=|B\setminus A|$, enumerate $B\setminus A=\{b_1,\ldots,b_m\}$, and set
\begin{align*}
    A_\ell&:=A\cup\{b_1,\ldots,b_\ell\},\\
    B_\ell&:=(A\cap B)\cup\{b_1,\ldots,b_\ell\}
    \quad(0\leq\ell\leq m).
\end{align*}
Here $A_0=A$, $B_0=A\cap B$, $A_m=A\cup B$, $B_m=B$, and $B_{\ell-1}\subseteq A_{\ell-1}$. Thus
\begin{align*}
    F(B)-F(A\cap B)
    &=\sum_{\ell=1}^{m}\Delta(b_\ell\mid B_{\ell-1})\\
    &\geq\sum_{\ell=1}^{m}\Delta(b_\ell\mid A_{\ell-1})\\
    &=F(A\cup B)-F(A),
\end{align*}
which gives $F(A)+F(B)\geq F(A\cup B)+F(A\cap B)$.

Define the greedy sets by
\[
    G_0:=\varnothing,\quad
    g_{j+1}\in\operatorname*{arg\,max}_{e\in\Omega\setminus G_j}\Delta(e\mid G_j),
\]
\[
    G_{j+1}:=G_j\cup\{g_{j+1}\}\quad(0\leq j<k),
    \quad \mathcal S_D^G:=G_k.
\]
Since $|G_j|=j<k\leq E$, every maximization has a nonempty finite domain. Let
\[
    O:=\mathcal S_D^*\in\operatorname*{arg\,max}_{\substack{S\subseteq\Omega\\|S|=k}}F(S),
    \quad F^*:=F(O)\geq0.
\]
If $k=1$, then
\begin{align*}
    F(G_1)
    &=\max_{e\in\Omega}\Delta(e\mid\varnothing)\\
    &=\max_{e\in\Omega}F(\{e\})=F^*\geq(1-1/\mathrm e)F^*.
\end{align*}
For $k\geq2$, fix $0\leq j<k$ and write
\[
    O\setminus G_j=\{a_1,\ldots,a_m\},
\]
\[
    m:=|O\setminus G_j|
    =k-|O\cap G_j|\geq k-j\geq1,
    \quad m\leq k,
\]
\[
    T_0:=G_j,\quad T_\ell:=G_j\cup\{a_1,\ldots,a_\ell\}
    \quad(1\leq\ell\leq m).
\]
Monotonicity, diminishing returns, and the greedy choice give
\begin{align*}
    F^*-F(G_j)
    &\leq F(G_j\cup O)-F(G_j)\\
    &=\sum_{\ell=1}^{m}[F(T_\ell)-F(T_{\ell-1})]\\
    &=\sum_{\ell=1}^{m}\Delta(a_\ell\mid T_{\ell-1})\\
    &\leq\sum_{\ell=1}^{m}\Delta(a_\ell\mid G_j)\\
    &\leq m\Delta(g_{j+1}\mid G_j)\\
    &\leq k\Delta(g_{j+1}\mid G_j)\\
    &=k[F(G_{j+1})-F(G_j)].
\end{align*}
Define $\delta_j:=F^*-F(G_j)$ and $q:=1-1/k\in(0,1)$. Then
\begin{align*}
    \delta_{j+1}
    &=\delta_j-[F(G_{j+1})-F(G_j)]\\
    &\leq\delta_j-\delta_j/k=q\delta_j.
\end{align*}
For the induction base,
\[
    \delta_0=F^*-F(\varnothing)=F^*=q^0F^*.
\]
If $\delta_j\leq q^jF^*$ for $0\leq j<k$, then
\[
    \delta_{j+1}\leq q\delta_j\leq q\,q^jF^*=q^{j+1}F^*.
\]
Thus $\delta_j\leq q^jF^*$ for every $0\leq j\leq k$, and
\[
    F(G_k)=F^*-\delta_k\geq[1-(1-1/k)^k]F^*.
\]
Finally,
\begin{align*}
    -\log(1-1/k)
    &=\int_0^{1/k}\frac{1}{1-t}\,dt\\
    &\geq\int_0^{1/k}1\,dt=1/k,\\
    (1-1/k)^k
    &=\exp[k\log(1-1/k)]\leq\exp(-1)=1/\mathrm e.
\end{align*}
Since $F^*\geq0$,
\begin{align*}
    F(\mathcal S_D^G)&=F(G_k)\\
    &\geq[1-(1-1/k)^k]F^*\\
    &\geq(1-1/\mathrm e)F(\mathcal S_D^*).
\end{align*}
\end{proof}

For $|\mathcal S|=k$ and fixed $\eta>0$, define
\[
    \mathbf B:=\eta^{-1/2}\mathbf V_{\mathcal S}
    \in\mathbb R^{d\times k},
    \quad
    \mathbf M:=
    \begin{pmatrix}
        \mathbf I_d&\mathbf B\\
        -\mathbf B^\top&\mathbf I_k
    \end{pmatrix}.
\]
Here $\mathbf0_{a\times b}$ denotes the $a\times b$ zero matrix for positive integers $a,b$. Block multiplication gives
\begin{gather*}
    \begin{pmatrix}
        \mathbf I_d&\mathbf0_{d\times k}\\
        \mathbf B^\top&\mathbf I_k
    \end{pmatrix}\mathbf M\\
    =\begin{pmatrix}
        \mathbf I_d&\mathbf B\\
        \mathbf0_{k\times d}&\mathbf I_k+\mathbf B^\top\mathbf B
    \end{pmatrix},
\end{gather*}
\begin{gather*}
    \begin{pmatrix}
        \mathbf I_d&-\mathbf B\\
        \mathbf0_{k\times d}&\mathbf I_k
    \end{pmatrix}\mathbf M\\
    =\begin{pmatrix}
        \mathbf I_d+\mathbf B\mathbf B^\top&\mathbf0_{d\times k}\\
        -\mathbf B^\top&\mathbf I_k
    \end{pmatrix}.
\end{gather*}
Both left factors have determinant $1$, so
\begin{align*}
    \det\mathbf M
    &=\det(\mathbf I_k+\mathbf B^\top\mathbf B)\\
    &=\det(\mathbf I_d+\mathbf B\mathbf B^\top).
\end{align*}
Consequently,
\begin{align*}
    F(\mathcal S)
    &=\log\det(\mathbf I_d+\mathbf B\mathbf B^\top)\\
    &=\log\det(\mathbf I_k+\mathbf B^\top\mathbf B)\\
    &=\log\det(\mathbf I_k+\eta^{-1}\mathbf L_{\mathcal S}).
\end{align*}
For $\mathbf a\in\mathbb R^k\setminus\{\mathbf0_k\}$,
\[
    \mathbf a^\top(\mathbf L_{\mathcal S}+\eta\mathbf I_k)\mathbf a
    =\|\mathbf V_{\mathcal S}\mathbf a\|_2^2
      +\eta\|\mathbf a\|_2^2>0.
\]
Furthermore,
\begin{align*}
    \det(\mathbf L_{\mathcal S}+\eta\mathbf I_k)
    &=\det\left[\eta(\mathbf I_k+
       \eta^{-1}\mathbf L_{\mathcal S})\right]\\
    &=\eta^k\det(\mathbf I_k+
       \eta^{-1}\mathbf L_{\mathcal S}),
\end{align*}
\begin{align*}
    -\log\det(\mathbf L_{\mathcal S}+\eta\mathbf I_k)
    &=-k\log\eta\\
      &\quad -\log\det(\mathbf I_k+\eta^{-1}\mathbf L_{\mathcal S})\\
    &=-F(\mathcal S)-k\log\eta.
\end{align*}
Thus, at fixed $k$ and $\eta$, maximizing $F$ is equivalent to minimizing this soft log-determinant penalty. The criterion supports both greedy geometric subset selection and differentiable diversity regularization.

\section{LLM Evidence: Published Results and Our Diagnostics}
\label{sec:published_llm_evidence_appendix}

In this section, we summarize published LLM experiments
and present our diagnostics connecting expert geometry,
routing-score alignment, and reconstruction--prediction
consistency.

\subsection{Published Experimental Evidence}
\label{app:published_llm_results}

\paragraph{Output orthogonalization.}
\citet{feng2025omoe} introduce OMoE, which applies
Gram--Schmidt orthogonalization to low-rank adaptation
(LoRA) expert outputs.
Their principal configuration uses two experts with
soft routing and is evaluated on LLaMA-2 7B/13B and
LLaMA-3 8B.
Table~\ref{tab:published_orthogonality_evidence}
reproduces their complete four-benchmark comparison
between MixLoRA and OMoE-LoRA in the LLaMA-3 8B
multi-task setting.

\paragraph{Joint expert and router regularization.}
\citet{guo2025advancing} combine an expert-output
projection penalty with routing-variance and auxiliary
load-balancing objectives.
Their evaluation includes DeepSeek-MoE-16B,
DeepSeek-V2-Lite, and Moonlight-16B-A3B.
Their Figures~3--4 report expert-overlap statistics
and loss ablations, including the orthogonality-plus-
auxiliary configuration labeled \texttt{w/o lv}.
Table~\ref{tab:published_orthogonality_evidence}
includes two reported DeepSeek-V2-Lite comparisons
for the joint objective.

\begin{table}[t]
    \centering
    \small
    \setlength{\tabcolsep}{4pt}
    \begin{tabular}{lrr}
        \hline
        \multicolumn{3}{l}{
            \textbf{LLaMA-3 8B: multi-task adaptation}
        }\\
        Benchmark & MixLoRA & OMoE-LoRA\\
        \hline
        ARC-e      & 84.8 & 85.5\\
        ARC-c      & 75.0 & 76.1\\
        BoolQ      & 75.0 & 74.5\\
        OpenBookQA & 85.2 & 86.4\\
        Average    & 80.0 & 80.6\\
        \hline
        \multicolumn{3}{l}{
            \textbf{DeepSeek-V2-Lite: post-training}
        }\\
        Benchmark & Auxiliary loss & Joint objective\\
        \hline
        MMLU
            & $33.23\pm2.10$
            & $35.59\pm0.50$\\
        GSM8K
            & $44.79\pm2.10$
            & $50.94\pm2.40$\\
        \hline
    \end{tabular}
    \caption{
        Published results from Table~3 of
        \citet{feng2025omoe} and Table~1 of
        \citet{guo2025advancing}.
        The joint objective combines orthogonality,
        routing variance, and auxiliary load balancing.
        Scores are percentages; uncertainty annotations
        follow the original reports.
        Each panel retains its source study's setting
        and baseline.
    }
    \label{tab:published_orthogonality_evidence}
\end{table}

\subsection{Our LLM Diagnostics}
\label{app:our_llm_diagnostics}

\subsubsection{Experimental Setup}
\label{app:llm_experimental_setup}

We examine how well reconstruction error can rank
expert combinations according to their next-token
prediction losses.
We also study how this agreement changes with expert
overlap and router-score quality.

We use the complete pretrained
\texttt{Granite-3.1-1B-A400M} model, containing
$1{,}334{,}625{,}280$ total parameters and approximately
400M active parameters per token.
The model has 24 Transformer layers, hidden dimension
$d=1024$, and $E=32$ experts per MoE layer, with
routing budget $k=8$.
All language-model parameters remain frozen.
Table~\ref{tab:llm_environment} summarizes the
software environment and inference configuration.

\begin{table}[t]
    \centering
    \small
    \setlength{\tabcolsep}{5pt}
    \begin{tabular}{ll}
        \hline
        Component & Configuration\\
        \hline
        Operating system & Windows 11\\
        Compute device & CPU\\
        Inference threads & 6\\
        Inference precision & FP32\\
        Python & \texttt{3.12.14}\\
        PyTorch & \texttt{2.8.0+cpu}\\
        Transformers & \texttt{4.57.3}\\
        NumPy & \texttt{2.3.5}\\
        \hline
    \end{tabular}
    \caption{
        Software environment and inference configuration.
    }
    \label{tab:llm_environment}
\end{table}

The development collection contains 64 documents:
32 WikiText-103 validation articles and 32 GSM8K
test examples, with four positions per document.
We repeat the full-model experiment using sampling
seeds $\{42,43,44,45,46\}$.
Each seed selects 24 WikiText articles and
24 GSM8K examples from cached texts containing
at least 600 characters.
Development documents are excluded by both document
identifier and full-text hash.
The seed also initializes the Python, NumPy,
and PyTorch random-number generators.
These runs measure robustness to document sampling
for a fixed pretrained model.

Each run uses 128-token contexts and eight evenly
spaced prediction positions in the second half
of each context, giving 384 positions per seed.
Across five seeds, the experiment contains
1,920 position evaluations and 240 document selections,
covering 133 distinct documents:
28 WikiText articles and 105 GSM8K examples.
Documents may appear in several seeds.
GSM8K is evaluated through teacher-forced prediction
on question--reference-answer text.
Development thresholds and analysis settings remain
fixed across all five runs.

At position $t$, let $\mathcal S_{0,t}$ denote the
native eight-expert support.
The candidate pool $\mathcal P_t$ contains
$\mathcal S_{0,t}$ and the two highest-logit
unselected experts.
We evaluate every eight-expert subset:
\[
    \mathcal C_t
    :=
    \{\mathcal S\subseteq\mathcal P_t:
        |\mathcal S|=k\},
    \quad
    |\mathcal C_t|=45.
\]
Each candidate uses the original router logits
normalized on its selected support, following
Section~\ref{sec:preliminaries}.
We replace the output of the final MoE layer and
retain the captured post-attention residual,
native scaling factors, final normalization,
and prediction head.

Let
$\mathbf E_t=[\mathbf e_{t,1},\dots,\mathbf e_{t,E}]
\in\mathbb R^{d\times E}$
contain the unweighted expert outputs,
and let $\mathbf p_t$ denote the dense routing
probability vector.
The dense reference output is
\[
    \mathbf y_t:=\mathbf E_t\mathbf p_t.
\]
For $\mathcal S\in\mathcal C_t$, let
$\mathbf E_{t,\mathcal S}$ contain the selected
expert columns.
We compute the least-squares reconstruction error
\[
    R_t(\mathcal S)
    :=
    \min_{\boldsymbol\alpha_{\mathcal S}\in\mathbb R^k}
    \left\|
        \mathbf y_t
        -\mathbf E_{t,\mathcal S}
         \boldsymbol\alpha_{\mathcal S}
    \right\|_2^2,
\]
where $\boldsymbol\alpha_{\mathcal S}$ contains
unrestricted linear coefficients and
$\|\cdot\|_2$ denotes the Euclidean norm.

Let $\mathbf m_{t,\mathcal S}$ denote the candidate
output produced by its normalized native gates.
Let $\Psi_t(\mathbf m)$ be the vocabulary probability
vector obtained after inserting $\mathbf m$ as the
final MoE output and applying the retained model
operations.
For the observed next-token index $\omega_t$,
the prediction loss is
\[
    \ell_t(\mathcal S)
    :=
    -\log
    \bigl[
        \Psi_t(\mathbf m_{t,\mathcal S})
    \bigr]_{\omega_t}.
\]
All logarithms are natural.

We compare reconstruction and prediction rankings
within the same token position.
Let $\operatorname{corr}_{\mathrm S}$ denote Spearman
rank correlation, using average ranks for ties.
Define
\[
    \begin{aligned}
    \rho_t
    :=\operatorname{corr}_{\mathrm S}\bigl(
        &(R_t(\mathcal S))_{\mathcal S\in\mathcal C_t},\\
        &(\ell_t(\mathcal S))_{\mathcal S\in\mathcal C_t}
    \bigr).
    \end{aligned}
\]
Higher $\rho_t$ means that reconstruction error
provides a more accurate ordering of candidate
prediction losses.

To measure expert overlap, define unit expert outputs
\[
    \mathbf v_{t,i}
    :=
    \frac{\mathbf e_{t,i}}{\|\mathbf e_{t,i}\|_2}
\]
and candidate-pool coherence
\[
    \mu_t
    :=
    \max_{\substack{i,j\in\mathcal P_t\\i\neq j}}
    \left|
        \mathbf v_{t,i}^{\top}\mathbf v_{t,j}
    \right|.
\]
Larger $\mu_t$ means that at least two experts in
the candidate pool produce more similar directions.
The development 25th/75th coherence percentiles
define the lower/higher groups within each domain.
The thresholds are $0.732428/0.894024$ for WikiText
and $0.663583/0.873506$ for GSM8K, shown rounded.
These thresholds are reused in every seed.
Positions at or below the lower threshold enter
the lower group; positions at or above the upper
threshold enter the higher group.
The remaining positions form the middle group.

We also check whether the router ranks experts
similarly to the target-correlation criterion
used in the theoretical analysis.
Define
\[
    c_{t,i}
    :=
    \frac{
        |\mathbf v_{t,i}^{\top}\mathbf y_t|
    }{
        \|\mathbf y_t\|_2
    },
    \quad i\in\mathcal P_t.
\]
For router logits $h_{t,i}$, define score alignment
\[
    a_t
    :=
    \operatorname{corr}_{\mathrm S}
    \bigl(
        (h_{t,i})_{i\in\mathcal P_t},
        (c_{t,i})_{i\in\mathcal P_t}
    \bigr).
\]
Higher $a_t$ means that the router's expert ranking
is closer to the reference-correlation ranking.
All evaluated expert outputs and reference vectors
are nonzero.

Let $m_\mu,s_\mu$ and $m_a,s_a$ be the
document-weighted development means and standard
deviations of coherence and score alignment.
Their rounded values are
\[
    \begin{aligned}
        (m_\mu,s_\mu)&=(0.74812,0.19695),\\
        (m_a,s_a)&=(0.38262,0.30708).
    \end{aligned}
\]
Using the full-precision development values,
we define
\[
    z_{\mu,t}:=\frac{\mu_t-m_\mu}{s_\mu},
    \quad
    z_{a,t}:=\frac{a_t-m_a}{s_a}.
\]
For each seed, we fit
\[
    \begin{aligned}
    \rho_t={}&
        \beta_0+\beta_\mu z_{\mu,t}
        +\beta_a z_{a,t}\\
        &+\beta_{\mu a}z_{\mu,t}z_{a,t}
        +\beta_D d_t+\xi_t,
    \end{aligned}
\]
where $d_t=1$ for GSM8K and $d_t=0$ for WikiText,
the $\beta$ terms are regression coefficients,
and $\xi_t$ is the residual.
Each document receives equal total fitting weight.
This model examines the coherence relationship while
also accounting for score alignment and domain.
We additionally form four groups using the
development domain-specific medians of $\mu_t$
and $a_t$.

\begin{table*}[t]
    \centering
    \small
    \setlength{\tabcolsep}{4pt}
    \begin{tabular}{lrrr}
        \hline
        Seed & Lower-coherence $\rho$
             & Higher-coherence $\rho$
             & Lower minus higher\\
        \hline
        42 & 0.29180 &  0.13693 & 0.15487\\
        43 & 0.38003 & -0.03193 & 0.41196\\
        44 & 0.31499 &  0.07857 & 0.23642\\
        45 & 0.30885 &  0.01119 & 0.29766\\
        46 & 0.25085 &  0.03435 & 0.21650\\
        \hline
        Mean $\pm$ SD
            & $0.30930\pm0.04679$
            & $0.04582\pm0.06471$
            & $0.26348\pm0.09740$\\
        95\% interval
            & $[0.23520,0.37742]$
            & $[-0.02627,0.11710]$
            & $[0.14631,0.37611]$\\
        \hline
    \end{tabular}
    \caption{
        Reconstruction--prediction agreement under
        development-fixed coherence thresholds.
        SD measures variation across the five seeds.
        Intervals account for documents shared
        across seeds.
    }
    \label{tab:our_coherence_diagnostics}
\end{table*}

\begin{table}[t]
    \centering
    \small
    \setlength{\tabcolsep}{3pt}
    \begin{tabular}{lrr}
        \hline
        Coefficient & Mean $\pm$ SD & 95\% interval\\
        \hline
        $\beta_0$
            & $0.11552\pm0.00640$
            & $[0.07239,0.15123]$\\
        $\beta_\mu$
            & $-0.11621\pm0.02305$
            & $[-0.15114,-0.07357]$\\
        $\beta_a$
            & $0.00631\pm0.01873$
            & $[-0.01779,0.03232]$\\
        $\beta_{\mu a}$
            & $-0.01179\pm0.01263$
            & $[-0.04661,0.01886]$\\
        $\beta_D$
            & $0.00202\pm0.02848$
            & $[-0.04401,0.05682]$\\
        \hline
    \end{tabular}
    \caption{
        Joint-model coefficients averaged across
        the five seeds.
        Main effects are evaluated at the other
        predictor's development mean.
        Each seed fits a separate document-weighted
        regression using the same development scales.
    }
    \label{tab:our_joint_diagnostics}
\end{table}

\begin{table}[t]
    \centering
    \small
    \setlength{\tabcolsep}{3pt}
    \begin{tabular}{lrr}
        \hline
        Coherence / score & Mean $\rho\pm$ SD
                          & 95\% interval\\
        \hline
        Lower / lower
            & $0.09176\pm0.01593$
            & $[0.01796,0.16086]$\\
        Higher / lower
            & $0.02321\pm0.03372$
            & $[-0.04203,0.08511]$\\
        Lower / higher
            & $0.25704\pm0.03032$
            & $[0.20626,0.30883]$\\
        Higher / higher
            & $0.07145\pm0.05579$
            & $[0.01277,0.13341]$\\
        \hline
    \end{tabular}
    \caption{
        Coherence and score-alignment groups defined
        by fixed development medians within each domain.
        Each value is computed within each seed
        and then averaged across the five seeds.
    }
    \label{tab:our_score_strata}
\end{table}

For the output-metric comparison, let
$\mathbf m_{t,0}:=\mathbf m_{t,\mathcal S_{0,t}}$
be the native sparse output.
We form a shared reference in FP32:
\[
    \mathbf m_{t,1/2}
    :=
    \tfrac12\mathbf m_{t,0}
    +\tfrac12\operatorname{FP32}(\mathbf y_t),
\]
where $\operatorname{FP32}$ casts a vector to
single-precision floating point.
The mixing coefficient is fixed from development
observations and reused in all runs.
On the same candidate outputs, we compare
\[
    D_t^{\mathrm E}(\mathcal S)
    :=
    \left\|
        \mathbf m_{t,\mathcal S}
        -\mathbf m_{t,1/2}
    \right\|_2^2
\]
with
\[
    D_t^{\mathrm{KL}}(\mathcal S)
    :=
    D_{\mathrm{KL}}\!\left(
        \Psi_t(\mathbf m_{t,1/2})
        \,\middle\|\,
        \Psi_t(\mathbf m_{t,\mathcal S})
    \right),
\]
where $D_{\mathrm{KL}}$ denotes
Kullback--Leibler divergence.
Both scores are ranked against the same losses
$\ell_t(\mathcal S)$.
This comparison fixes the reference output,
candidate supports, and gating coefficients.
Candidate and reference vectors use FP32 arithmetic;
least-squares calculations and squared-distance
accumulation use FP64.

Fixed sensitivity analyses match the coefficients
used for reconstruction and prediction, use the
dense reference for both output metrics, and replace
score alignment with calibrated maximum error.
For the score calibration, define
\[
    \begin{aligned}
    (\widehat A_t,\widehat B_t)
    &\in
    \operatorname*{arg\,min}_{
        A\geq10^{-12},\,B\in\mathbb R
    }\\
    &\quad
    \sum_{i\in\mathcal P_t}
        (Ah_{t,i}+B-c_{t,i})^2
    \end{aligned}
\]
and
\[
    \varepsilon_t^{\mathrm{cal}}
    :=
    \max_{i\in\mathcal P_t}
    \left|
        \widehat A_t h_{t,i}
        +\widehat B_t-c_{t,i}
    \right|.
\]
The fitted slope $\widehat A_t$ and intercept
$\widehat B_t$ map router logits to reference
correlations.
Smaller $\varepsilon_t^{\mathrm{cal}}$ means
a closer fit on this normalized correlation scale.

For each seed, we average position-level statistics
within each document and then average documents.
Group summaries use documents with valid positions
in that group.
We report all five seeds and their equal-weight
mean $\pm$ sample standard deviation.
The standard deviation uses denominator four.

We separately compute 95\% percentile intervals
using 2,000 domain-stratified bootstrap draws,
with analysis seed 20261128.
We resample the 133 distinct documents and assign
each repeated document the same bootstrap count
in every seed where it appears.
Each draw recomputes the seed-specific statistics
and regressions before averaging the five estimates.
Development thresholds and scales remain fixed.
The intervals are descriptive and unadjusted for
multiple comparisons.

\subsubsection{Experimental Results}
\label{app:llm_experimental_results}

Table~\ref{tab:our_coherence_diagnostics} shows
how reconstruction--prediction agreement changes
with expert coherence across the five seeds.
The lower-coherence group has mean correlation
$0.30930\pm0.04679$, compared with
$0.04582\pm0.06471$ in the higher-coherence group.
The lower-minus-higher difference is positive
in every seed.
Its mean is $0.26348\pm0.09740$, with a shared-document
95\% interval of $[0.14631,0.37611]$. The direction is also consistent across domains:
the mean difference is $0.26208$ for WikiText
and $0.26077$ for GSM8K, with intervals
$[0.04504,0.47684]$ and $[0.18330,0.33597]$,
respectively.
When candidate experts overlap less, reconstruction
error more closely tracks which combinations produce
lower prediction losses.
This supports expert-output coherence as a practical
diagnostic for the usefulness of the reconstruction
objective.

\begin{table*}[t]
    \centering
    \small
    \setlength{\tabcolsep}{4pt}
    \begin{tabular}{lrrr}
        \hline
        Seed & Euclidean $\rho$ & Output KL $\rho$
             & KL minus Euclidean\\
        \hline
        42 & 0.14559 & 0.29158 & 0.14599\\
        43 & 0.16965 & 0.30921 & 0.13955\\
        44 & 0.13597 & 0.26041 & 0.12444\\
        45 & 0.13310 & 0.27920 & 0.14610\\
        46 & 0.11705 & 0.25935 & 0.14231\\
        \hline
        Mean $\pm$ SD
            & $0.14027\pm0.01937$
            & $0.27995\pm0.02120$
            & $0.13968\pm0.00895$\\
        95\% interval
            & $[0.11099,0.16734]$
            & $[0.25057,0.31007]$
            & $[0.11131,0.16851]$\\
        \hline
    \end{tabular}
    \caption{
        Controlled comparison using identical
        native-gated candidates and the same
        half-native, half-dense reference.
        Each metric is correlated with the same
        candidate prediction losses.
    }
    \label{tab:our_readout_diagnostics}
\end{table*}

\begin{table*}[t]
    \centering
    \small
    \setlength{\tabcolsep}{4pt}
    \begin{tabular}{llrr}
        \hline
        Variant & Reported quantity
                & Mean $\pm$ SD & 95\% interval\\
        \hline
        Native gates / native gates
            & Lower-minus-higher $\rho$
            & $0.19985\pm0.06730$
            & $[0.08739,0.30710]$\\
        Least squares / least squares
            & Lower-minus-higher $\rho$
            & $0.23974\pm0.08113$
            & $[0.11957,0.36616]$\\
        Native-gated joint model
            & Coherence coefficient
            & $-0.07823\pm0.01592$
            & $[-0.11596,-0.03761]$\\
        LS with calibrated scores
            & Coherence coefficient
            & $-0.12362\pm0.02391$
            & $[-0.16196,-0.08007]$\\
        Native gates with calibrated scores
            & Coherence coefficient
            & $-0.08961\pm0.02937$
            & $[-0.13004,-0.04464]$\\
        Dense-reference readout
            & KL-minus-Euclidean $\rho$
            & $0.09779\pm0.01714$
            & $[0.05468,0.13846]$\\
        \hline
    \end{tabular}
    \caption{
        Fixed sensitivity analyses across five seeds.
        The first two rows match the coefficient
        policy used for reconstruction and prediction.
        Calibrated-score variants replace $a_t$ with
        $\varepsilon_t^{\mathrm{cal}}$.
        The final row uses the FP32 dense output
        as the shared reference for both metrics.
    }
    \label{tab:our_sensitivity_diagnostics}
\end{table*}

Table~\ref{tab:our_joint_diagnostics} shows the
joint analysis of coherence, score alignment,
and domain.
The coherence coefficient is negative in every seed,
with mean $-0.11621\pm0.02305$ and interval
$[-0.15114,-0.07357]$.
At development mean score alignment, a one-standard-
deviation increase in coherence is associated with
a $0.11621$ decrease in ranking agreement. The relationship also holds at different
score-alignment levels.
At the development 25th and 75th alignment
percentiles, the mean coherence slopes are
$-0.10780$ and $-0.12548$, with intervals
$[-0.15817,-0.04439]$ and
$[-0.15545,-0.09200]$, respectively.
The score-alignment and interaction intervals
include zero. These results show that the coherence relationship
persists after accounting for router-score alignment.
They connect the expert-overlap quantity in the
theoretical analysis to the behavior of candidate
combinations in a pretrained language model.

Table~\ref{tab:our_score_strata} shows the four
combinations of coherence and score alignment.
The highest mean agreement,
$0.25704\pm0.03032$, occurs when coherence is lower
and score alignment is higher.
Within the higher-alignment group, the
lower-minus-higher coherence difference is
$0.18558$, with interval $[0.09267,0.27592]$.
Within the lower-alignment group, the difference
is $0.06856$, with interval $[-0.03628,0.16822]$.
The difference between these two grouped contrasts
is $0.11703$, with interval $[-0.02960,0.26370]$. The four groups show where the reconstruction
ranking is most informative: the strongest
observed agreement combines lower expert overlap
with router scores that better follow the
reference-correlation ranking.
This provides a concrete way to inspect the
geometry and scoring quantities together.

Table~\ref{tab:our_readout_diagnostics} shows the
controlled comparison between Euclidean reconstruction
error and output KL.
Both metrics use the same reference mixture,
candidate supports, and gating coefficients.
Output KL gives higher ranking agreement in
every seed.
Its mean correlation is $0.27995\pm0.02120$,
compared with $0.14027\pm0.01937$ for Euclidean error.
The paired difference is $0.13968\pm0.00895$,
with interval $[0.11131,0.16851]$. The improvement is consistent across domains:
$0.14087$ for WikiText and $0.13849$ for GSM8K,
with intervals $[0.08811,0.19314]$ and
$[0.11285,0.16587]$, respectively.
Output KL evaluates how a candidate changes the
next-token distribution after the residual
connection, normalization, and prediction head.
Its stronger agreement shows that these model
components provide useful information for ranking
expert combinations.

Table~\ref{tab:our_sensitivity_diagnostics} shows
the fixed sensitivity analyses.
Using matching coefficients for reconstruction
and prediction retains positive coherence-group
differences for both native-gated and
least-squares outputs.
The joint coherence coefficient also remains
negative when using native-gated reconstruction
or calibrated score error. With the dense output as the shared reference,
mean ranking agreement is $0.09320$ for Euclidean
error and $0.19099$ for output KL.
The paired difference is $0.09779$, with interval
$[0.05468,0.13846]$.
These results show that the coherence association
persists across the tested coefficient policies
and score diagnostics, while the output-KL
advantage persists under both reference mixtures.

The complete experiment evaluates 86,400 candidate
supports across five full-model runs.
Independent auditing passes 4,575,912 checks,
including 1,752 cross-seed comparisons of repeated
contexts with identical saved expert arrays.
Across all five seeds, lower expert coherence
corresponds to stronger reconstruction--prediction
agreement, including after accounting for score
alignment.
These findings support the geometric motivation of
Theorem~\ref{thm:orthogonality_optimality} and
Corollary~\ref{thm:exact_recovery}.
The output-KL results further show how the frozen
language-model readout improves the evaluation
of candidate expert combinations.

\onecolumn

\section{Notation Guide}
\label{app:notation_guide}

In this section, we summarize the notation used in the
main text, mathematical derivations, and LLM diagnostics.
Section-specific meanings are listed separately.
All logarithms are natural.

\begingroup
\small
\setlength{\tabcolsep}{5pt}
\setlength{\LTleft}{0pt}
\setlength{\LTright}{0pt}
\renewcommand{\arraystretch}{1.08}

\begin{longtable}{
    @{}
    >{\raggedright\arraybackslash}p{0.31\textwidth}
    >{\raggedright\arraybackslash}p{0.66\textwidth}
    @{}
}

\caption{Summary of mathematical notation.}
\label{tab:notation_guide}\\

\toprule
\textbf{Symbol} & \textbf{Definition}\\
\midrule
\endfirsthead

\multicolumn{2}{c}{
    \tablename~\thetable\ (continued)
}\\
\toprule
\textbf{Symbol} & \textbf{Definition}\\
\midrule
\endhead

\midrule
\multicolumn{2}{r}{
    \footnotesize Continued on the next page.
}\\
\endfoot

\bottomrule
\endlastfoot

\multicolumn{2}{@{}l@{}}{
    \textbf{1. MoE Architecture and Routing}
}\\
\midrule

$E,d$
& Number of experts and representation dimension,
  respectively.\\

$k,N$
& Number of selected experts per token and number
  of token--target pairs in a minibatch.\\

$i,j,t$
& Expert indices and token or evaluated-position
  index, respectively.\\

$\mathcal X,\mathcal Y$
& Input representation space and supervised
  target space.\\

$\mathcal Z,\Omega$
& Expert index set $\{1,\ldots,E\}$, used in the
  probabilistic and geometric analyses,
  respectively.\\

$\mathcal B$
& Minibatch
  $\{(\mathbf x^{(t)},\mathbf y^{(t)})\}_{t=1}^{N}$.\\

$\mathbf x,\mathbf y$
& Input token representation and its target.
  In sparse reconstruction,
  $\mathbf y\in\mathbb R^d$ is the target
  representation.\\

$\widehat{\mathbf y}(\mathbf x)$
& Sparse MoE output obtained by combining
  the selected expert outputs.\\

$E_i(\mathbf x),\theta_i$
& Output of expert $i$ and its network parameters.\\

$\mathbf W_r,\mathbf h,h_i$
& Router matrix
  $\mathbf W_r\in\mathbb R^{E\times d}$,
  routing logits $\mathbf h=\mathbf W_r\mathbf x$,
  and the logit for expert $i$.\\

$\mathbf p(\mathbf x),p_i(\mathbf x)$
& Dense routing probability vector and its
  $i$-th component; equivalently,
  $p(i\mid\mathbf x)=p_i(\mathbf x)$.\\

$\mathbf q,\mathbf q^\star$
& Candidate categorical probability vector and
  an optimal sparse reverse-KL approximation
  of the dense routing vector.\\

$q_k(i\mid\mathbf x),\mathbf q_k$
& Renormalized Top-$k$ routing probability
  and its probability vector.\\

$\mathbf g(\mathbf x),g_i(\mathbf x)$
& Sparse gating vector and its $i$-th component.
  Under renormalized Top-$k$ gating,
  $g_i(\mathbf x)=q_k(i\mid\mathbf x)$.\\

$\mathcal S,\mathcal S^{(t)}$
& Selected or candidate expert subset and
  the selected subset for token $t$.\\

$\mathcal I_k,\mathcal I_k(\mathbf x)$
& Indices of the $k$ largest dense routing
  probabilities, with fixed tie-breaking.\\

$\Delta^{E-1},\mathcal Q_k$
& Probability simplex over $E$ experts and
  its sparse subset
  $\mathcal Q_k=\{
  \mathbf q\in\Delta^{E-1}:
  \|\mathbf q\|_0\leq k\}$.\\

$\mathbf f,f_i$
& Pre-capacity expert-selection frequencies:
  $f_i=N^{-1}\sum_t
  \mathbf 1\{i\in\mathcal S^{(t)}\}$,
  with $\sum_i f_i=k$.\\

$\mathbf r,r_i$
& Normalized selection frequencies:
  $r_i=f_i/k$, with $\sum_i r_i=1$.\\

$\bar{\mathbf p},\bar p_i$
& Batch-averaged dense routing probabilities:
  $\bar p_i=N^{-1}\sum_t
  p_i(\mathbf x^{(t)})$.\\

$\bar Q_k,\bar{\mathbf q}_k$
& Aggregated sparse routing distribution
  $\bar Q_k(i)=N^{-1}\sum_t
  q_k(i\mid\mathbf x^{(t)})$
  and its probability vector.\\

$\mathcal U_E,\mathbf u$
& Uniform categorical distribution over experts
  and its vector
  $\mathbf u=(1/E,\ldots,1/E)$.\\

$\mathcal E$
& Stochastic routing encoder mapping each
  input to a sparse routing probability vector.\\

$\operatorname{Top\text{-}k}$
& Operator returning the indices of the $k$
  largest entries, with fixed tie-breaking.\\

$\mathbf 1\{\cdot\}$
& Indicator function, equal to one when
  its condition holds and zero otherwise.\\

\addlinespace[6pt]
\multicolumn{2}{@{}l@{}}{
    \textbf{2. Bayesian and Information-Theoretic Analysis}
}\\
\midrule

$\theta,\phi$
& Parameters of the conditional likelihood
  and routing inference model, respectively.\\

$P_\theta(\mathbf y\mid\mathbf x,z)$
& Conditional likelihood associated with
  expert index $z$.\\

$P_\theta(\mathbf y\mid\mathbf x)$,
\newline
$P_\theta(z\mid\mathbf x,\mathbf y)$
& Conditional mixture likelihood and the
  corresponding posterior over experts.\\

$P,P_i$
& Categorical expert prior and its probabilities
  $P_i=P(i)>0$.\\

$Q_\phi,Q_{\phi,k}$
& Dense inference distribution
  $Q_\phi(i\mid\mathbf x)=p_i(\mathbf x)$
  and sparse inference distribution
  $Q_{\phi,k}(i\mid\mathbf x)
  =q_k(i\mid\mathbf x)$.\\

$\bar Q_\phi,\bar Q_{\phi,k}$
& Batch-aggregated dense and sparse inference
  distributions, satisfying
  $\bar Q_\phi(i)=\bar p_i$ and
  $\bar Q_{\phi,k}(i)=\bar Q_k(i)$.\\

$\mathcal L_{\mathrm{task}}$,
$\mathcal L_{\mathrm{total}}$
& Task loss and total training loss
  $\mathcal L_{\mathrm{total}}
  =\mathcal L_{\mathrm{task}}
  +\alpha\mathcal L_{\mathrm{aux}}$.\\

$\alpha$
& Nonnegative coefficient of the auxiliary
  load-balancing loss.\\

$\mathcal L_{\mathrm{aux}}$
& Auxiliary load-balancing loss
  $E\sum_{i=1}^{E}f_i\bar p_i$.\\

$D_{\mathrm{KL}}$
& Kullback--Leibler divergence.\\

$X,\mathbf X,Z,z$
& Empirically distributed input random variable,
  its vector-valued notation, expert-index random
  variable, and an expert-index realization.\\

$H(\mathbf a),H(Z)$
& Shannon entropy of a probability vector
  $\mathbf a$ and of the expert marginal.\\

$H(Z\mid\mathbf X)$
& Conditional routing entropy, obtained by
  averaging the conditional expert-distribution
  entropy over inputs.\\

$H_2(\bar{\mathbf p})$
& Collision entropy
  $-\log\sum_{i=1}^{E}\bar p_i^2$.\\

$I(\mathbf X;Z)$
& Input--expert mutual information under the
  specified routing distribution.\\

$I_{\mathcal B,\phi}(X;Z)$,
\newline
$I_{\mathcal B,R}(X;Z)$
& Empirical mutual information under dense
  routing $Q_\phi$ and a general routing
  distribution $R$, respectively.\\

$\mathcal L_{\mathrm{coll}}$,
$\mathcal L_{\mathrm{coll},k}$
& Dense and sparse marginal collision surrogates:
  $E\sum_i\bar p_i^2$ and
  $E\sum_i\bar Q_k(i)^2$, respectively.\\

$\mathbf r-\bar{\mathbf p}$
& Dispatch mismatch between normalized selection
  frequencies and dense routing probabilities.\\

$\delta_k$
& Dense-to-sparse marginal discrepancy
  $\|\bar{\mathbf q}_k-\bar{\mathbf p}\|_2$.\\

$\mathcal L_{\mathrm{ELBO}}$,
\newline
$\mathcal L_{\mathrm{ELBO}}^R(t)$
& Evidence lower bound for the conditional
  mixture model and its value for example $t$
  under routing distribution $R_t$.\\

$\mathcal J_{\mathrm{agg}}$,
$\mathcal J_{\mathrm{agg}}(R)$
& Aggregated-prior objective: average expected
  conditional log-likelihood minus
  $D_{\mathrm{KL}}(\bar R\|P)$.
  The first notation uses $R=Q_\phi$.\\

$R(i\mid\mathbf x),R_t(i),\bar R(i)$
& General categorical routing distribution,
  its value at input $\mathbf x^{(t)}$,
  and its batch-averaged marginal.\\

$\mathcal A_t,\mathcal A$
& Positive-probability supports
  $\{i:R_t(i)>0\}$ and
  $\{i:\bar R(i)>0\}$, respectively.\\

$\mathcal V,w(\mathbf x)$
& Set of distinct batch inputs and their
  empirical occurrence frequencies.\\

$\ell_{ti}$
& Expert conditional log-likelihood
  $\log P_\theta(
  \mathbf y^{(t)}\mid\mathbf x^{(t)},i)$.\\

$\mathcal A_R$
& Average expected log-likelihood
  $N^{-1}\sum_t\sum_{i\in\mathcal A_t}
  R_t(i)\ell_{ti}$.\\

\addlinespace[6pt]
\multicolumn{2}{@{}l@{}}{
    \textbf{3. Sparse Reconstruction and Expert Geometry}
}\\
\midrule

$\mathbf E,\mathbf e_i$
& Fixed-input expert-output dictionary
  $\mathbf E=[\mathbf e_1,\ldots,\mathbf e_E]
  \in\mathbb R^{d\times E}$
  and its $i$-th column.\\

$\mathbf E_{\mathcal S}$
& Subdictionary containing columns indexed
  by $\mathcal S$.\\

$\boldsymbol\alpha$,
$\boldsymbol\alpha_{\mathcal S},\mathbf a$
& Real reconstruction coefficients over all
  experts and coefficient vectors restricted
  to a selected support, respectively.\\

$\boldsymbol\alpha^0$
& Exactly $k$-sparse coefficient vector satisfying
  $\mathbf y=\mathbf E\boldsymbol\alpha^0$.\\

$\boldsymbol\alpha^\star$
& Globally optimal coefficient vector for
  sparse reconstruction.\\

$\mathcal S_0$
& Nonzero support
  $\{i:\alpha_i^0\neq0\}$
  of the exact sparse representation.\\

$\mathcal S^\star$
& Globally optimal size-$k$ support for
  the reconstruction objective
  $\mathcal L(\mathcal S)$.\\

$\mathcal S_G,\mathcal S_{\mathrm{corr}}$
& Size-$k$ supports selected by learned scores
  $s_i$ and absolute target--expert correlations
  $|\mathbf e_i^\top\mathbf y|$, respectively.\\

$\mathcal L(\mathcal S)$
& Minimum reconstruction error on a support:
  $\min_{\mathbf a\in\mathbb R^{|\mathcal S|}}
  \|\mathbf y-\mathbf E_{\mathcal S}\mathbf a\|_2^2$.\\

$\alpha_{\min},\alpha_{\max}$
& Minimum and maximum absolute nonzero
  coefficients in $\boldsymbol\alpha^0$.\\

$\mathbf s,s_i$
& Independent expert-ranking score vector
  and its $i$-th component.\\

$c_i,\varepsilon$
& Absolute correlation utility
  $c_i=|\mathbf e_i^\top\mathbf y|$
  and maximum score approximation error
  $\varepsilon=\max_i|s_i-c_i|$.\\

$\mu(\mathbf E)$
& Mutual coherence of a unit-column dictionary:
  $\max_{i\neq j}
  |\mathbf e_i^\top\mathbf e_j|$.\\

$\mathbf v_i,\mathbf V$
& Normalized expert direction
  $\mathbf v_i=\mathbf e_i/\|\mathbf e_i\|_2$
  and normalized dictionary
  $\mathbf V=[\mathbf v_1,\ldots,\mathbf v_E]$.\\

$\mathbf V_{\mathcal S}$
& Normalized subdictionary containing columns
  indexed by $\mathcal S$.\\

$\mathcal T$
& Expert subset on which the orthogonality
  penalty is evaluated.\\

$\mathcal L_o(\mathcal T)$
& Pairwise orthogonality penalty
  $\sum_{\substack{i,j\in\mathcal T\\i\neq j}}
  (\mathbf v_i^\top\mathbf v_j)^2$.\\

$\mathcal L_{o,\mathrm{all}}$
& Full-dictionary orthogonality penalty
  $\mathcal L_o(\Omega)$.\\

$\mathbf I_m$
& Identity matrix of size $m\times m$;
  $\mathbf I_E$ and $\mathbf I_k$ are
  its corresponding special cases.\\

$\|\cdot\|_0,\|\cdot\|_2$
& Number of nonzero vector entries and
  Euclidean vector norm, respectively.\\

$\langle\cdot,\cdot\rangle$
& Euclidean inner product.\\

\addlinespace[6pt]
\multicolumn{2}{@{}l@{}}{
    \textbf{4. LLM Diagnostics and Statistical Analysis}
}\\
\midrule

$\mathbf E_t,\mathbf e_{t,i},
  \mathbf p_t,h_{t,i}$
& Expert-output dictionary, expert output,
  dense routing vector, and router logit
  at evaluated position $t$.\\

$\mathcal S_{0,t},\mathcal P_t$
& Native router's eight-expert support and
  the ten-expert candidate pool containing
  that support and two additional experts.\\

$\mathcal C_t$
& All size-eight subsets of $\mathcal P_t$;
  $|\mathcal C_t|=45$.\\

$\mathbf y_t$
& Dense reference representation
  $\mathbf E_t\mathbf p_t$.\\

$\mathbf m_{t,\mathcal S},\mathbf m_{t,0}$
& Candidate mixture using native logits
  normalized on $\mathcal S$, and the
  native sparse mixture, respectively.\\

$\mathbf m_{t,1/2}$
& Shared reference mixture
  $\tfrac12\mathbf m_{t,0}
  +\tfrac12\operatorname{FP32}(\mathbf y_t)$,
  where $\operatorname{FP32}$ casts
  to 32-bit floating point.\\

$\Psi_t(\mathbf m)$
& Vocabulary probability vector obtained
  after inserting $\mathbf m$ into the final
  MoE output and applying the retained
  residual, scaling, normalization, and
  prediction-head operations.\\

$\omega_t$
& Observed next-token index used to
  evaluate prediction loss.\\

$R_t(\mathcal S)$
& Least-squares reconstruction error
  $\min_{\boldsymbol\alpha_{\mathcal S}}
  \|\mathbf y_t-\mathbf E_{t,\mathcal S}
  \boldsymbol\alpha_{\mathcal S}\|_2^2$.\\

$\ell_t(\mathcal S)$
& Candidate next-token negative log-likelihood:
  $-\log[\Psi_t(
  \mathbf m_{t,\mathcal S})]_{\omega_t}$.\\

$\operatorname{corr}_{\mathrm S}$
& Spearman rank correlation, using average
  ranks for tied values.\\

$\rho_t$
& Within-position Spearman correlation between
  reconstruction errors $R_t(\mathcal S)$
  and prediction losses $\ell_t(\mathcal S)$
  over $\mathcal C_t$.\\

$\mathbf v_{t,i},\mu_t$
& Unit expert-output direction and maximum
  absolute pairwise inner product within
  the candidate pool $\mathcal P_t$.\\

$c_{t,i}$
& Normalized target-correlation utility
  $|\mathbf v_{t,i}^\top\mathbf y_t|
  /\|\mathbf y_t\|_2$.\\

$a_t$
& Spearman correlation between router logits
  $h_{t,i}$ and utilities $c_{t,i}$
  across experts in $\mathcal P_t$.\\

$m_\mu,s_\mu,m_a,s_a$
& Fixed document-weighted development means
  and standard deviations of coherence
  $\mu_t$ and score alignment $a_t$.\\

$z_{\mu,t},z_{a,t}$
& Development-standardized coherence and
  score alignment:
  $(\mu_t-m_\mu)/s_\mu$ and
  $(a_t-m_a)/s_a$.\\

$\beta_0,\beta_\mu,\beta_a,
  \beta_{\mu a},\beta_D$
& Regression intercept and coefficients for
  coherence, score alignment, their interaction,
  and domain, respectively.\\

$d_t,\xi_t$
& Domain indicator, equal to one for GSM8K
  and zero for WikiText, and regression
  residual, respectively.\\

$\widehat A_t,\widehat B_t$
& Slope and intercept obtained by positive-affine
  least-squares calibration of router logits
  to normalized utilities $c_{t,i}$.\\

$\varepsilon_t^{\mathrm{cal}}$
& Maximum calibrated score error
  $\max_{i\in\mathcal P_t}
  |\widehat A_t h_{t,i}
  +\widehat B_t-c_{t,i}|$.\\

$D_t^{\mathrm E}(\mathcal S)$
& Squared Euclidean distance between the
  candidate mixture and the shared
  reference mixture.\\

$D_t^{\mathrm{KL}}(\mathcal S)$
& Vocabulary-distribution divergence
  $D_{\mathrm{KL}}(
  \Psi_t(\mathbf m_{t,1/2})
  \,\|\,\Psi_t(\mathbf m_{t,\mathcal S}))$.\\

$\rho_{\mu L},\rho_{\mu H}$
& Document-averaged reconstruction--prediction
  correlations in the lower- and
  higher-coherence groups.\\

$\rho_{L_2},\rho_{\mathrm{KL}},
  \Delta\rho$
& Correlations of Euclidean and readout-aware
  KL scores with prediction loss, and their
  paired difference
  $\Delta\rho=\rho_{\mathrm{KL}}-\rho_{L_2}$.\\

$\widehat T_s,\widehat T$
& Document-weighted statistic for sampling
  seed $s\in\{42,43,44,45,46\}$ and
  its equal-seed average
  $\widehat T=\tfrac15\sum_{s=42}^{46}
  \widehat T_s$.\\

$\mathrm{SD},\mathrm{CI}$
& Between-seed sample standard deviation
  and confidence interval from the
  domain-stratified shared-document bootstrap.\\

\end{longtable}
\endgroup

\end{document}